\documentclass[11pt, a4paper]{utda}

\usepackage{microtype}
\usepackage{subcaption}
\usepackage{graphicx}
\graphicspath{./Figures}
\usepackage{amsmath,amssymb,amsfonts}
\usepackage{acronym}
\usepackage{enumitem}
\usepackage{caption}
\usepackage{booktabs}
\usepackage{balance}
\usepackage{xspace,setspace}
\usepackage[dvipsnames]{xcolor}
\usepackage{tabularx,colortbl,multirow,array,makecell}
\usepackage{mathtools}
\usepackage{amsthm}
\usepackage{hyperref}
\usepackage{todonotes}
\usepackage[misc]{ifsym}
\usepackage{bm}
\usepackage{amsmath}
\usepackage{amssymb}

\makeatletter
\renewcommand{\paragraph}{%
  \@startsection{paragraph}{4}%
  {\z@}{0ex \@plus 0ex \@minus 0ex}{-1em}%
  {\normalfont\normalsize\bfseries}%
}
\makeatother

\makeatletter
\DeclareRobustCommand\onedot{\futurelet\@let@token\@onedot}
\def\@onedot{\ifx\@let@token.\else.\null\fi\xspace}

\makeatother

\definecolor{gblue}{HTML}{4285F4}
\definecolor{gred}{HTML}{DB4437}

\theoremstyle{plain}
\newtheorem{theorem}{Theorem}[section]
\newtheorem{proposition}[theorem]{Proposition}

\newtheorem{corollary}[theorem]{Corollary}
\theoremstyle{definition}

\theoremstyle{remark}
\newtheorem{remark}[theorem]{Remark}

\definecolor{bigaired}{RGB}{156, 0, 0}
\definecolor{uclablue}{RGB}{39, 116, 174}

\definecolor{darkred}{RGB}{200, 0, 0}
\definecolor{darkblue}{RGB}{0, 0, 200}
\definecolor{blue}{RGB}{0, 0, 250}

\definecolor{light}{RGB}{225, 250, 250}
\definecolor{lightgray}{RGB}{0.9, 0.9, 0.9}
\definecolor{lightred}{RGB}{250, 200, 200}
\definecolor{lightblue}{RGB}{210, 220, 250}

\definecolor{doderblue}{RGB}{30, 144, 255}
\definecolor{select}{RGB}{222, 235, 247}
\definecolor{unselect}{RGB}{247, 207, 206}

\hypersetup{colorlinks=true, citecolor=uclablue, linkcolor=blue, urlcolor=darkblue}

\newcommand{\R}{\mathbb{R}}

\acrodef{llm}[LLM]{large language model}
\acrodef{lrm}[LRM]{large reasoning model}
\acrodef{ar}[AR]{autoregressive}
\acrodef{sd}[SD]{speculative decoding}

\usepackage[table]{xcolor} 
\usepackage{booktabs}
\usepackage{multirow}
\usepackage[utf8]{inputenc} 
\usepackage[T1]{fontenc}    
\usepackage{hyperref}       
\usepackage{url}            
\usepackage{booktabs}       
\usepackage{amsfonts}       
\usepackage{nicefrac}       
\usepackage{microtype}      
\usepackage{xcolor}         
\usepackage{fontawesome5}

\usepackage{amsmath} 
\DeclareMathOperator*{\argmin}{arg\,min}

\usepackage{algpseudocode}
\usepackage{wrapfig}
\usepackage{algorithm}
\usepackage{graphicx}
\usepackage{subcaption}
\usepackage{pifont}
\usepackage{xurl}

\usepackage[most]{tcolorbox}

\definecolor{gray94}{gray}{.94}
\definecolor{gray90}{gray}{.90}
\definecolor{basegray}{RGB}{245,245,245}
\definecolor{rlvrblue}{RGB}{230,245,255}

\definecolor{lightred}{RGB}{255,240,240}
\definecolor{lightgreen}{RGB}{240,255,240}
\definecolor{lightblue}{RGB}{240,245,255}
\definecolor{lightgray}{RGB}{245,245,245}
\definecolor{MyGreen}{rgb}{0.13, 0.55, 0.13}

\newtcolorbox{AIbox}[2][]{aibox,title=#2,#1}

\definecolor{rliableolive}{HTML}{BBCC33}
\definecolor{rliableblue}{HTML}{77AADD}
\definecolor{rliablered}{HTML}{EE8866}
\definecolor{SDEblue}{RGB}{28 58 88}
\definecolor{cc1}{rgb}{1.0, 0.44, 0.37}
\definecolor{cc2}{rgb}{0.0, 0.2, 0.6}
\definecolor{cc3}{RGB}{255, 191, 0}
\definecolor{cc4}{RGB}{0, 128, 128}

\providecommand{\R}{\mathbb{R}}

\providecommand{\merge}{\mathrm{merge}}
\providecommand{\TA}{\mathrm{TA}}
\providecommand{\TIES}{\mathrm{TIES}}
\providecommand{\TSV}{\mathrm{TSV}}
\providecommand{\RAM}{\mathrm{RAM}}

\providecommand{\TopKMask}{\operatorname{TopKMask}}
\providecommand{\sign}{\operatorname{sign}}

\providecommand{\blockdiag}{\operatorname{blockdiag}}

\definecolor{myblue}{RGB}{0,77,64} 
\definecolor{mybg}{RGB}{240,248,247} 

\newenvironment{isoblock}[1][1.0em]{%
  \begin{list}{}{%
    \setlength{\leftmargin}{#1}%
    \setlength{\rightmargin}{0pt}%
    \setlength{\topsep}{0.25em}%
    \setlength{\partopsep}{0pt}%
    \setlength{\parsep}{0.15em}%
    \setlength{\itemsep}{0pt}%
  }%
  \item[]\ignorespaces
}{%
  \end{list}
}

\newcommand\blfootnote[1]{%
  \begingroup
  \renewcommand\thefootnote{}\footnote{#1}%
  \addtocounter{footnote}{-1}%
  \endgroup
}

\newtcolorbox{takeawaybox}{
  enhanced,
  colback=mybg,
  colframe=myblue,
  coltitle=myblue,
  boxrule=0.6pt,
  arc=6pt,
  left=6pt,
  right=6pt,
  top=6pt,
  bottom=6pt,
  title=\textsc{Takeaway}
}

\DeclareRobustCommand{\isosymbol}{%
  \tikz[baseline=-0.52ex, x=1.08em, y=1.08em, line cap=round, line join=round]{%
    \begin{scope}[rotate=-34]
      \draw[orange, line width=1.05pt] (0,0) ellipse [x radius=0.58, y radius=0.18];%
    \end{scope}%
    \filldraw[fill=white, draw=orange, line width=1.08pt] (0,0) circle (0.29);%
    \begin{scope}[rotate=-34]
      \draw[orange, line width=1.08pt] (-0.58,0) arc[start angle=180,end angle=342,x radius=0.58,y radius=0.18];%
      \fill[orange] (0.55,-0.055) -- (0.66,-0.005) -- (0.565,0.055) -- cycle;%
    \end{scope}%
  }%
}
\title{\texorpdfstring{\isosymbol\ \textcolor{orange}{ISO}: An RLVR-Native Optimization Stack}{ISO: An RLVR-Native Optimization Stack}}
\renewcommand{\today}{}

\author{
\textbf{Hanqing Zhu}\textsuperscript{1,\ding{171}}\blfootnote{Work completed when Hanqing Zhu was at UT Austin.
This work builds on our prior analysis of RL optimization dynamics
(\href{https://arxiv.org/abs/2511.08567}{\textit{The Path Not Taken}})
and develops a concrete RL-native optimization stack including both offline and online.},
\textbf{Wenyan Cong}\textsuperscript{1,\ding{171}},
\textbf{Zhizhou Sha}\textsuperscript{1},
\textbf{Sagnik Mukherjee}\textsuperscript{2},
\textbf{Xinyuan Song}\textsuperscript{3},
\textbf{David González-Martínez}\textsuperscript{6},
\textbf{Xiaoxia Wu}\textsuperscript{4},
\textbf{Yuandong Tian}\textsuperscript{5},
\textbf{Shiwei Liu}\textsuperscript{6},
\textbf{David Z. Pan}\textsuperscript{1},
\textbf{Zhangyang "Atlas" Wang}\textsuperscript{1,\textdagger} \\[0.5ex]
\textsuperscript{1}The University of Texas at Austin 
\textsuperscript{2}UIUC,
\textsuperscript{3}Emory University,
\textsuperscript{4}Together AI,\\
\textsuperscript{5}Recursive Superintelligence Inc, 
\textsuperscript{6}ELLIS Institute Tübingen
\\[1.65ex]
\centering
\href{https://github.com/zhuhanqing/ISO}{\faGithub\ Code}
\qquad
\href{https://iso-rlvr.github.io}{\faGlobe\ Website}\vspace{-1em}
}

\correspondingauthor{ \textsuperscript{\ding{171}}Equal contribution.
\textsuperscript{\textdagger}Corresponding author.
\quad \href{mailto:hqzhu@utexas.edu}{\faEnvelope\ \texttt{hqzhu@utexas.edu}}, \href{mailto:wycong@utexas.edu}{\faEnvelope\ \texttt{wycong@utexas.edu}}
}

\begin{document}

\begin{abstract}
Reinforcement learning with verifiable rewards (RLVR) advances the reasoning capabilities of language models, yet the optimization layer that converts reward feedback into weight-space updates remains poorly understood.
We study this layer and functionally validate \emph{spectral inheritance}: RLVR can reuse the base model's weight spectra while acquiring new behavior by adapting both the associated input and output singular frames.
We operationalize this finding as \textbf{Isospectral Optimization (ISO)}, an RLVR-native optimization stack with complementary offline and online instantiations.
\textit{Offline}, \textbf{ISO-Merger} composes shared-base specialists in fixed-spectrum frame coordinates without post-merge data, rollouts, or distillation, and achieves the highest aggregate scores among the data-free merging methods.
\textit{Online}, \textbf{ISO-Optimizer} applies AdamW or Muon directly to the frame variables under fixed base spectra, improving aggregate accuracy in the reported math and coding experiments.
On Qwen3-4B-Base, ISO-AdamW reaches AdamW's 400-update accuracy in 260 updates with 1.5 $\times$ fewer updates and 34\% less wall-clock time, and exceeds AdamW by 2.2 points at update 400.
These results support a reusable RLVR design principle: \emph{inherit the spectrum, optimize the frames}.
\end{abstract}

\maketitle

\vspace{-15pt}
\begin{figure}[th]
\centering
\includegraphics[width=0.98\linewidth]{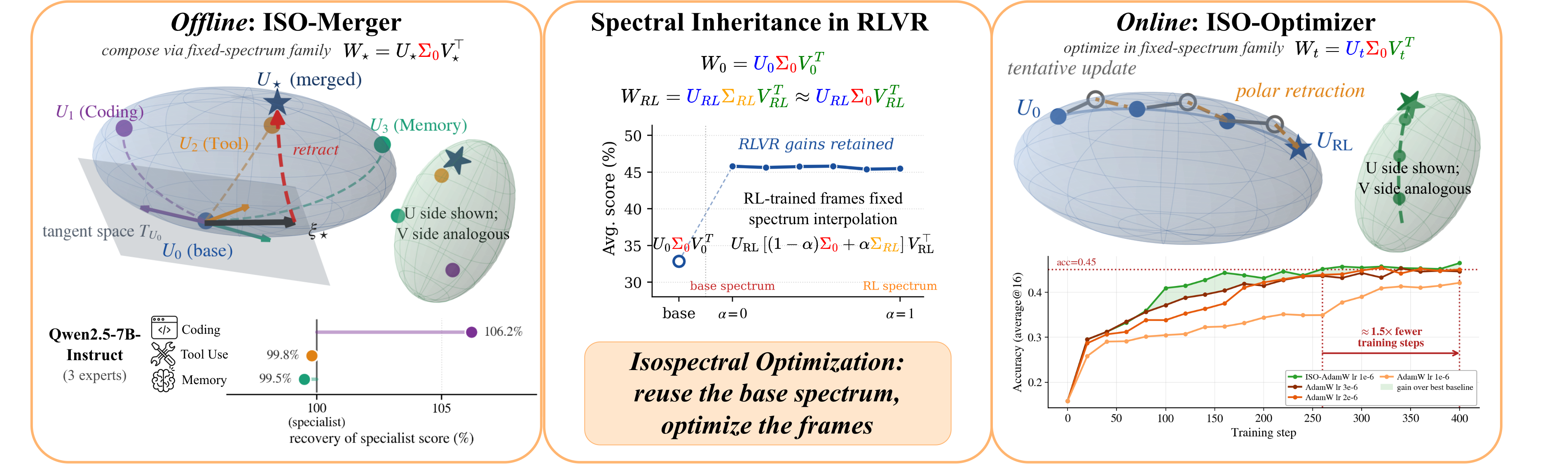}
\caption{\small
\textbf{From spectral inheritance to Isospectral Optimization.}
\textbf{Middle.}
Unconstrained RLVR changes both singular frames while its spectra remain close
to their base values: restoring the base spectra while retaining RL-trained frames preserves most acquired gains. ISO turns this regularity into an inductive bias: reuse $\Sigma_0$ and optimize $(U,V)$.
\textbf{Left.}
ISO-Merger (offline) applies project--mask--merge--retract to shared-base specialists'
frame changes and reconstructs a fixed-spectrum model that recovers their
specialist capabilities.
\textbf{Right.}
ISO-Optimizer (online) applies a conventional optimizer (e.g., AdamW) to $(U,V)$ under fixed
$\Sigma_0$; on Qwen3-4B-Base, ISO-AdamW reaches AdamW's end-of-run accuracy with
$1.5\times$ fewer actor updates.
}
\vspace{-15pt}
\label{fig:teaser}
\end{figure}

\section{Introduction}
\label{sec:intro}

Reinforcement learning with verifiable rewards (RLVR) has become a major
scaling axis for reasoning models~\cite{grok2025}, with advances in data~\cite{scale2026rlenvironments},
objectives~\cite{deepseek_r1,tajwar2026maximum},
and systems~\cite{slime_github}.
Yet the optimizers and parameterizations that translate reward feedback into weight-space updates remain largely inherited from pre-training~\cite{adamw,liu2025muon,zhu2025apollo,wen2025fantastic}. This default is not obviously natural: pre-training learns from dense token-level supervision, whereas RLVR adapts an already capable policy using comparatively sparse outcome-level rewards.

\noindent We study this missing optimization layer and uncover a separation between \textit{what RLVR reuses and what it changes}. We call this \textbf{spectral inheritance}: RLVR can \textit{reuse the base model's weight spectra} while \textit{acquiring new behavior} by \textit{changing both the associated input and output singular frames}.

\noindent Existing accounts do not expose this separation, but they do suggest
that RLVR differs structurally from pre-training and supervised fine-tuning
(SFT). At the policy level, RLVR has been described through reverse-KL and
KL-proximal views~\cite{zhu2025path,shenfeld2025rl}. At the parameter level,
its Euclidean updates have been reported to be sparse and off-principal, with
strongly overlapping footprints across independent runs initialized from the
same base model~\cite{mukherjee2025reinforcement,zhu2025path}. These findings
identify distinctive structure in RLVR, but not a coordinate system that
separates what is reused from what changes.

\paragraph{Observation: spectra stay.}
Building on prior observations of limited spectral drift in
RLVR~\cite{zhu2025path}, we formalize and quantify
\emph{near-isospectrality}: in the unconstrained RLVR runs studied
here, learned weight matrices remain close to the
\emph{fixed-spectrum families} of their base weights.
Specifically, $W_t=U_t\Sigma_tV_t^\top$ is close to
$U_t\Sigma_0V_t^\top$, which shares the base spectrum $\Sigma_0$.
Since this proximity arises without an explicit spectral constraint,
we distinguish this observed proximity from an implicit
optimization preference: dimension-aware calibration reveals
no strong additional suppression of spectrum-changing motion
beyond an isotropic reference.
The key question is therefore not merely whether the spectra
remain close, but whether they can be functionally reused:
\emph{are the remaining spectral changes necessary for the
acquired behavior, or can the base spectra be reused?}

\paragraph{Functional regularity: spectral inheritance.}
We test this through complementary interventions.
Restoring the base spectra after RLVR while retaining the learned frames preserves most of the acquired gains.
Keeping the base spectra fixed throughout training while optimizing the frames yields gains over standard training, whereas spectrum-only training yields limited improvement.

\paragraph{Structure: two-sided frame adaptability.}
Spectral inheritance shows that the base spectra can be reused, but leaves open how the associated frame changes are structured.
Could the learned endpoint be explained by a simpler transformation that remixes only within the incoming input and output singular subspaces, or that adapts only one side?
At the tested truncation ranks, both alternatives leave substantially more of the checkpoint change unexplained than reusing the incoming spectrum while adapting both frames.
This pattern also recurs across sequential RL stages with distinct objectives.

\paragraph{ISO: From spectral inheritance to isospectral Optimization.}
These findings motivate a simple design principle:
\emph{inherit the spectrum, optimize the frames}.
We operationalize this inductive bias in
\textbf{Isospectral Optimization (ISO)}, an RLVR-native framework that represents post-training change within the fixed-spectrum families of the base weights while keeping both singular frames adaptable.
ISO has two complementary instantiations spanning the RLVR workflow
(Fig.~\ref{fig:teaser}):

\begin{isoblock}[0.5em]
\textbf{\textit{Offline: ISO-Merger}} enables checkpoint-only composition of shared-base specialists.
A modular RLVR workflow trains domain specialists from a shared base and later
consolidates them, often via an on-policy distillation
stage~\cite{zeng2026glm}.
Guided by spectral inheritance, ISO-Merger instead combines their singular-frame changes under
the shared base spectra to recover complementary capabilities
in a single model, without post-merge data, rollouts, gradient
updates, or distillation.

\textbf{\textit{Online: ISO-Optimizer}} for online RL learning.
Conventional weight-space training applies optimizers such as AdamW
or Muon directly to the model weights.
ISO-Optimizer instead applies their update rules to the singular-frame
variables while keeping the base spectra fixed, with retraction
restoring frame orthogonality after each update.
Across the reported reasoning and coding experiments at multiple
model scales, its variants improve aggregate accuracy and reach
matched accuracy in fewer actor updates than their weight-space
counterparts.
\end{isoblock}

Our contributions are summarized as follows:
\begin{itemize}[
    itemsep=0pt,   
    topsep=-2pt,    
    parsep=0pt,    
    partopsep=0pt,  
    leftmargin=*,
]
    \item \textbf{Spectral inheritance in RLVR.}
    We quantify near-isospectrality in RLVR, validate the functional
    reuse of base spectra, and show that both singular frames must remain adaptable for a low-residual characterization of the observed checkpoint changes.

    \item \textbf{RLVR-Native Isospectral Optimization Framework.}
    We then introduce ISO, an RL-native optimziation framework that reuses the base spectra
    and represents reward-driven post-training change through the associated
    singular-frame coordinates.

    \item \textbf{Data-free offline RL expert composition.}
     We develop ISO-Merger for checkpoint-only composition of
    shared-base specialists.
    It recovers complementary specialist capabilities and achieves
    the highest aggregate scores among the evaluated data-free
    merging methods in both tested settings.

    \item \textbf{Fixed-spectrum online RLVR.}
    We develop ISO-Optimizer, instantiated with AdamW and Muon,
    and demonstrate improved aggregate accuracy in the reported
    mathematical-reasoning and coding experiments with
    1.5B--4B models.
    On Qwen3-4B-Base, ISO-AdamW reaches AdamW's 400-update
    accuracy in 260 actor updates with 34\% less wall-clock time,
    and exceeds AdamW by 2.2 percentage points at update 400.
\end{itemize}

\section{Spectra Stay: Near-Isospectrality to Spectral Inheritance}
\label{sec:dynamics}

\paragraph{Probe spectral structure.}
Building on prior observations of limited spectral
drift~\cite{zhu2025path}, we formalize and quantify
\emph{near-isospectrality} and calibrate spectrum-changing motion
against an isotropic reference.
We then establish \emph{spectral inheritance} by answering:
\emph{are the remaining spectral changes necessary for the acquired
behavior, or can the base spectra be reused?}

\paragraph{Notation and fixed-spectrum families.}
For $W\in\mathbb{R}^{d_{\mathrm{out}}\times d_{\mathrm{in}}}$,
let $q=\min\{d_{\mathrm{out}},d_{\mathrm{in}}\}$ and write an ordered
thin SVD $W=U\Sigma V^\top$, with
$\Sigma=\operatorname{Diag}(\sigma(W))$.
The vector $\sigma(W)$ lists all $q$ singular values in
nonincreasing order, including zeros.
The frames satisfy $U\in\mathrm{St}(d_{\mathrm{out}},q)$ and
$V\in\mathrm{St}(d_{\mathrm{in}},q)$, where
$\mathrm{St}(d,r)=\{X\in\mathbb{R}^{d\times r}:X^\top X=I_r\}$
is the Stiefel manifold of matrices with orthonormal columns.
For a base weight $W_0=U_0\Sigma_0V_0^\top$, define its
\emph{fixed-spectrum family}:
\begin{equation}
\small
\label{eq:fixed_spectrum_family}
\begin{aligned}
\mathcal F(W_0)
&:=\{Z\in\mathbb{R}^{d_{\mathrm{out}}\times d_{\mathrm{in}}}:
       \sigma(Z)=\sigma(W_0)\}\\
&=\{U\Sigma_0V^\top:
       U\in\mathrm{St}(d_{\mathrm{out}},q),\,
       V\in\mathrm{St}(d_{\mathrm{in}},q)\}.
\end{aligned}
\end{equation}
This family fixes the singular values while allowing both frames
to vary.
The \emph{base} is the checkpoint initializing the RL stage under
study; it may be pretrained, SFT-trained, or already RL-trained.

\begin{figure}[t]
    \centering
    \includegraphics[width=\linewidth]
        {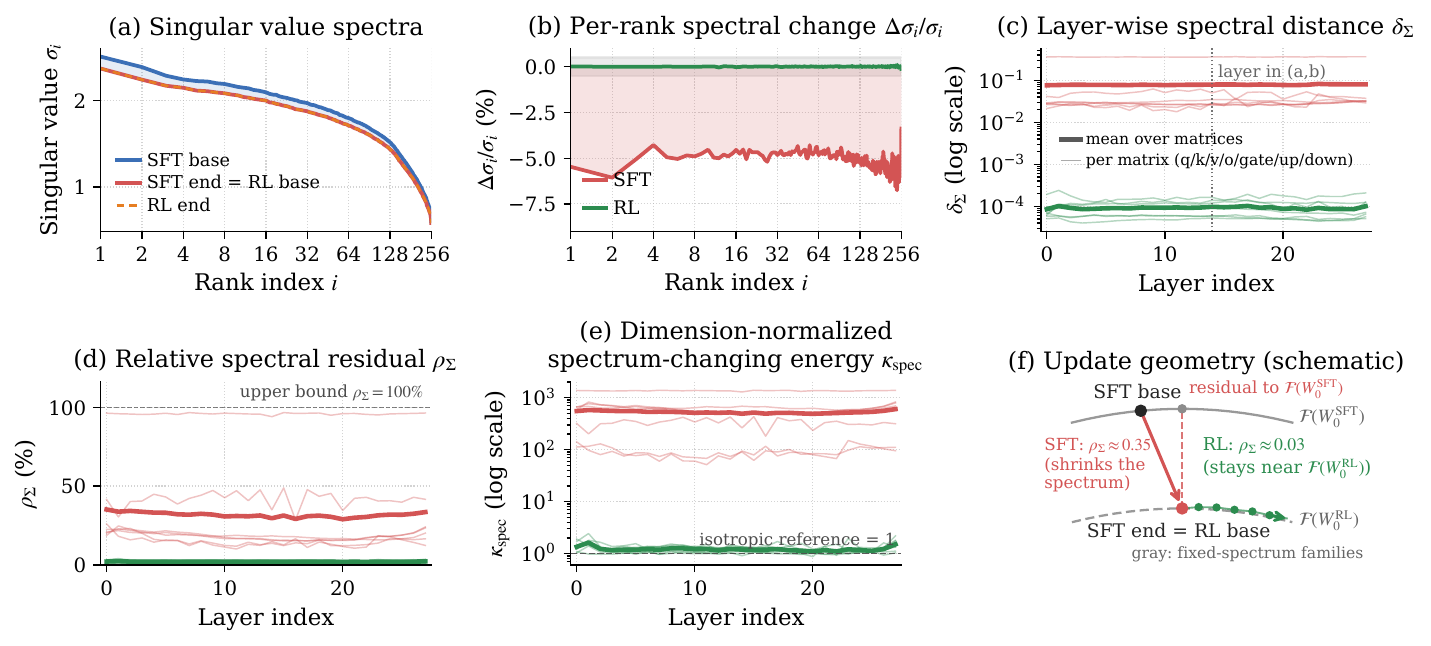}
    \caption{
    \textbf{Spectral dynamics under sequential SFT and RLVR.}
    \textbf{(a)} Representative singular-value profiles
    (\texttt{v\_proj}, layer 14); the SFT endpoint is the RLVR base.
    \textbf{(b)} Rank-wise relative change: SFT reshapes the spectrum,
    whereas RLVR remains close to its incoming values.
    \textbf{(c)} Base-normalized spectral distance $\delta_\Sigma$:
    the RLVR endpoint is two to three orders of magnitude closer
    to its base fixed-spectrum family than the illustrative SFT
    endpoint is to its own.
    \textbf{(d)} Mean relative spectral residuals are approximately 2\% for RLVR and 32\% for SFT.
    \textbf{(e)} Dimension-normalized spectrum-changing energy
    $\kappa_{\mathrm{spec}}$
    (Equation~\ref{eq:kappa_spec}): RLVR remains order-one,
    whereas SFT lies far above the isotropic reference.
    \textbf{(f)} Schematic endpoint geometry.
    The straight arrow denotes the SFT endpoint displacement,
    not a continuous optimization path.
    }
    \label{fig:spectrum_stability}
\end{figure}

\begin{figure}[t]
    \centering
    \includegraphics[width=\linewidth]
        {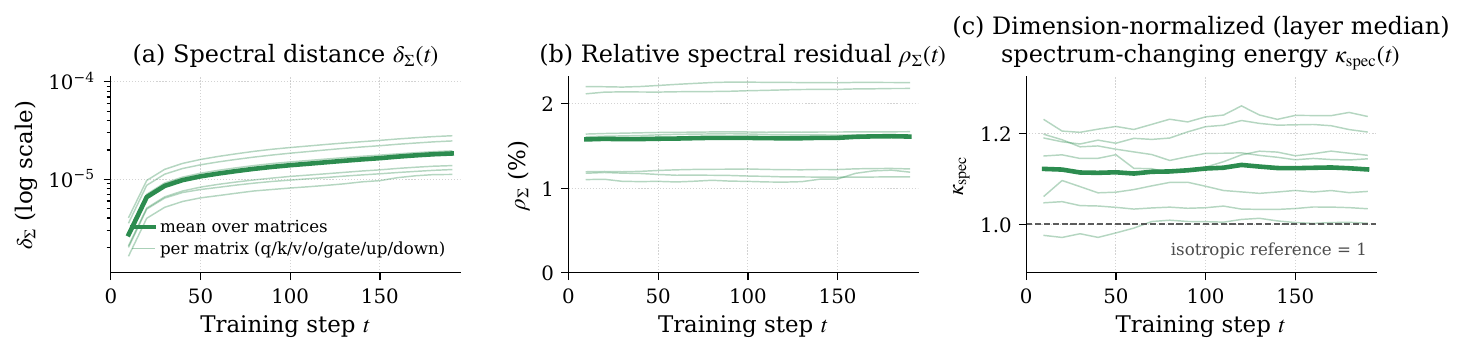}
    \caption{
    \textbf{Near-isospectrality at sampled training checkpoints.}
    AdamW on Qwen3-4B-Base, sampled every 10 actor updates through update 190.
    \textbf{(a)} The mean base-normalized spectral distance
    $\delta_\Sigma(t)$ stays below $2\times10^{-5}$.
    \textbf{(b)} The mean relative spectral residual
    $\bar{\rho}_\Sigma(t)$ is approximately $1.6\%$.
    \textbf{(c)} The dimension-normalized spectrum-changing
    energy $\kappa_{\mathrm{spec}}(t)$ remains order-one
    relative to the isotropic reference.
    }
    \label{fig:trajectory}
\end{figure}

\subsection{A Precise Observation: Spectra Stay}
\label{sec:spectral_stability}

We first make ``\textbf{spectra stay}'' precise by measuring the distance from a
learned weight matrix $W$ to the fixed-spectrum family
$\mathcal F(W_0)$ of its base matrix $W_0$.

\begin{proposition}[Exact distance to the fixed-spectrum family]
\label{prop:base_spectrum_projection}
For any
$W,W_0\in\mathbb{R}^{d_{\mathrm{out}}\times d_{\mathrm{in}}}$,
\begin{equation}
\label{eq:exact_spectrum_distance}
\small
    \operatorname{dist}_F
    \bigl(
        W,\mathcal F(W_0)
    \bigr)
    =
    \|\sigma(W)-\sigma(W_0)\|_2.
\end{equation}
Moreover, if
$W=U\Sigma V^\top$ and
$\Sigma_0=\operatorname{Diag}(\sigma(W_0))$, then
\begin{equation}
\label{eq:nearest_base_spectrum_point}
\small
    U\Sigma_0V^\top
    \in
    \operatorname*{arg\,min}_{Z\in\mathcal F(W_0)}
    \|W-Z\|_F.
\end{equation}
\end{proposition}
\noindent
Proposition~\ref{prop:base_spectrum_projection} gives singular-value drift a
precise checkpoint-level meaning: it is the distance from the learned matrix
to $\mathcal F(W_0)$, the family of matrices sharing the corresponding base
spectrum. The proposition also identifies a closest representative with that spectrum,
which motivates the functional intervention in
Section~\ref{sec:spectral_inheritance_tests}.
The proof and the treatment of repeated singular values are given in
Appendix~\ref{app:fixed_spectrum_distance}.

We normalize the distance by the base-weight norm and the full
checkpoint displacement:
\begin{equation}
\small
\label{eq:spectrum_distance_metrics}
\delta_\Sigma
:=\|\sigma(W)-\sigma(W_0)\|_2 / \|W_0\|_F,
\qquad
\rho_\Sigma
:=\|\sigma(W)-\sigma(W_0)\|_2 / \|W-W_0\|_F.
\end{equation}
These ratios require nonzero denominators; aggregation is
specified in Appendix~\ref{app:spectral_diagnostics}.

\paragraph{Endpoint and training-time evidence.}
We analyze a released reasoning RLVR endpoint trained for over
3,000 updates~\cite{liu2025prorl}, with
\texttt{DeepSeek-R1-Distill-Qwen-1.5B} (DS-1.5B) as its base.
Across the analyzed layers, $\delta_\Sigma\approx10^{-2}\%$,
and the reported pooled relative residual is
$\rho_\Sigma^{\mathrm{pool}}=0.03$.
Thus, the endpoint is close to the base fixed-spectrum families
relative to both the base-weight scale and the overall displacement.
An independent unconstrained AdamW run, sampled every 10 actor
updates, shows the same qualitative pattern at intermediate
checkpoints
(Appendix~\ref{app:spectral_evidence};
Fig.~\ref{fig:spectrum_stability} and Fig.~\ref{fig:trajectory}).
We call this descriptive proximity \emph{near-isospectrality}.

\paragraph{Calibration: proximity versus preference.}
Near-isospectrality does not by itself establish preferential
suppression of spectrum-changing motion.
We therefore measure the dimension-normalized first-order
spectrum-changing energy $\kappa_{\mathrm{spec}}$, whose isotropic
reference is one
(Appendix~\ref{app:spectral_calibration}).
In the DS-1.5B case, matrix-type means range from $1.02$ to $1.35$
for RLVR, versus $89$ to $1364$ for the preceding SFT transition.
Thus, the diagnostic reveals no strong additional suppression in
RLVR relative to this reference, while the SFT--RLVR contrast
remains pronounced.
Whether the remaining spectral changes are functionally
necessary is a separate question, which we test next.

\subsection{Spectral Inheritance: Reusing the Base Model's Spectrum}
\label{sec:spectral_inheritance_tests}

The dimensional calibration changes the interpretation of the initial
observation without overturning it. RLVR remains close to
$\mathcal F(W_0)$, but this proximity alone does not reveal a strong preference
for spectrum-preserving updates. The key functional question is whether the
small spectral changes that do occur are needed for the acquired behavior.

We address this question through two complementary interventions.

\paragraph{Restoring the base spectrum after training.}
For an RLVR endpoint~\cite{liu2025prorl},
$W_{\mathrm{RL}}
=U_{\mathrm{RL}}\Sigma_{\mathrm{RL}}V_{\mathrm{RL}}^\top$,
we retain its learned frames and interpolate only the spectrum:
\begin{equation}
\small
\label{eq:spectrum_interpolation}
\widetilde W(\alpha)
=U_{\mathrm{RL}}
 \bigl[(1-\alpha)\Sigma_0+\alpha\Sigma_{\mathrm{RL}}\bigr]
 V_{\mathrm{RL}}^\top,
\qquad \alpha\in[0,1].
\end{equation}
At $\alpha=0$, Proposition~\ref{prop:base_spectrum_projection}
identifies $\widetilde W(0)$ as a nearest representative in
$\mathcal F(W_0)$; at $\alpha=1$, the original RL checkpoint
is recovered.
Restoration preserves most acquired gains
(Fig.~\ref{fig:spectrum_functional_tests}(a)).
The converse substitution, placing the RL spectrum in the base
frames, does not yield comparable gains
(Fig.~\ref{fig:sft_rebase}).

\paragraph{Learning with the base spectrum fixed.}
For each transformed matrix, we train $W=U\Sigma_0V^\top$ from
the base checkpoint, optimizing $(U,V)$ while maintaining their
orthonormality.
This parameterization acquires strong reasoning gains and
outperforms AdamW in the illustrated run
(Fig.~\ref{fig:spectrum_functional_tests}(b)).
By contrast, training only $\Sigma$ with fixed base frames, a
restricted control with $q$ trainable scalars per matrix, yields
limited improvement, indicating that spectral rescaling alone
does not recover comparable gains.

Together, these interventions support \emph{spectral inheritance}:
the base spectra are functionally reusable for both retaining
acquired gains and supporting new learning.

\begin{tcolorbox}[
    colback=orange!10!white,
    colframe=black,
    boxrule=0.9pt,
    boxsep=2pt,
    top=3pt,
    bottom=3pt,
    left=3pt,
    right=3pt
]
\textit{\textbf{Take-away 1: Spectral inheritance.}}
Unconstrained RLVR checkpoints remain close to the fixed-spectrum families of
their base weights. More importantly, the interventions show that RLVR can
reuse this spectral structure rather than having to rewrite it, a property we
call \textbf{\emph{spectral inheritance}}.
\end{tcolorbox}

\begin{figure}[t]
    \centering
    \begin{subfigure}[t]{0.48\linewidth}
        \centering
        \includegraphics[width=\linewidth]
            {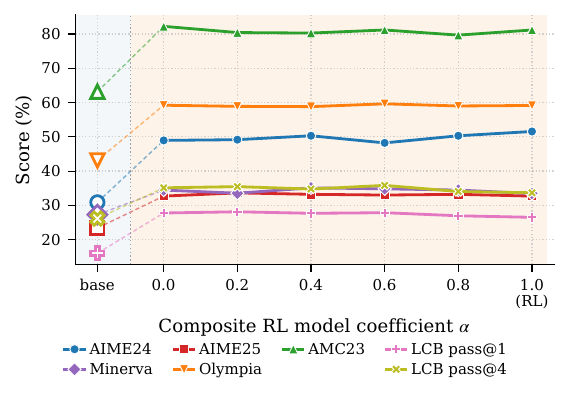}
        \caption{Spectrum restoration after training.}
    \end{subfigure}
    \hfill
    \begin{subfigure}[t]{0.48\linewidth}
        \centering
        \includegraphics[width=\linewidth]
            {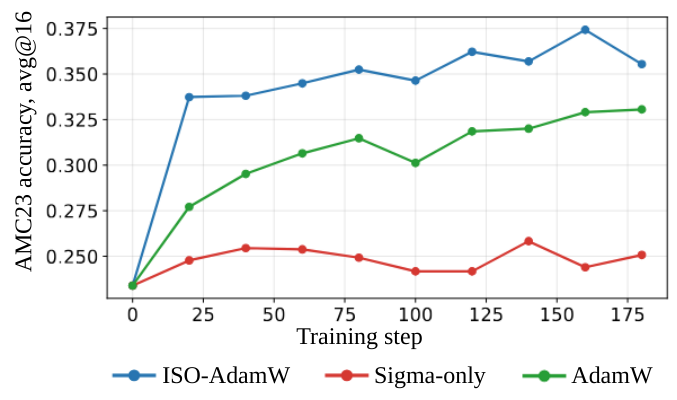}
        \caption{Fixed-spectrum and spectrum-only training.}
    \end{subfigure}
    \caption{\small
    \textbf{Functional evidence for spectral inheritance.}
    \textbf{(a)} Restoring the base spectrum while retaining
    RL-trained frames preserves most acquired gains.
    \textbf{(b)} Fixed-spectrum frame training on a 1.7B model supports strong
    gains, whereas the restricted spectrum-only control yields
    limited improvement under the tested recipe.
    }
    \label{fig:spectrum_functional_tests}
\end{figure}

\section{Frames Move: What Changes Under Spectral Inheritance?}
\label{sec:transport}

\paragraph{Structure: both frames must remain adaptable.}
Spectral inheritance identifies what can remain fixed, but not which variables
must remain adaptable. \textit{Could the endpoint be explained by a simpler
transformation that either remixes only within the incoming input and output
spans or retains one incoming singular subspace while leaving the other side free?}
We find that these simpler alternatives leave substantially more of the
checkpoint update unexplained than reusing the incoming spectrum while adapting
both frames. Thus, among the transformation classes tested, a low-residual
fixed-spectrum description requires both singular frames to remain adaptable.

\paragraph{Q1: Can a more conservative transformation explain the endpoint?}
For a transition $W_j\rightarrow W_i$, we refer to $W_j$ as the
\emph{incoming checkpoint} and to $W_i$ as the \emph{learned endpoint}.
Let $q=\min\{d_{\mathrm{out}},d_{\mathrm{in}}\}$ and consider an admissible
top-$r$ truncation
\begin{equation}
\small
\label{eq:transport_truncation}
    W_t^{(r)}
    =
    U_t^{(r)}\Sigma_t^{(r)}
    \bigl(V_t^{(r)}\bigr)^\top,
    \qquad
    P_t^{(r)}
    =
    U_t^{(r)}\bigl(U_t^{(r)}\bigr)^\top,
    \qquad
    Q_t^{(r)}
    =
    V_t^{(r)}\bigl(V_t^{(r)}\bigr)^\top .
\end{equation}
where $t\in\{i,j\}$.
Admissibility requires a positive rank-$r$ boundary gap at both endpoints, so
that the retained projectors are well-defined. We use
$r=\lfloor0.9q\rfloor$ in the main text and report rank sensitivity in
Appendix~\ref{app:rank}.

Define the rank-$r$ fixed-spectrum family associated with the incoming
checkpoint:
\begin{equation}
\label{eq:truncated_fixed_spectrum_family}
\small
    \mathcal F_j^{(r)}
    :=
    \left\{
        U\Sigma_j^{(r)}V^\top:
        U\in\mathrm{St}(d_{\mathrm{out}},r),
        \;
        V\in\mathrm{St}(d_{\mathrm{in}},r)
    \right\}.
\end{equation}
This family retains the incoming singular values $\Sigma_j^{(r)}$ while
allowing both singular frames to vary.

\begin{proposition}[Optimal reconstructions and unexplained-update ratios]
\label{prop:checkpoint_transport_test}
For an admissible transition $W_j\rightarrow W_i$, the Frobenius-optimal
reconstructions under the four structural restrictions are
\begin{equation}
\label{eq:transport_reconstructions}
\small
\begin{array}{lll}
\widehat W_{\mathrm{mix}}
&=
P_j^{(r)}W_i^{(r)}Q_j^{(r)},
&\text{remix within both incoming subspaces},\\[2pt]
\widehat W_L
&=
P_j^{(r)}W_i^{(r)},
&\text{retain the incoming output subspace},\\[2pt]
\widehat W_R
&=
W_i^{(r)}Q_j^{(r)},
&\text{retain the incoming input subspace},\\[2pt]
\widehat W_{\mathrm{iso}}
&=
U_i^{(r)}
\Sigma_j^{(r)}
\bigl(V_i^{(r)}\bigr)^\top,
&\text{retain the incoming spectrum and adapt both frames}.
\end{array}
\end{equation}

For
$W_i^{(r)}\neq W_j^{(r)}$, define the
\emph{unexplained-update ratio}
\begin{equation}
\label{eq:unexplained_update_ratio}
\small
    u_h
    :=
    \frac{
        \|W_i^{(r)}-\widehat W_h\|_F
    }{
        \|W_i^{(r)}-W_j^{(r)}\|_F
    },
    \qquad
    h\in\{\mathrm{mix},L,R,\mathrm{iso}\}.
\end{equation}
Because $W_j^{(r)}$ is feasible under all four restrictions, optimality gives
$0\leq u_h\leq1$.

Thus, $u_h=0$ denotes exact reconstruction, whereas $u_h=1$ means that the
best reconstruction under restriction $h$ is no closer to the endpoint than
the unchanged incoming checkpoint.

Moreover,
\begin{equation}
\label{eq:iso_transport_residual}
\small
    \widehat W_{\mathrm{iso}}
    \in
    \operatorname*{arg\,min}_{Z\in\mathcal F_j^{(r)}}
    \|W_i^{(r)}-Z\|_F,
    \qquad
    u_{\mathrm{iso}}
    =
    \frac{
        \|\Sigma_i^{(r)}-\Sigma_j^{(r)}\|_F
    }{
        \|W_i^{(r)}-W_j^{(r)}\|_F
    }.
\end{equation}
\end{proposition}

The remix reconstruction can be written as
\begin{equation}
\label{eq:incoming_subspace_remix}
\small
    \widehat W_{\mathrm{mix}}
    =
    U_j^{(r)}
    B_{ij}^{(r)}
    \bigl(V_j^{(r)}\bigr)^\top,
    \qquad
    B_{ij}^{(r)}
    :=
    \bigl(U_j^{(r)}\bigr)^\top
    W_i^{(r)}
    V_j^{(r)}.
\end{equation}
The core $B_{ij}^{(r)}$ is unconstrained and may rotate, mix, rescale, and
change the represented spectrum. Thus, this class fixes only the incoming input
and output spans. The one-sided classes are similarly permissive: they retain
one incoming span while allowing the remaining mapping and spectrum to refit
freely. Their $u_h$ values are therefore optimistic lower bounds on the
reconstruction error incurred by freezing the corresponding frame structure.
We use $u_h$ to test parameterization sufficiency. Formal class definitions and
proofs are provided in Appendix~\ref{app:frame_adaptability}.

\paragraph{Q2: Does the same structure recur under a new RL objective?}
A second question is whether this structure is specific to a single RL stage
and objective, or whether it reappears when a new RL stage starts from an
already RL-trained checkpoint and targets a distinct capability.

\paragraph{A sequential objective-shift stress test.}
We evaluate the parameterization test on a two-stage vision-language-action RL
pipeline initialized from
\texttt{Qwen2.5-VL-3B-Instruct}~\cite{yuan2025embodied}.
Stage~I optimizes spatial reasoning and produces $W_1$.
Stage~II starts from this already RL-trained checkpoint, optimizes an
embodied-manipulation objective, and produces $W_2$.
The resulting sequence provides three linked transitions:
\begin{equation}
\label{eq:vla_transitions}
\small
    W_0\rightarrow W_1,
    \qquad
    W_1\rightarrow W_2,
    \qquad
    W_0\rightarrow W_2.
\end{equation}
The transition $W_0\rightarrow W_1$ provides the first-stage reference.
The decisive objective-shift test is $W_1\rightarrow W_2$: it asks whether the
same frame-adaptability requirement reappears after re-anchoring at an already
RL-trained checkpoint and changing the RL objective.
The cumulative transition $W_0\rightarrow W_2$ tests consistency across the
complete two-stage sequence. We analyze both language and vision modules.

\begin{figure}[t]
    \centering
    \includegraphics[width=\linewidth]
        {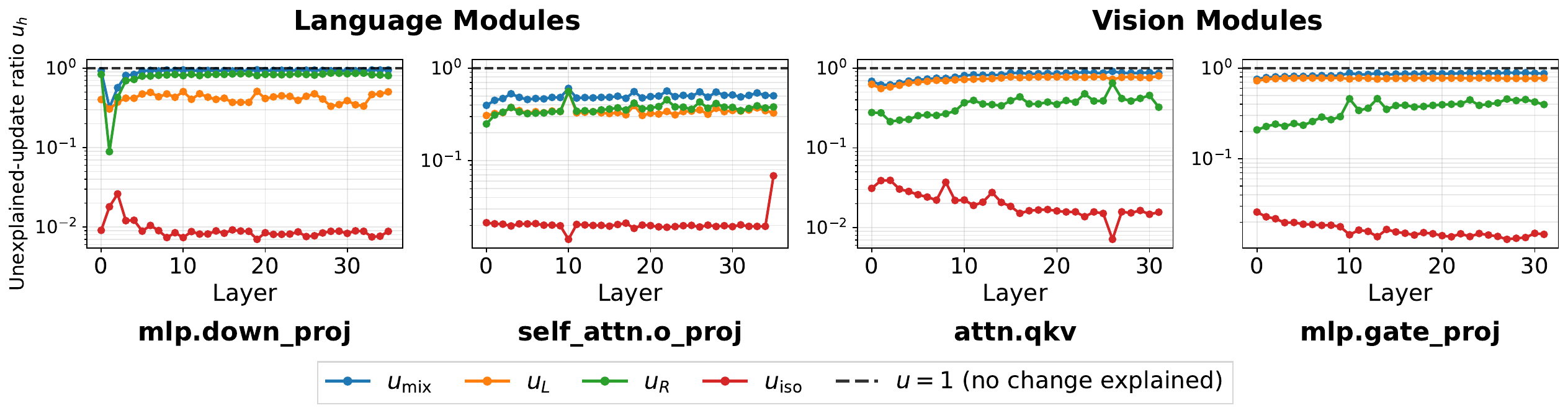}
    \caption{\small
    \textbf{Two-sided frame adaptability in endpoint reconstruction.}
    Unexplained-update ratios for $W_0\rightarrow W_2$ at
    $r=\lfloor0.9q\rfloor$, across layers of representative
    language and vision modules.
    Retaining one or both incoming singular subspaces leaves
    substantially larger residuals than retaining the incoming
    spectrum while adapting both frames.
    The dashed line $u_h=1$ marks the error of the unchanged
    incoming checkpoint.
    }
    \label{fig:frame_adaptability}
    \vspace{-15pt}
\end{figure}

\paragraph{Results: both frames must remain adaptable.}
For the cumulative transition $W_0\rightarrow W_2$,
Figure~\ref{fig:frame_adaptability} shows that an optimal remix within both
incoming subspaces leaves a median $87\%$ of the checkpoint update
unexplained. Retaining only the incoming output or input subspace still leaves
$45\%$ and $42\%$, respectively. These are optimistic residuals because the
corresponding classes may freely refit their remaining core and spectrum.

By contrast, retaining the incoming spectrum while adapting both frames leaves
a median residual of only $1.8\%$. Thus, among the tested structural
restrictions, neither incoming singular subspace can be frozen while retaining
a low-residual endpoint description: the spectrum can remain fixed, but both
frame variables must remain adaptable. 

The same qualitative conclusion holds for the key transition
$W_1\rightarrow W_2$, where the incoming checkpoint has already undergone RL
and Stage~II optimizes a distinct embodied-manipulation objective.
Thus, the frame-adaptability requirement is not confined to the first RL
stage. It reappears after an objective shift and re-anchoring at an RL-trained
checkpoint. Results for all transitions and truncation levels are reported in
Appendix~\ref{app:rank}. This retrospective test concerns endpoint representation, not the optimization trajectory or necessity for task performance.

\begin{tcolorbox}[
    colback=orange!10!white,
    colframe=black,
    boxrule=0.9pt,
    boxsep=2pt,
    top=3pt,
    bottom=3pt,
    left=3pt,
    right=3pt
]
\textit{\textbf{Take-away 2: Fix the spectrum, but keep both frames
adaptable.}}
Even permissive remix and one-sided reconstruction classes leave much of the
checkpoint update unexplained. Thus, among the tested restrictions, the
spectrum can remain fixed but both frame variables must remain adaptable,
directly motivating ISO's parameterization
$W(U,V)=U\Sigma_0V^\top$.
\end{tcolorbox}

\section{From Spectral Inheritance to Isospectral Optimization (ISO)}
\label{sec:iso}

Sections~\ref{sec:dynamics}--\ref{sec:transport} motivate a
separation between reusable spectra and adaptable singular frames.
We adopt this separation as an inductive bias,
\emph{inherit the spectrum, optimize the frames}, rather than
as an exact update law for unconstrained RLVR.
\textbf{Isospectral Optimization (ISO)} operationalizes this
principle through a shared parameterization for offline expert
composition and online RLVR optimization.

\paragraph{A shared fixed-spectrum parameterization.}
For a base matrix $W_0=U_0\Sigma_0V_0^\top$, ISO uses
\begin{equation}
\small
\label{eq:iso_parameterization}
    W(U,V)=U\Sigma_0V^\top,
    \qquad
    U\in\mathrm{St}(d_{\mathrm{out}},q),\quad
    V\in\mathrm{St}(d_{\mathrm{in}},q),
\end{equation}
where $q=\min\{d_{\mathrm{out}},d_{\mathrm{in}}\}$.
Unlike the rank-$r$ diagnostic in Section~\ref{sec:transport},
this parameterization retains all $q$ singular modes.
Composition and optimization act on $(U,V)$, while $\Sigma_0$
is inherited from the base.
Frame orthonormality ensures
$W(U,V)\in\mathcal F(W_0)$ in exact arithmetic.
\begin{proposition}[First-order spectrum preservation]
\label{prop:iso_tangent}
Let
$W=U\Sigma_0V^\top\in\mathcal F(W_0)$,
where the singular values in $\Sigma_0$ are positive and simple.
For any differentiable curve
$W(t)\in\mathcal F(W_0)$
with $W(0)=W$ and tangent
\[
    \dot W
    :=
    \left.
    \frac{d}{dt}W(t)
    \right|_{t=0},
\]
we have 
\begin{equation}
\label{eq:fixed_spectrum_first_order_condition}
\small
    \operatorname{diag}
    \bigl(
        U^\top\dot W V
    \bigr)
    =
    0.
\end{equation}
By contrast, for an arbitrary perturbation $H$,
\begin{equation}
\label{eq:singular_value_directional_derivative}
\small
    \left.
    \frac{d}{d\varepsilon}
    \sigma_k(W+\varepsilon H)
    \right|_{\varepsilon=0}
    =
    u_k^\top H v_k,
    \qquad
    k=1,\ldots,q.
\end{equation}
\end{proposition}

For Stiefel tangent directions
    $\xi_U
    \in
    T_U\mathrm{St}(d_{\mathrm{out}},q),
    \
    \xi_V
    \in
    T_V\mathrm{St}(d_{\mathrm{in}},q)$,
differentiating Equation~\ref{eq:iso_parameterization} gives the induced
weight-space direction
\begin{equation}
\label{eq:iso_weight_tangent}
\small
    \dot W
    =
    \xi_U\Sigma_0V^\top
    +
    U\Sigma_0\xi_V^\top.
\end{equation}
Because
$U^\top\xi_U$
and
$V^\top\xi_V$
are skew-symmetric,
Equation~\ref{eq:iso_weight_tangent} satisfies
Equation~\ref{eq:fixed_spectrum_first_order_condition}.
Thus, feasible frame motion preserves the spectrum to first
order, while orthonormal reconstruction preserves it exactly
(proof in Appendix~\ref{app:proof_iso_tangent}).
The reconstruction map also provides ISO-Merger's local effect
proxies through linearization and ISO-Optimizer's factor
gradients through the chain rule.

\subsection{ISO-Merger: Merging RL Experts Without Rollouts}
\label{sec:iso_merger}
\paragraph{Data-free composition of shared-base RL experts.}
Training a single RLVR policy across heterogeneous domains can be expensive
and unstable. A modular alternative is to train several specialists from a
shared base and then consolidate their capabilities into a single policy.
One common consolidation strategy is an additional on-policy distillation
stage driven by online rollouts~\cite{lu2025onpolicydistillation,zeng2026glm}.
ISO-Merger instead addresses the strictly checkpoint-only setting: it requires
no post-merge data, rollout generation, gradient updates, or distillation. It
pursues the same broad objective as on-policy distillation: recovering
complementary specialist capabilities in a single policy.

We focus on specialists trained from the same base checkpoint, with each
expert targeting a distinct domain or capability. Merging already-generalist
policies trained on broad, substantially overlapping mixtures constitutes a
different setting that we do not study here.

\paragraph{Setup and shared coordinates.}
For each base weight matrix $W_0=U_0\Sigma_0V_0^\top$, let
$W_i=U_i\Sigma_iV_i^\top$, $i=1,\ldots,K$, be the corresponding
matrices from $K$ specialists trained from the same base.
All singular values are ordered nonincreasingly.
Rather than combining Euclidean task vectors $W_i-W_0$,
ISO-Merger represents each expert by
$U_i\Sigma_0V_i^\top\in\mathcal F(W_0)$ and composes its
frame changes at the common anchor $(U_0,V_0)$.

\paragraph{Aligned tangent directions.}
Singular-vector pairs admit a joint sign ambiguity.
For each expert and column $k$, we set
$s_{i,k}=\operatorname{sign}(\langle u_{0,k},u_{i,k}\rangle)$,
with $\operatorname{sign}(0)=1$, and apply
$(u_{i,k},v_{i,k})\leftarrow
(s_{i,k}u_{i,k},s_{i,k}v_{i,k})$.
This leaves $W_i$ unchanged while removing artificial joint-sign
differences from the frame displacements.
For an orthonormal anchor $X$, the Frobenius-orthogonal
Stiefel tangent projection is
\begin{equation}
\small
\label{eq:iso_main_tangent_projection}
    \Pi_X(Z)=Z-X\operatorname{sym}(X^\top Z),
    \qquad
    \operatorname{sym}(A)=\tfrac12(A+A^\top).
\end{equation}
Using the aligned frames, we obtain
\begin{equation}
\small
\label{eq:iso_merger_tangent_directions}
    \xi_{U,i}=\Pi_{U_0}(U_i-U_0),
    \qquad
    \xi_{V,i}=\Pi_{V_0}(V_i-V_0).
\end{equation}
The common anchor makes these directions directly composable
in shared frame coordinates.
To limit interference from trailing modes, we zero all but the
leading $\operatorname{round}(\rho_{\mathrm{keep}}q)$ columns
of each projected displacement, obtaining
$\widetilde\xi_{U,i}$ and $\widetilde\xi_{V,i}$.
We use $\rho_{\mathrm{keep}}=0.9$ unless otherwise stated.
Masking acts on frame displacements, not on $\Sigma_0$,
and need not preserve tangent feasibility; we restore it
after aggregation.

\paragraph{Retention-aware composition.}
To account for overlap between expert directions, we linearize
the reconstruction map at $(U_0,V_0)$ and form local
first-order effect proxies
\begin{equation}
\label{eq:iso_merger_effect_proxy}
    g_i=\widetilde\xi_{U,i}\Sigma_0V_0^\top
        +U_0\Sigma_0\widetilde\xi_{V,i}^\top.
\end{equation}
Their Gram matrix
$\Gamma_{ij}=\langle g_i,g_j\rangle_F$
measures pairwise interactions.
For the merged proxy $g(c)=\sum_{i=1}^{K}c_i g_i$, we target
unit self-retention: its projection coefficient along each
nonzero $g_i$ should satisfy
$\langle g(c),g_i\rangle_F/\|g_i\|_F^2=1$.
The ideal target is therefore
$\Gamma c=\operatorname{diag}(\Gamma)$.

Let $\mathcal I=\{i:\|g_i\|_F>0\}$.
We set coefficients outside $\mathcal I$ to zero and, when
$\mathcal I\neq\varnothing$, solve and clip the active system:
\begin{equation}
\small
\label{eq:iso_merger_retention_system}
\begin{aligned}
    \bigl(\Gamma_{\mathcal I\mathcal I}
        +\lambda_{\mathrm{ridge}}I_{|\mathcal I|}\bigr)
        \bar c_{\mathcal I}
    &=\operatorname{diag}(\Gamma_{\mathcal I\mathcal I}),\\
    c^\star_{\mathcal I}
    &=\operatorname{clip}
        (\bar c_{\mathcal I};c_{\min},c_{\max}).
\end{aligned}
\end{equation}
Ridge stabilization handles an ill-conditioned Gram matrix,
while clipping limits extreme coefficients.
Unit self-retention is a local proxy-space target, not a
guarantee of capability retention.

\paragraph{Retraction and reconstruction.}
We reproject the weighted aggregate onto each base tangent space:
$\xi_{U,\star}=\Pi_{U_0}(\sum_i c_i^\star\widetilde\xi_{U,i})$,
and analogously obtain $\xi_{V,\star}$.
For a thin SVD $X=P_XS_XQ_X^\top$, the polar factor
$\operatorname{polar}(X)=P_XQ_X^\top$ has orthonormal columns.
We retract the updated frames and reconstruct the merged weight:
\begin{equation}
\label{eq:iso_merger_reconstruction}
\begin{aligned}
    U_\star&=\operatorname{polar}(U_0+\xi_{U,\star}),
    \qquad
    V_\star=\operatorname{polar}(V_0+\xi_{V,\star}),\\
    W_\star&=U_\star\Sigma_0V_\star^\top.
\end{aligned}
\end{equation}
The resulting matrix shares the base spectrum up to numerical
precision.
We apply this construction to 2D projection, embedding, and
unembedding matrices; one-dimensional parameters use an
arithmetic average across experts.
Alignment, projection, masking, coefficient settings, and
retraction details, together with the complete algorithm,
are provided in Appendix~\ref{app:iso_merger_details}.

\subsection{ISO-Optimizer: Online Fixed-Spectrum RLVR Training}
\label{sec:iso_optimizer}

ISO-Optimizer defines a generic transformation
$\mathrm{Opt}\mapsto\mathrm{ISO}[\mathrm{Opt}]$:
the base optimizer supplies the update equations and state
machinery, while ISO supplies the fixed-spectrum parameterization
and feasibility restoration.
For each transformed matrix, we initialize
$(U,V)=(U_0,V_0)$ and train $W=U\Sigma_0V^\top$ with
$\Sigma_0$ fixed.
Using AdamW or Muon yields ISO-AdamW or ISO-Muon, respectively.

\paragraph{Factor gradients and tentative updates.}
Let $G_W=\nabla_W\mathcal L(W)$ be the gradient of the RLVR loss
at $W=U\Sigma_0V^\top$.
The chain rule gives the Euclidean factor gradients
\begin{equation}
\small
\label{eq:iso_factor_gradients}
    G_U=G_WV\Sigma_0,
    \qquad
    G_V=G_W^\top U\Sigma_0.
\end{equation}

Let $\mathrm{Opt}(X,G,s)$ denote a base-optimizer update of
variable $X$ with gradient $G$ and state $s$.
We apply the update rule to both frames using separate states:
\begin{equation}
\small
\label{eq:iso_optimizer_updates}
\begin{aligned}
    (\bar U,s_U^+)&=\mathrm{Opt}(U,G_U,s_U),\\
    (\bar V,s_V^+)&=\mathrm{Opt}(V,G_V,s_V).
\end{aligned}
\end{equation}
For example, ISO-AdamW maintains first- and second-moment
buffers for each frame rather than for the dense weight.
The base optimizer receives the ambient factor gradients;
the tentative factors need not remain orthonormal.

\paragraph{Feasibility restoration.}
We apply the polar map and reconstruct the weight:
\begin{equation}
\small
\label{eq:iso_optimizer_reconstruction}
\begin{aligned}
    U^+&=\operatorname{polar}(\bar U),
    \qquad
    V^+=\operatorname{polar}(\bar V),\\
    W^+&=U^+\Sigma_0(V^+)^\top.
\end{aligned}
\end{equation}
Thus, $W^+\in\mathcal F(W_0)$ in exact arithmetic.
Both the update rule and optimizer states operate in frame
coordinates: ISO is not a weight-space optimizer step followed
by projection onto $\mathcal F(W_0)$.
Although it introduces two factor variables, it does not enlarge
the represented weight-space class; the reconstructed weights
remain constrained to the base fixed-spectrum family.
The factor parameterization and optimizer-state dynamics are
both part of the method.

\paragraph{Implementation overhead.}
The main additional structured operation in ISO-Optimizer is the polar
retraction applied after each factor update.
The additional operations
also include factor-gradient construction, factor updates, and weight reconstruction. We implement it using an FP64
SVD-based $\operatorname{polar}$ computation for numerical stability. Following a DION-style distributed
implementation~\cite{ahn2025dion}, we distribute the matrix retractions across multiple GPUs.
In our RLVR profiling, the additional optimizer cost increases end-to-end step
time by approximately $1.2\%$, as reported in
Section~\ref{sec:exp_optimizer}. 

The additional storage remains manageable in our sharded RLVR
setting: on \texttt{Qwen3-4B-Base}, ISO-AdamW increases model-wide
moment memory by 31\%, while the measured peak-memory increase
under identical 16-way FSDP is 1.8\,GB per GPU, including factor
storage (Appendix~\ref{app:training_details}). 

Despite the added overhead, ISO uses $1.5\times$ fewer updates and reduces wall-clock time to matched accuracy by 34\% (Section~\ref{sec:exp_optimizer}).

\section{Experiments}
\label{sec:experiments}
We evaluate ISO through its two complementary instantiations:
ISO-Merger for recovering specialist capabilities through
checkpoint-only composition, and ISO-Optimizer for improving
accuracy and accelerating RLVR training.

\subsection{ISO-Merger: Composing Heterogeneous RL Experts Without Rollouts}

\definecolor{expertcolor}{HTML}{F2F7FF} 
\definecolor{ourcolor}{HTML}{E9F5E9}    

\begin{table*}[h]
\centering
\small
\renewcommand{\arraystretch}{1.2}
\setlength{\tabcolsep}{4pt}
\caption{
\label{tab:great_merge1}
\small
\textbf{Merging three RL experts on Qwen2.5-7B-Instruct.}
Coding reports pass accuracy (ACC) and unit-test accuracy (UT) on LiveBench (LB) and LiveCodeBench (LCB).
Tool use follows BFCL v4 with Live and Non-Live averages; memory uses RULER's recurrent setting at 32K–64K context lengths for HotpotQA and SQuAD.
Avg.\ is the mean of the ten metric columns.
Bold marks column-wise best scores separately among experts
and merged models.
The final row reports ISO-Merger's score as a percentage of the
corresponding domain expert's score, using the reported means.
Shading indicates \colorbox{expertcolor}{Experts} and \colorbox{ourcolor}{Ours}. 
}
\resizebox{\textwidth}{!}{
\begin{tabular}{l cccc cc cccc c}
\toprule
& \multicolumn{4}{c}{Coding} & \multicolumn{2}{c}{Tool Use} & \multicolumn{4}{c}{Memory} & \\
\cmidrule(lr){2-5} \cmidrule(lr){6-7} \cmidrule(lr){8-11}
& \multicolumn{2}{c}{LB} & \multicolumn{2}{c}{LCB v2} & Live & Non-Live & \multicolumn{2}{c}{HotpotQA} & \multicolumn{2}{c}{SQuAD} & \\
\cmidrule(lr){2-3} \cmidrule(lr){4-5} \cmidrule(lr){6-6} \cmidrule(lr){7-7} \cmidrule(lr){8-9} \cmidrule(lr){10-11}
Method & ACC & UT & ACC & UT & ACC & ACC & 32K & 64K & 32K & 64K & Avg \\
\midrule
Base~\cite{qwen2.5} & 36.25 $\pm$ 0.86 & 48.27 $\pm$ 1.44 & 28.31 $\pm$ 0.51 & 43.08 $\pm$ 0.82 & 62.96 & 69.26 & 49.87 $\pm$ 1.45 & 45.54 $\pm$ 0.58 & 58.90 $\pm$ 1.07 & 56.77 $\pm$ 0.61 & 49.92 \\
\rowcolor{expertcolor}
RLVR-Coder~\cite{wang2025code} & \textbf{37.42 $\pm$ 1.22} & 49.42 $\pm$ 1.60 & \textbf{30.17 $\pm$ 0.99} & \textbf{44.24 $\pm$ 0.99} & 63.37 & 68.07 & 50.21 $\pm$ 0.16 & 45.61 $\pm$ 0.33 & 60.51 $\pm$ 0.70 & 56.84 $\pm$ 0.42 & 50.59 \\
\rowcolor{expertcolor}
RLVR-Tool~\cite{qian2025toolrl} & 36.87 $\pm$ 1.79 & 48.08 $\pm$ 0.82 & 28.20 $\pm$ 1.95 & 42.43 $\pm$ 1.90 & \textbf{71.83} & \textbf{81.82} & 50.39 $\pm$ 0.68 & 45.33 $\pm$ 0.26 & 59.73 $\pm$ 1.18 & 57.86 $\pm$ 0.52 & 52.25 \\
\rowcolor{expertcolor}
RLVR-Memory~\cite{yu2025memagent} & 37.32 $\pm$ 1.37 & \textbf{50.05 $\pm$ 1.61} & 27.41 $\pm$ 2.99 & 41.70 $\pm$ 4.63 & 63.82 & 72.61 & \textbf{78.60 $\pm$ 0.54} & \textbf{77.95 $\pm$ 0.29} & \textbf{79.98 $\pm$ 0.38} & \textbf{78.52 $\pm$ 0.22} & \textbf{60.80} \\
\midrule
Task Arithmetic~\cite{ta} & 38.85 $\pm$ 0.70 & 52.18 $\pm$ 0.40 & 31.55 $\pm$ 0.18 & 47.06 $\pm$ 0.58 & \textbf{73.73} & \textbf{82.86} & 72.22 $\pm$ 0.22 & 70.49 $\pm$ 0.10 & 75.03 $\pm$ 0.29 & 74.95 $\pm$ 0.32 & 61.89 \\
TIES~\cite{ties} & 40.30 $\pm$ 2.18 & 52.04 $\pm$ 2.50 & 31.46 $\pm$ 0.35 & 47.33 $\pm$ 0.30 & 68.24 & 76.54 & 76.42 $\pm$ 0.36 & 75.24 $\pm$ 0.14 & 78.27 $\pm$ 0.46 & 77.86 $\pm$ 0.59 & 62.37 \\
TSV~\cite{tsv} & 36.96 $\pm$ 0.86 & 49.23 $\pm$ 0.59 & 28.92 $\pm$ 1.79 & 43.37 $\pm$ 3.25 & 73.64 & 82.69 & 75.37 $\pm$ 0.29 & 74.10 $\pm$ 0.30 & 78.30 $\pm$ 0.17 & 77.26 $\pm$ 0.78 & 61.98 \\
RAM~\cite{ram} & 38.59 $\pm$ 0.79 & 50.38 $\pm$ 1.53 & 31.31 $\pm$ 0.28 & 46.33 $\pm$ 0.30 & 72.56 & 81.89 & 76.17 $\pm$ 0.20 & 75.29 $\pm$ 0.49 & 78.42 $\pm$ 0.47 & 77.83 $\pm$ 0.46 & 62.88 \\
OrthoMerge-G-TIES~\cite{yang2026orthogonal}  & 40.74 $\pm$ 2.00 & \textbf{53.22 $\pm$ 1.90} & \textbf{31.98 $\pm$ 0.84} & \textbf{47.48 $\pm$ 0.77} & 66.98 & 78.17 & 74.79 $\pm$ 0.49 & 74.20 $\pm$ 0.49 & 77.75 $\pm$ 0.12 & 76.32 $\pm$ 0.21  & 62.16\\
\midrule
\rowcolor{ourcolor}ISO-Merger & \textbf{41.81 $\pm$ 2.00} & 52.73 $\pm$ 1.57 & 31.06 $\pm$ 1.88 & 45.69 $\pm$ 3.06 & 72.32 & 81.07 & \textbf{79.46 $\pm$ 0.78} & \textbf{76.77 $\pm$ 0.27} & \textbf{79.10 $\pm$ 0.14} & \textbf{78.01 $\pm$ 0.32} & \textbf{63.80} \\
\rowcolor{ourcolor}
\quad\textit{Expert-relative score (\%)} & 111.7 & 106.7 & 102.9 & 103.3 & 100.7 & 99.1 & 101.1 & 98.5 & 98.9 & 99.4 & -- \\
\bottomrule
\end{tabular}
}
\end{table*}

\begin{table}[ht]
\centering
\caption{
\label{tab:great_merge2}
\small
\textbf{Main results on DeepSeek-R1-Distill-Qwen-1.5B with 2 RL experts}. Coding: pass accuracy (ACC) and unit-test accuracy (UT) on LiveBench (LB) and LiveCodeBench (LCB). Math: ACC on AIME 2024, AIME 2025, AMC 2023, Minerva, and OlympiadBench. Avg is the average over all columns. Best expert/merged result per column in \textbf{bold}. Shading indicates \colorbox{expertcolor}{Experts} and \colorbox{ourcolor}{Ours}. The final row reports ISO-Merger's score as a percentage of the corresponding domain expert's score. Ratios are computed from the reported means, with $100\%$ indicating an equal score.
}
\resizebox{\textwidth}{!}{
\begin{tabular}{lcccccccccc}
\toprule
& \multicolumn{4}{c}{Coding} & \multicolumn{5}{c}{Math} & \\
\cmidrule(lr){2-5} \cmidrule(lr){6-10}
& \multicolumn{2}{c}{LB} & \multicolumn{2}{c}{LCB v5} & AIME 2024 & AIME 2025 & AMC 2023 & Minerva & OlympiadBench & \multirow{2}{*}{Avg} \\
\cmidrule(lr){2-3} \cmidrule(lr){4-5} \cmidrule(lr){6-6} \cmidrule(lr){7-7} \cmidrule(lr){8-8} \cmidrule(lr){9-9} \cmidrule(lr){10-10}
Method & ACC & UT & ACC & UT & Avg@32 & Avg@32 & Avg@8 & Avg@4 & Avg@4 & \\
\midrule
Base~\cite{deepseek_r1} & 17.99 $\pm$ 0.29 & 26.24 $\pm$ 0.53 & 16.82 $\pm$ 0.28 & 20.85 $\pm$ 0.18 & 30.90 $\pm$ 1.15 & 23.40 $\pm$ 0.86 & 63.15 $\pm$ 0.38 & 27.33 $\pm$ 0.43 & 43.11 $\pm$ 0.26 & 29.98 \\
\rowcolor{expertcolor}
Archer2.0~\cite{archer} & \textbf{26.66 $\pm$ 0.45} & \textbf{37.26 $\pm$ 0.55} & \textbf{26.90 $\pm$ 0.21} & \textbf{38.99 $\pm$ 0.37} & 42.05 $\pm$ 0.43 & 28.02 $\pm$ 0.44 & 73.44 $\pm$ 0.50 & 30.09 $\pm$ 0.23 & 49.99 $\pm$ 0.65 & 39.27 \\
\rowcolor{expertcolor}
JustRL~\cite{he2025justrl} & 22.23 $\pm$ 0.23 & 31.85 $\pm$ 0.74 & 23.05 $\pm$ 0.21 & 30.18 $\pm$ 0.31 & \textbf{53.16 $\pm$ 0.44} & \textbf{36.84 $\pm$ 0.73} & \textbf{82.78 $\pm$ 0.51} & \textbf{34.71 $\pm$ 0.50} & \textbf{55.67 $\pm$ 0.42} & \textbf{41.16} \\
\midrule
Task Arithmetic~\cite{ta} & 26.22 $\pm$ 0.22 & 35.82 $\pm$ 0.11 & 26.70 $\pm$ 0.21 & 36.58 $\pm$ 0.32 & 50.83 $\pm$ 0.39 & 33.09 $\pm$ 1.16 & 80.82 $\pm$ 0.51 & 33.00 $\pm$ 0.34 & 54.31 $\pm$ 0.15 & 41.93 \\
TIES~\cite{ties} & 26.37 $\pm$ 0.73 & 36.75 $\pm$ 0.90 & 27.62 $\pm$ 0.13 & 39.45 $\pm$ 0.15 & 54.51 $\pm$ 0.72 & 35.76 $\pm$ 0.30 & 82.13 $\pm$ 0.14 & 33.76 $\pm$ 0.48 & 55.36 $\pm$ 0.15 & 43.52 \\
TSV~\cite{tsv} & 26.40 $\pm$ 0.40 & 36.83 $\pm$ 0.15 & 27.61 $\pm$ 0.17 & 40.01 $\pm$ 0.26 & 53.16 $\pm$ 0.57 & 35.17 $\pm$ 0.72 & 80.97 $\pm$ 0.92 & 33.88 $\pm$ 0.46 & 54.68 $\pm$ 0.27 & 43.19 \\
RAM~\cite{ram} & \textbf{26.71 $\pm$ 0.29} & \textbf{36.99 $\pm$ 0.47} & 27.18 $\pm$ 0.63 & 38.83 $\pm$ 0.67 & 54.20 $\pm$ 0.55 & 35.80 $\pm$ 0.71 & 82.28 $\pm$ 0.38 & 34.10 $\pm$ 0.34 & 55.54 $\pm$ 0.25 & 43.51 \\
OrthoMerge-G-TIES~\cite{yang2026orthogonal}  & 26.46 $\pm$ 0.62 & 36.73 $\pm$ 0.37 & 27.74 $\pm$ 0.30 & 39.86 $\pm$ 0.64  & 54.55 $\pm$ 1.44 & 35.14 $\pm$ 0.87 & 81.88 $\pm$ 0.63 & 33.55 $\pm$ 0.42 & 55.47 $\pm$ 0.24 &  43.49 \\
\midrule
\rowcolor{ourcolor}ISO-Merger & 25.52 $\pm$ 0.08 & 36.42 $\pm$ 0.55 & \textbf{27.84 $\pm$ 0.13} & \textbf{41.84 $\pm$ 0.31} & \textbf{55.00 $\pm$ 0.52} & \textbf{37.81 $\pm$ 0.31} & \textbf{83.53 $\pm$ 1.02} & \textbf{34.77 $\pm$ 0.23} & \textbf{56.64 $\pm$ 0.05} & \textbf{44.38} \\
\rowcolor{ourcolor}
\quad\textit{Expert-relative score (\%)} & 95.7 & 97.7 & 103.5 & 107.3 & 103.5 & 102.6 & 100.9 & 100.2 & 101.7 & -- \\
\bottomrule
\end{tabular}
}
\end{table}

\paragraph{Setup.}
We evaluate ISO-Merger in its intended setting: composing multiple RLVR specialists trained from the same base model, without
post-merge data, rollouts, or distillation.
We consider two settings:
(1) three specialists for coding~\cite{wang2025code},
tool use~\cite{qian2025toolrl}, and long-context
memory~\cite{yu2025memagent}, initialized from
\texttt{Qwen2.5-7B-Instruct}~\cite{qwen2.5};
(2) two specialists for coding~\cite{archer} and
math~\cite{he2025justrl}, initialized from
\texttt{DeepSeek-R1-Distill-Qwen-1.5B} (DS-1.5B).
We compare against representative data-free merging baselines:
Task Arithmetic~\cite{ta}, TIES~\cite{ties},
TSV-Merge~\cite{tsv}, RAM~\cite{ram}, and
OrthoMerge-G-TIES~\cite{yang2026orthogonal}.
We evaluate immediately after merging, with no fine-tuning. Stochastic evaluations are repeated three times, with reported mean and standard deviation across runs. Unless otherwise specified, we report average@16.
Full evaluation protocols, baseline configurations, best@4/worst@4 results, and ablation studies are provided in
Appendix~\ref{app:merging_details}.

\paragraph{Results.}
ISO-Merger achieves the highest aggregate scores among the compared
data-free methods in both settings:
$63.80$ versus $62.88$ on Qwen2.5-7B
(Tab.~\ref{tab:great_merge1}), and
$44.38$ versus $43.52$ on DS-1.5B
(Tab.~\ref{tab:great_merge2}).
On Qwen2.5-7B, ISO-Merger achieves $106.2\%$, $99.8\%$, and
$99.5\%$ of the corresponding specialists' domain-average scores
on coding, tool use, and memory, respectively.
By comparison, Task Arithmetic reaches $101.9\%$ on tool use
but only $92.9\%$ on memory, highlighting ISO-Merger's more
balanced recovery of specialist performance.
On the applicable metrics, ISO-Merger also improves aggregate
worst@4 over the strongest baselines by $1.62$ and $1.36$
percentage points in the two settings, respectively, while
maintaining comparable aggregate best@4
(Appendix~\ref{app:iso_merger_best_worst}).
Together, these results support the practical value of spectral
inheritance for recovering complementary specialist capabilities
in a single model without post-merge training.
We further examine the contributions of spectral rebasing,
retention-aware composition, and mode masking
($\rho_{\mathrm{keep}}=0.9$) through ablations in
Appendix~\ref{app:iso_merger_ablate}.

\subsection{ISO-Optimizer: Improving RLVR Accuracy and Training Efficiency}
\label{sec:exp_optimizer}

\begin{table}[!t]
    \centering
    \small
\caption{
\small
\textbf{Math RLVR results.}
We report average@16 accuracy. 
}    \label{tab:math_main_results}
\vspace{-10pt}
    \resizebox{0.8\linewidth}{!}{
    \begin{tabular}{ll|ccccc|c}
        \toprule
        Model & Method & AIME24 & AIME25 & AMC23 & Minerva & Olympiad & Avg. \\
        \midrule

        \multirow{4}{*}{Qwen3-1.7B-Base}
        & Base  & 3.33 & 2.50 & 23.34 & 18.43 & 19.19 & 13.36 \\
        & AdamW & 14.38 & 10.21 & 44.58 & 33.02 & 38.80 & 28.20 \\
        & Muon  & 15.22 & 11.26 & \textbf{45.51} & 32.96 & 40.09 & 29.01 \\ 
        \rowcolor{ourcolor}
        & ISO-AdamW & \textbf{16.67} & \textbf{11.67} & 44.50 & \textbf{33.92} & \textbf{40.30} & \textbf{29.41} \\

        \midrule

        \multirow{4}{*}{Qwen3-4B-Base}
        & Base  & 5.00 & 3.96 & 26.58 & 20.43 & 21.95 & 15.58 \\
        & AdamW & 32.92 & 22.08 & 67.77 & 45.29 & 56.02 & 44.82 \\
        & Muon  & 32.29 & 25.00 & 66.34 & 46.12 & \textbf{57.82} & 45.51 \\
        \rowcolor{ourcolor}
        & ISO-AdamW & \textbf{33.54} & \textbf{28.13} & \textbf{68.60} & \textbf{46.97} & 57.64 & \textbf{46.98} \\

        \bottomrule
    \end{tabular}
    }
    \vspace{-10pt}
\end{table}

\begin{figure}[t]
\vspace{-5pt}
    \centering
    \begin{subfigure}{0.48\textwidth}
        \centering
        \caption{\small Qwen3-1.7B-Base}
        \vspace{-3pt}
        \includegraphics[width=0.9\linewidth]{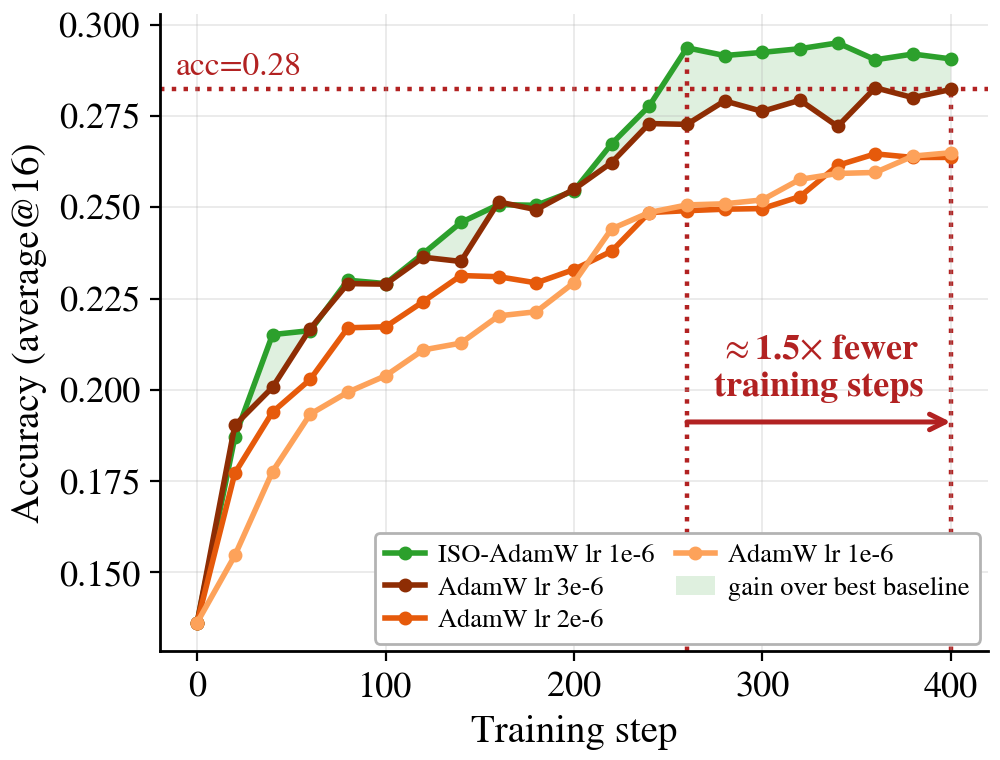}
        \label{fig:math_1.7b_adam}
    \end{subfigure}
    \hfill
    \begin{subfigure}{0.48\textwidth}
        \centering
        \caption{\small Qwen3-4B-Base}
        \vspace{-3pt}
        \includegraphics[width=0.9\linewidth]{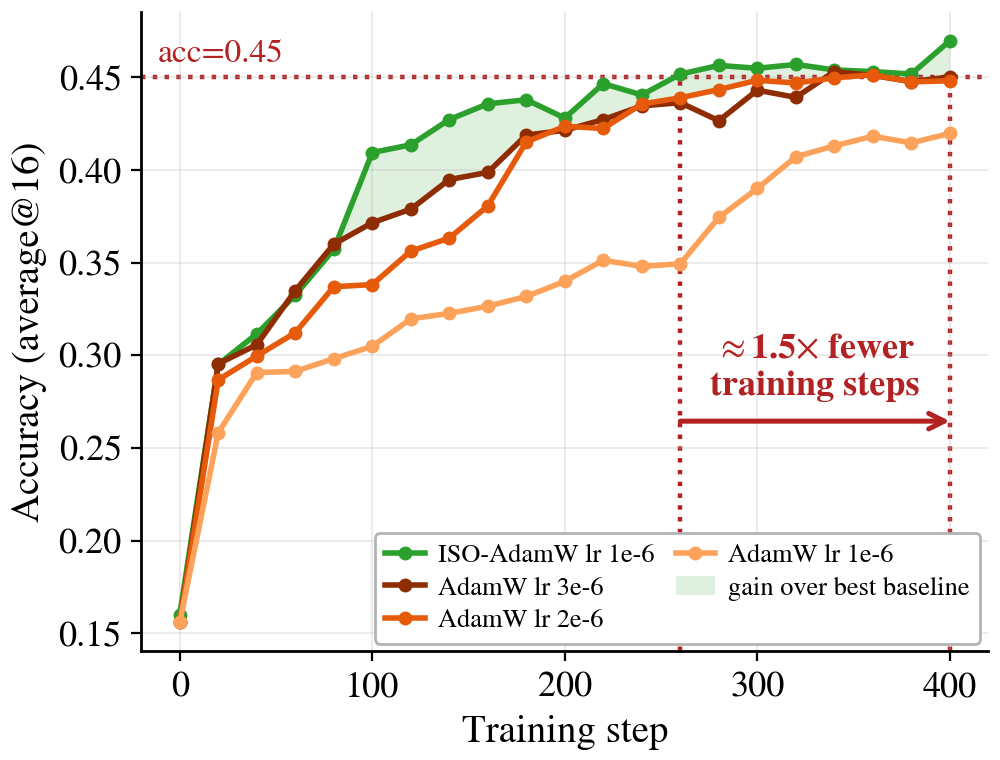}
        \label{fig:math_4b_adam}
    \end{subfigure}
    \vspace{-1em}
\caption{
\small
\textbf{Aggregate math accuracy during RLVR training.} ISO-AdamW uses a single learning rate of $1\times10^{-6}$, whereas the
weight-space AdamW baselines are tuned over the sweep described in the text. 
AdamW with a learning rate of $7.5\times10^{-7}$ performs strictly worse and is omitted from the figure for clarity.
}
\label{fig:math_adam_curves}
\vspace{-15pt}
\end{figure}

\paragraph{Setup.}
Motivated by spectral inheritance, ISO-Optimizer reuses the base
spectra and applies conventional optimizers directly to the
associated singular frames.
We evaluate accuracy and training efficiency on mathematical
reasoning and coding against the corresponding weight-space
optimizers.
For the main math experiments, we train
\texttt{Qwen3-1.7B-Base} and
\texttt{Qwen3-4B-Base}~\cite{qwen3technicalreport}
on \textsc{DeepMath}-103K~\cite{deepmath} for $400$ actor updates.
ISO-AdamW uses a frame learning rate of $10^{-6}$ at both scales.
We sweep AdamW learning rates over
$\{7.5\times10^{-7},10^{-6},2\times10^{-6},3\times10^{-6}\}$
and Muon learning rates over
$\{2.5\times10^{-5},5\times10^{-5},7.5\times10^{-5}\}$,
with an additional AdamW run at $5\times10^{-6}$ on 4B.
We evaluate AIME24, AIME25, AMC23, Minerva, and OlympiadBench.
Full configurations and evaluation protocols are provided in Appendix~\ref{app:training_details}.

\begin{wrapfigure}{r}{0.42\textwidth}
  \centering
  \includegraphics[width=\linewidth]{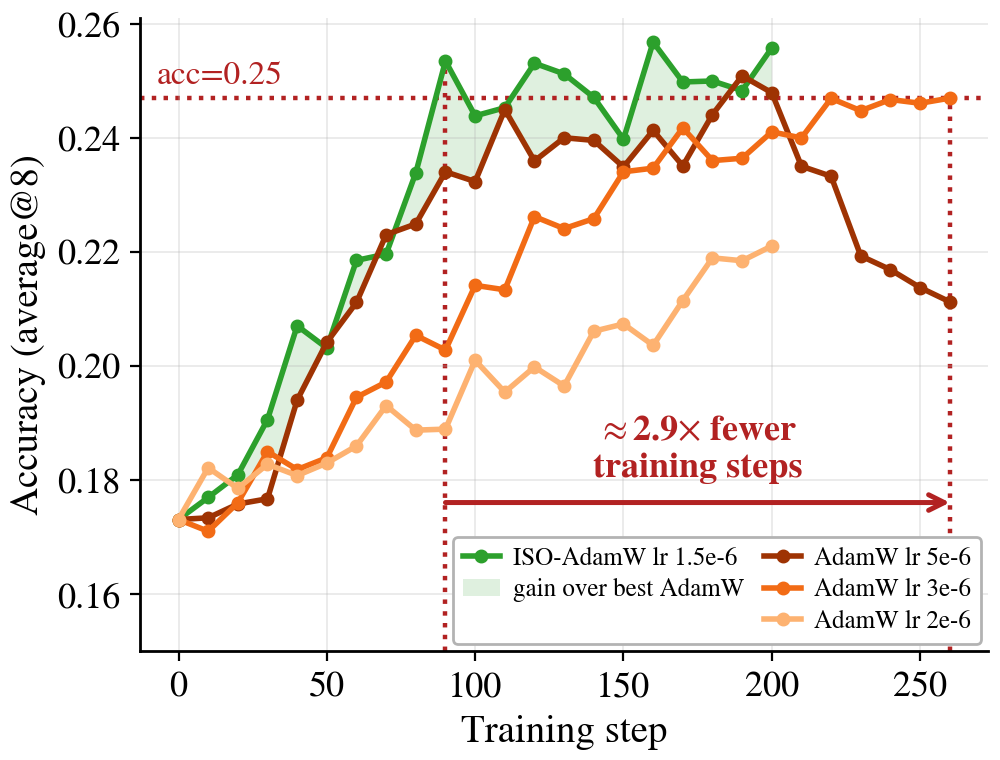}
  \vspace{-2em}
\caption{
\small
\textbf{Coding accuracy on DS-1.5B}.}
\vspace{-15pt}
\label{fig:curve_coding}
\end{wrapfigure}
\paragraph{Math results.}
ISO-AdamW achieves the highest aggregate average@16 accuracy at
both model scales (Tab.~\ref{tab:math_main_results}), improving
over the strongest tested AdamW baselines by $1.21$ and $2.16$ points on Qwen3-1.7B and Qwen3-4B, respectively.
It also exceeds Muon in aggregate and achieves the highest score
on four of the five benchmarks at each scale.
The learning curves further show that ISO-AdamW reaches the
strongest AdamW run's $400$-update accuracy in fewer actor updates at both scales (Fig.~\ref{fig:math_adam_curves}).
We also ran an additional AdamW experiment at a larger learning rate of $5\times10^{-6}$. It peaks early
and then declines on Qwen3-4B-Base (Appendix Fig.~\ref{fig:math_4b_5e6}).

\begin{wraptable}{r}{0.46\textwidth}
    \centering
    \vspace{-0.8em}
    \caption{
    \small \textbf{Coding results on DS-1.5B.}
    We report average@8 LiveCodeBench accuracy after 200 actor updates.
    }
    \label{tab:code_main_results}
    \resizebox{0.40\textwidth}{!}{
    \begin{tabular}{l|cc|c}
        \toprule
        Method & LCB v5 & LCB v6 & Avg. \\
        \midrule
        DS-1.5B & 16.31 & 18.32 & 17.31 \\
        AdamW   & 24.60 & 25.00 & 24.80 \\
        \rowcolor{ourcolor}
        ISO-AdamW     & \textbf{24.64} & \textbf{26.53} & \textbf{25.58} \\
        \bottomrule
    \end{tabular}
    }
    \vspace{-1em}
\end{wraptable}
\paragraph{Coding results.}
Beyond math, we train \texttt{DS-1.5B}
following \textsc{ArcherCodeR}~\cite{archer}. 
An initial AdamW sweep over $\{1,2,3\}\times10^{-6}$ favors larger learning rates, so we extend the sweep to $5\times10^{-6}$. We also increase ISO-AdamW’s frame learning rate from $10^{-6}$ to $1.5\times10^{-6}$ for this setting. To limit multi-epoching on the smaller coding dataset, we use 200 actor updates for the main comparison. ISO-AdamW achieves an average@8 of 25.58 across LiveCodeBench v5 and v6, versus 24.80 for the strongest AdamW baseline tested (Tab.~\ref{tab:code_main_results}). The learning curves show ISO-AdamW ahead at most evaluation checkpoints. Extending the AdamW runs at $3\times10^{-6}$ and $5\times10^{-6}$ to 260 updates still does not match ISO-AdamW’s peak of 25.7, reached within 200 updates (Fig.~\ref{fig:curve_coding}).

\begin{figure}[!t]
  \centering
  \includegraphics[width=0.8\linewidth]{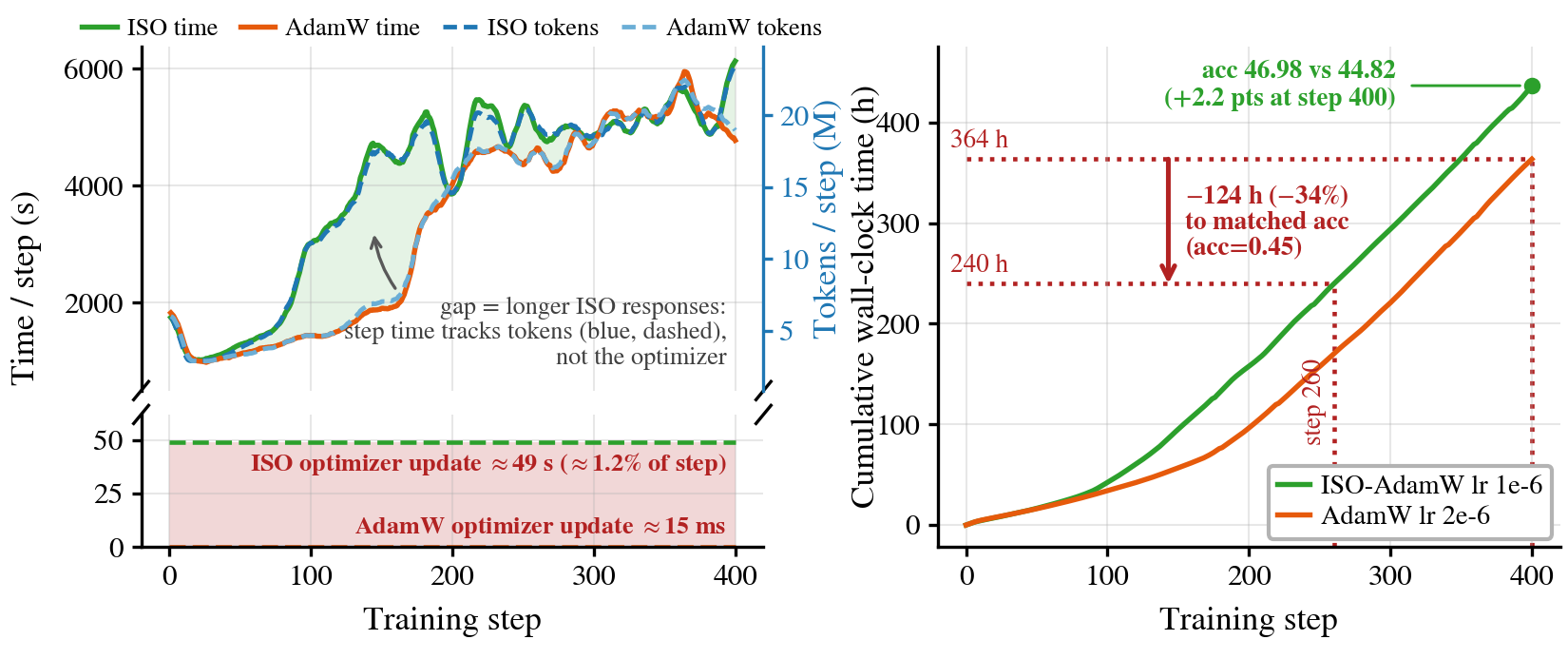}
\caption{
\small
\textbf{Wall-clock cost of ISO vs.\ AdamW on \texttt{Qwen3-4B-Base}} Left, top: per-step time (solid, left axis) and tokens per step (dashed blue, right axis); Left, bottom: the optimizer update alone. Right: cumulative wall-clock time\vspace{-1em}
}
\label{fig:iso_overhead}
\end{figure}

\paragraph{Wall-clock efficiency and overhead.}
We profile \texttt{Qwen3-4B-Base} using the ISO-AdamW ($1\times10^{-6}$) and AdamW ($2\times10^{-6}$)
runs in Fig.~\ref{fig:math_4b_adam}. Additional optimizer computation costs about 49 seconds per step, only 1.2\% of the profiled end-to-end step time (Fig.~\ref{fig:iso_overhead}, left). ISO-AdamW reaches AdamW's 400-update accuracy at update 260, saving 124h (34\%), and finishes 2.2 points higher at update 400 (Fig.~\ref{fig:iso_overhead}, right). Its roughly 20\% higher average step time mainly reflects an earlier increase in response length (Fig.~\ref{fig:iso_overhead}, left), consistent with faster learning rather than optimizer overhead. 
Thus, ISO-AdamW not only uses fewer actor updates, but also translates the reduction in actor updates into lower time to matched accuracy. Full profiling details are provided in Appendix~\ref{app:training_details}.

\begin{wrapfigure}{r}{0.42\textwidth}
  \vspace{-0.5em}
  \centering
  \includegraphics[width=\linewidth]{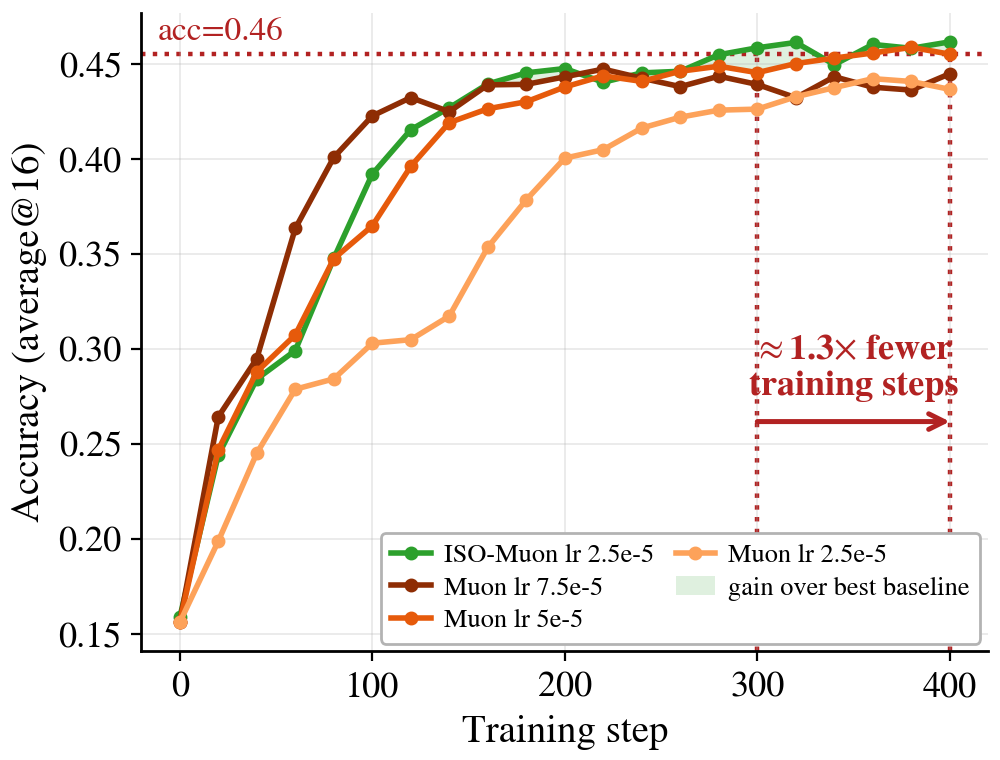}
  \vspace{-2em}
\caption{
\small
\textbf{ISO-Muon versus Muon on \texttt{Qwen3-4B-Base}.}
}
\vspace{-1em}
\label{fig:math_4b_muon}
\end{wrapfigure}
\paragraph{ISO-Muon variant.}
To test whether ISO's benefits extend beyond AdamW, we instantiate
ISO-Muon by applying the Muon update in fixed-spectrum Stiefel
coordinates.
On \texttt{Qwen3-4B-Base}, we use a learning rate of
$2.5\times10^{-5}$ for ISO-Muon and sweep weight-space Muon
over $\{2.5\times10^{-5},5\times10^{-5},7.5\times10^{-5}\}$
in Fig.~\ref{fig:math_4b_muon}.
Although Muon exhibits a late-stage accuracy decline at the
largest learning rate $7.5\times10^{-5}$, ISO-Muon's advantage also holds against the strongest tested Muon run: it reaches that run's
400-update accuracy by update $300$ and achieves a higher final aggregate score at update $400$ ($46.2$ versus $45.5$).
These results extend the improvements observed with ISO-AdamW
to a second base optimizer.

\begin{figure}[!b]
  \centering
  \includegraphics[width=0.8\linewidth]{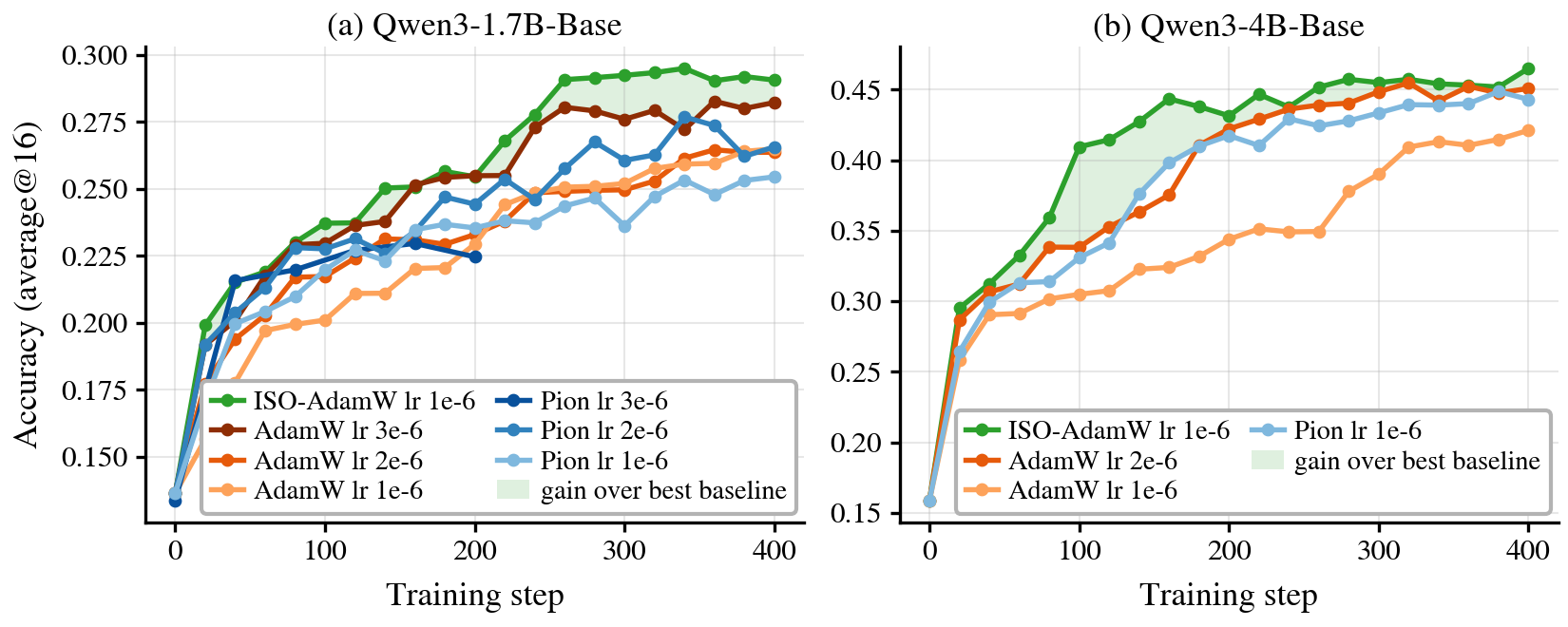}
\caption{
\small
\textbf{ISO-AdamW vs.\ Pion vs.\ AdamW on \texttt{Qwen3-1.7B-Base} (left) and \texttt{Qwen3-4B-Base} (right).}
}
\label{fig:pion_compare}
\end{figure}

\paragraph{Comparison with Pion.}
Pion~\cite{shi2026pion} is a concurrent spectrum-preserving
optimizer that updates $W$ through orthogonal rotations,
with its own momentum and step-size rules.
Its RLVR configuration alternates between left- and right-side
transformations.
ISO instead provides a generic optimizer transformation,
applying a chosen base optimizer's update rule directly to
both singular frames.
We use Pion's official implementation, sweeping learning rates
over $\{1,2,3\}\times10^{-6}$ on Qwen3-1.7B-Base (stopping the $3\times10^{-6}$ run early following an accuracy decline beginning at step 160).
Due to computational constraints, we use its published RLVR
learning rate of $1\times10^{-6}$ on Qwen3-4B-Base without further tuning.
Among completed $400$-update runs in Fig.~\ref{fig:pion_compare}, Pion outperforms the
$1\times10^{-6}$ AdamW baseline but not the strongest tested AdamW
configuration, while ISO-AdamW achieves the highest final
accuracy at both scales.
On 4B, ISO-AdamW reaches Pion's $400$-update accuracy using
approximately $40\%$ as many actor updates.

\paragraph{Ablation: Does learning the spectrum help?}
We compare ISO-AdamW with a variant that 
additionally trains
the diagonal spectrum $\Sigma$, initialized at $\Sigma_0$,
while retaining the same factorization $W=U\Sigma V^\top$
and orthonormality constraints on the frames.
On \texttt{Qwen3-1.7B-Base}, the fixed- and trainable-spectrum
variants exhibit similar learning trajectories, reaching aggregate
accuracies of $29.4$ and $29.6$, respectively, at update $400$
in Fig.~\ref{fig:learnable_sigma}.
The learned spectrum remains close to initialization, with a
relative drift of $0.0065\%$.
The comparable observed accuracy supports reusing the
\begin{wrapfigure}{r}{0.4\textwidth}
  \vspace{-1.5em}
  \centering
  \includegraphics[width=\linewidth]{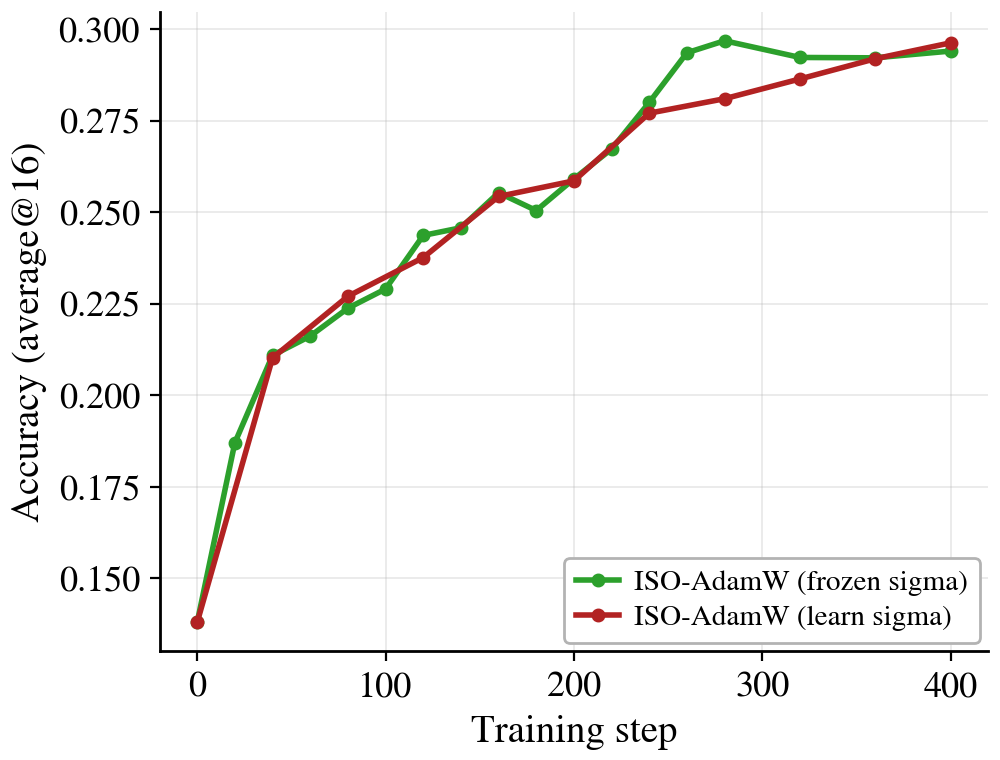}
  \vspace{-2em}
\caption{
\small
\textbf{ISO with \textbf{Frozen vs.\ trainable $\Sigma$} on \texttt{Qwen3-1.7B-Base}.}
\vspace{-1em}
}
  \label{fig:learnable_sigma}
\end{wrapfigure}
base spectrum, reinforcing the rationale for ISO's frame-coordinate
formulation: inherit the spectrum and directly optimize the
adaptable singular frames.

\paragraph{Ablation: Random-seed sensitivity.}
We train Qwen3-1.7B-Base with two independent seeds per method under a supplementary 8K-context configuration to reduce the computational cost, fixing the learning rate at $10^{-6}$ for both AdamW and ISO-AdamW. ISO-AdamW achieves aggregate accuracies of $27.6$ and $27.3$, compared with $24.4$ and $24.8$ for AdamW. The absolute run-to-run differences are $0.30$ and $0.40$ percentage points, respectively, under this fixed-learning-rate 8K configuration.
\section{Related Work}
\label{sec:related_works}

\paragraph{RLVR dynamics and optimization.}
Most advances in RLVR have focused on data and environments
, learning
objectives, or systems infrastructure.
By contrast, the optimization layer that translates reward feedback into
parameter-space motion remains comparatively underexplored and is typically
inherited from pre-training, including the choice of optimizer and
parameterization.
Recent analyses have begun to distinguish RLVR from pre-training and
supervised fine-tuning at both the policy and parameter levels.
At the policy level, RLVR has been characterized through KL-proximal and
conservative policy-improvement views; at the parameter level, its updates
have been reported to be sparse, off-principal, and strongly shared across
runs initialized from the same source
model~\cite{wu2025invisible,shenfeld2025rl,mukherjee2025reinforcement,zhu2025path}.
Building on the prior observation of limited spectral drift in
RLVR~\cite{zhu2025path}, we formalize and test \emph{spectral inheritance}: the
base model's singular-value spectra can remain functionally reusable while the
associated input and output frames adapt.
We then use this principle to guide RLVR optimization algorithm design.

\paragraph{Matrix-geometric optimization and constrained parameterizations.}
A growing line of work treats matrix-valued parameters and updates as
structured objects rather than as flattened Euclidean vectors. These methods
act on different objects and impose different forms of structure.
Muon orthogonalizes matrix-valued momentum updates, while DION develops
scalable distributed approximations to such orthonormalized
updates~\cite{liu2025muon,ahn2025dion}.
Hyperball constrains the Frobenius norms of both weight matrices and optimizer
updates, making the relative angular step size explicit rather than controlling
it indirectly through weight decay~\cite{wen2025fantastic}.
The Modular Manifolds perspective advocates co-designing module-specific
constraints and update norms, including Stiefel-constrained weights and a
manifold version of Muon~\cite{bernstein2025manifolds}.
POET and POET-X preserve weight spectra through two-sided
orthogonal-equivalence reparameterizations, with an emphasis on stable and
efficient LLM training~\cite{qiu2026reparameterized,qiu2026poet}.
Related structure has also been used for parameter-efficient adaptation:
Spectral Adapter modifies a leading spectral subspace, while StelLA learns
orthonormal input and output subspaces for a low-rank
adapter~\cite{zhang2024spectral,li2026stella}.

\paragraph{Comparison with concurrent spectrum-preserving optimization.}
An independent concurrent work, Pion~\cite{shi2026pion}, also preserves
weight spectra through left- and right-orthogonal transformations.
Pion updates each weight matrix directly through a dedicated
Lie-algebra-based rule with its own momentum, scale control, and
approximate exponential-map machinery, whereas ISO exposes the
singular frames as optimization variables and applies an existing base
optimizer, such as AdamW or Muon, to those variables, followed by
retraction.
Thus, although both methods operate within spectrum-preserving weight
families, they induce substantially different optimization dynamics.

The distinction is also empirical.
We port Pion into the same RLVR pipeline and match the training data,
rollout configuration, and evaluation protocol
(Section~\ref{sec:exp_optimizer}).
Across the completed Qwen3-1.7B and Qwen3-4B runs, the tested Pion
configurations underperform the strongest tuned AdamW baseline, whereas
ISO-AdamW achieves the highest final aggregate accuracy.
On Qwen3-4B, ISO-AdamW reaches Pion's final accuracy with approximately
$2.5\times$ fewer actor updates and finishes $2.2$ percentage points
higher.
These results indicate that spectrum preservation alone does not
determine optimization quality; how the fixed-spectrum geometry is
parameterized and optimized is consequential in RLVR.

More broadly, Pion starts from spectrum preservation as a general LLM
optimizer design principle, motivated by orthogonal-equivalence
geometry and training stability.
ISO instead starts from an RLVR-specific empirical and functional
characterization: we identify \emph{spectral inheritance}, calibrate
near-isospectrality against ambient dimensionality, and test reuse of
the source spectrum through functional interventions and sequential RL
stages.
ISO then promotes this evidence into a broader RLVR post-training
framework, with both an online optimizer transformation and an offline
checkpoint-only expert-composition method.

\paragraph{From matrix-geometric priors to an RLVR-native post-training stack.}
Our contribution is not the first use of orthogonality, manifolds, or
spectral constraints.
Rather, ISO follows an evidence-to-design route tailored to RLVR.
We first establish spectral inheritance through calibrated checkpoint
analysis and functional interventions, and then use it as an inductive
bias for two complementary stages of RLVR post-training.
Online, ISO transforms conventional base optimizers to operate in
fixed-spectrum frame coordinates; offline, the same geometry provides a
shared coordinate system for checkpoint-only composition of RL experts.
Thus, ISO-Optimizer and ISO-Merger are two realizations of the same
spectral-inheritance principle rather than independent algorithms.

\textit{\textbf{In summary, ISO contributes an evidence-to-stack design:
it identifies spectral inheritance from RLVR's own dynamics, validates its
functional relevance, and turns it into a common optimization principle
spanning online policy learning and offline expert consolidation.}}

\paragraph{Geometry-aware model merging.}
Most training-free merging methods combine Euclidean task vectors, with
techniques such as Task Arithmetic, TIES, and TSV-Merge reducing interference
through sign resolution or low-rank
structure~\cite{ta,ties,tsv}.
OrthoMerge instead merges orthogonal fine-tuning transformations in a Lie
algebra and extends to generally fine-tuned models by separating orthogonal
and residual components~\cite{yang2026orthogonal}.
ISO-Merger is specialized to shared-base RL experts: it projects each expert
to the source-spectrum geometry, represents its frame displacement in a common
base-anchored Stiefel tangent space, and performs checkpoint-only composition
without post-merge prompts, rollouts, teacher queries, or optimization.
\vspace{-1em}
\section{Conclusion}
\label{sec:conclusion}
\vspace{-1em}
We identify and functionally validate \emph{spectral inheritance}:
RLVR can reuse the base model's weight spectra while acquiring
new behavior through changes in both singular frames.
This separation motivates \emph{Isospectral Optimization} (ISO),
an RLVR-native framework for offline composition and online learning.
ISO-Merger combines shared-base specialists without post-merge
data or training, achieving the strongest aggregate performance
among the compared data-free methods.
ISO-Optimizer applies conventional optimizers to the frames,
improving reasoning and coding accuracy and reaching matched
accuracy in fewer actor updates.
Together, these results support a reusable design principle
for RLVR: \emph{inherit the spectrum, optimize the frames}.


\bibliography{reference/llm,reference/ref}
\bibliographystyle{unsrt}

\appendix

\newpage
\onecolumn
\appendix

\section*{Appendix Outline}
\addcontentsline{toc}{section}{Appendix Outline}

\begingroup
\small
The appendices are organized as follows. 
\begin{itemize}[
    leftmargin=*,
    itemsep=0.5em,
    topsep=0.5em,
    parsep=0pt
]
    \item \textbf{Appendix~\ref{app:spectral_diagnostics}: Spectral-Distance Metrics, Empirical Evidence, and Calibration.} Defines the spectral-distance diagnostics, reports endpoint and training-time evidence, and calibrates spectral motion against an isotropic dimensional reference.

    \item \textbf{Appendix~\ref{app:index_alignment}: Qualitative
    Index-Alignment Check.}
    Provides rank-index singular-direction sanity checks for the
    spectrum-restoration intervention.

    \item \textbf{Appendix~\ref{app:fixed_spectrum_theory}: Fixed-Spectrum
    Theory and Frame-Adaptability Diagnostics.}
    Establishes the fixed-spectrum distance and nearest representatives,
    derives the dimension-aware calibration, formalizes the reconstruction
    classes and unexplained-update ratios, and reports rank sensitivity.

    \item \textbf{Appendix~\ref{app:iso_first_order}: First-Order Properties
    of the ISO Parameterization.}
    Proves first-order spectrum preservation for feasible frame motion and
    derives the ISO factor gradients.

    \item \textbf{Appendix~\ref{app:14b_sft_rl}: A Same-Base SFT--RLVR Case
    Study.}
    Compares SFT and RLVR checkpoints derived from the same 14B backbone and
    reports an additional spectral-interpolation intervention.

    \item \textbf{Appendix~\ref{app:iso_merger_details}: ISO-Merger Details.}
    Specifies sign canonicalization, tangent projection, mode masking,
    retention coefficients, retraction, parameter scope, and the complete
    merging algorithm.

    \item \textbf{Appendix~\ref{app:merging_details}: Data-Free Merging
    Experimental Details and Additional Results.}
    Describes the experts, baselines, hyperparameters, evaluation protocols, additional experimental results and metrics, and ablation studies for ISO-Merger.

    \item \textbf{Appendix~\ref{sec:svd_precision_issue}: Numerical Precision of the SVD-Based Retraction.}
    Documents the SVD-based retraction, evaluates the accuracy--runtime
    trade-off of PyTorch SVD configurations, and motivates the FP64 GPU
    implementation.

    \item \textbf{Appendix~\ref{app:training_details}: Online RLVR Training
    Details.} Gives the shared setup and the mathematical-reasoning and competitive-coding
    RLVR configurations, optimizer settings, training budgets, and evaluation
    protocols.

\end{itemize}
\endgroup

\clearpage

\section*{Revision note}
We update the manuscript to include more detailed computational profiling,
additional comparisons with Pion, a concurrent spectrum-preserving
optimizer, and corrected ISO-AdamW results.
We found that the spectral learning-rate scaling used in the
ISO-Muon and Muon code paths had inadvertently also been enabled
for ISO-AdamW. This introduced an additional learning-rate
adjustment not applied to the AdamW baseline. We have revised
the affected results accordingly. 
We are investigating to understanding when and why this scaling benefits ISO-AdamW.

\section{Spectral-Distance Metrics, Empirical Evidence, and Calibration}
\label{app:spectral_diagnostics}

\paragraph{Distance metrics and aggregation.}
For each analyzed matrix $\ell$, let
$\Delta W^{(\ell)}=W_1^{(\ell)}-W_0^{(\ell)}$.
We evaluate the exact distance to the base fixed-spectrum family
at two scales:
\begin{equation}
\label{eq:app-spectrum_distance_metrics}
\begin{aligned}
    \delta_{\Sigma}^{(\ell)}
    &:=
    \frac{
        \|\sigma(W_1^{(\ell)})-\sigma(W_0^{(\ell)})\|_2
    }{
        \|W_0^{(\ell)}\|_F
    },
    \\
    \rho_{\Sigma}^{(\ell)}
    &:=
    \frac{
        \|\sigma(W_1^{(\ell)})-\sigma(W_0^{(\ell)})\|_2
    }{
        \|\Delta W^{(\ell)}\|_F
    }.
\end{aligned}
\end{equation}
Here $\delta_\Sigma$ measures proximity relative to the base-weight
scale, whereas $\rho_\Sigma$ measures the fraction of the full
checkpoint displacement left unexplained by a closest
fixed-spectrum representative.
The ratios are defined for nonzero denominators.
For a collection $\mathcal L$ of analyzed matrices, the pooled
relative residual is
\begin{equation}
\label{eq:pooled_spectrum_residual}
    \rho_{\Sigma}^{\mathrm{pool}}
    :=
    \left[
        \frac{
            \sum_{\ell\in\mathcal L}
            \|\sigma(W_1^{(\ell)})-\sigma(W_0^{(\ell)})\|_2^2
        }{
            \sum_{\ell\in\mathcal L}
            \|\Delta W^{(\ell)}\|_F^2
        }
    \right]^{1/2}.
\end{equation}
This pooled quantity is distinct from the arithmetic mean
of the per-matrix ratios.

\subsection{Endpoint and Training-Time Evidence}
\label{app:spectral_evidence}

\paragraph{Long-horizon endpoint.}
We examine whether spectral proximity persists after prolonged
RLVR using the released endpoint of a reasoning run initialized
from \texttt{DeepSeek-R1-Distill-Qwen-1.5B} (DS-1.5B) and trained
for over 3,000 updates~\cite{liu2025prorl,hu2025brorl}.
The checkpoint sequence starts from
\nolinkurl{Qwen2.5-Math-1.5B}, proceeds through an SFT transition to
\texttt{DS-1.5B}, and then through RLVR to
\nolinkurl{Nemotron-Research-Reasoning-Qwen-1.5B}.
Thus, the SFT endpoint serves as the base for the subsequent
RLVR stage.
These are sequential transitions rather than an SFT--RLVR
comparison from a shared initialization.
This case complements the broader evidence across models,
datasets, RL objectives, and training horizons reported
by~\cite{zhu2025path}.

Fig.~\ref{fig:spectrum_stability} shows that the RLVR endpoint's
spectrum nearly overlaps that of its incoming base, whereas the
preceding SFT transition produces substantial spectral contraction.
Across RLVR layers, $\delta_{\Sigma}^{(\ell)}$ is approximately
$10^{-2}\%$, and $\rho_{\Sigma}^{\mathrm{pool}}=0.03$.
Over the analyzed matrices, the pooled distance to the base
fixed-spectrum families is therefore only $3\%$ of the full
base-to-RL displacement.

\paragraph{Sampled training checkpoints.}
Endpoint proximity alone cannot exclude an intermediate spectral
excursion that subsequently cancels.
We therefore analyze an independent unconstrained Qwen3-4B-Base
AdamW run, using checkpoints saved every 10 actor updates through
update 190.
At each saved checkpoint, we measure proximity to the same
initial fixed-spectrum family $\mathcal F(W_0)$.

Fig.~\ref{fig:trajectory} shows that both
$\delta_\Sigma(t)$ and $\rho_\Sigma(t)$ remain small over the
observed sequence.
The mean base-normalized distance stays below $2\times10^{-5}$, while
the reported mean relative residual is approximately $1.6\%$.
At the resolution of the saved checkpoints, spectral stability
therefore persists during training rather than appearing only
at the final endpoint.

\subsection{Dimension-Aware Calibration}
\label{app:spectral_calibration}

\paragraph{Why dimensional calibration is needed.}
Small spectral drift does not by itself establish an optimization
preference for spectrum-preserving motion.
In the generic full-rank, simple-spectrum case, only
$q=\min\{d_{\mathrm{out}},d_{\mathrm{in}}\}$ independent
first-order coordinates change the singular values within the
$d_{\mathrm{out}}d_{\mathrm{in}}$-dimensional matrix space.
At $W_0=U_0\Sigma_0V_0^\top$, these coordinates for a displacement
$\Delta W$ are
$\operatorname{diag}(U_0^\top\Delta W V_0)$.
Consequently, even an isotropically oriented displacement
places only a small fraction of its squared norm in
spectrum-changing coordinates.

\paragraph{Normalized spectrum-changing energy.}
We account for this dimensional effect using
\begin{equation}
\label{eq:kappa_spec}
    \kappa_{\mathrm{spec}}^{(\ell)}
    :=
    \frac{d_{\mathrm{out}}d_{\mathrm{in}}}{q}
    \frac{
        \left\|
            \operatorname{diag}
            \left(
                U_0^{(\ell)\top}
                \Delta W^{(\ell)}
                V_0^{(\ell)}
            \right)
        \right\|_2^2
    }{
        \|\Delta W^{(\ell)}\|_F^2
    },
\end{equation}
where the dimensions and $q$ refer to matrix $\ell$.
Under the full-rank, simple-spectrum assumptions, the matrices
$\{u_{0,k}v_{0,k}^\top\}_{k=1}^q$ form an orthonormal basis
for the first-order spectrum-changing subspace.
An isotropic direction places expected squared energy in this
subspace in proportion to its dimension:
\begin{equation}
\label{eq:isotropic_spectrum_fraction}
    \mathbb{E}\left[
        \frac{
            \|\operatorname{diag}(U_0^\top\Delta W V_0)\|_2^2
        }{
            \|\Delta W\|_F^2
        }
    \right]
    =
    \frac{q}{d_{\mathrm{out}}d_{\mathrm{in}}}.
\end{equation}
Hence $\mathbb{E}[\kappa_{\mathrm{spec}}]=1$ under the
isotropic reference.
Values $\kappa_{\mathrm{spec}}\ll1$ indicate additional
suppression of first-order spectrum-changing motion;
order-one values do not reveal such a strong preference;
and $\kappa_{\mathrm{spec}}\gg1$ indicates concentration
in these coordinates.

\paragraph{Results and interpretation.}
Across matrix types in the public DS-1.5B reasoning run,
RLVR yields mean $\kappa_{\mathrm{spec}}$ values between
$1.02$ and $1.35$
(Fig.~\ref{fig:spectrum_stability}(e)).
The same order-one pattern persists along the sampled
Qwen3-4B-Base trajectory
(Fig.~\ref{fig:trajectory}(c)).
This diagnostic therefore does not reveal strong additional
suppression of spectrum-changing motion beyond the
isotropic dimensional reference.

The illustrative SFT transition, by contrast, yields mean
$\kappa_{\mathrm{spec}}$ values between $89$ and $1364$.
Thus, correcting for ambient dimensionality preserves a
pronounced distinction: the analyzed SFT displacement is
far more concentrated in spectrum-changing coordinates
than the RLVR displacement.
These results distinguish the descriptive observation of
near-isospectrality from an explanation based on preferential
spectral suppression; neither determines the functional role
of the remaining spectral changes.

\section{Qualitative Index-Alignment Check}
\label{app:index_alignment}

\paragraph{Why this check is needed.}
The main text avoids using individual singular-vector identities as formal
evidence, since SVD bases are gauge-dependent under sign flips and rotations
inside near-degenerate singular-value clusters. Nevertheless, because our
spectrum-rebasing intervention pairs singular values by rank index, it is useful
to verify whether rank-index tracking is qualitatively stable in the RLVR runs
we analyze.

\paragraph{Rank-index singular-vector rotations.}
For each layer, we compute the principal angle between rank-index paired
singular directions from the base model and the post-training checkpoint, after
standard sign alignment. These plots should be interpreted only as qualitative
sanity checks. The formal frame-adaptability conclusions in
Section~\ref{sec:transport} use projector-based reconstructions and
unexplained-update ratios.

Fig.~\ref{fig:rl_rotation} and~\ref{fig:sft_rotation} provide a qualitative
contrast between RLVR and SFT. Under RLVR, rank-index paired singular directions
rotate mildly and smoothly, suggesting that the spectrum-restoration intervention
is not dominated by large index-wise mismatch. Under SFT, many matched
directions become nearly orthogonal, indicating that rank-index tracking is much
less stable. We emphasize that these plots are not used as formal evidence for
singular-subspace motion. The main diagnostic evidence is the gauge-invariant
projector and reconstruction analysis in Section~\ref{sec:transport}.

\begin{figure}[!ht]
\centering
\begin{minipage}[b]{0.49\linewidth}
    \centering
    \includegraphics[width=\linewidth]{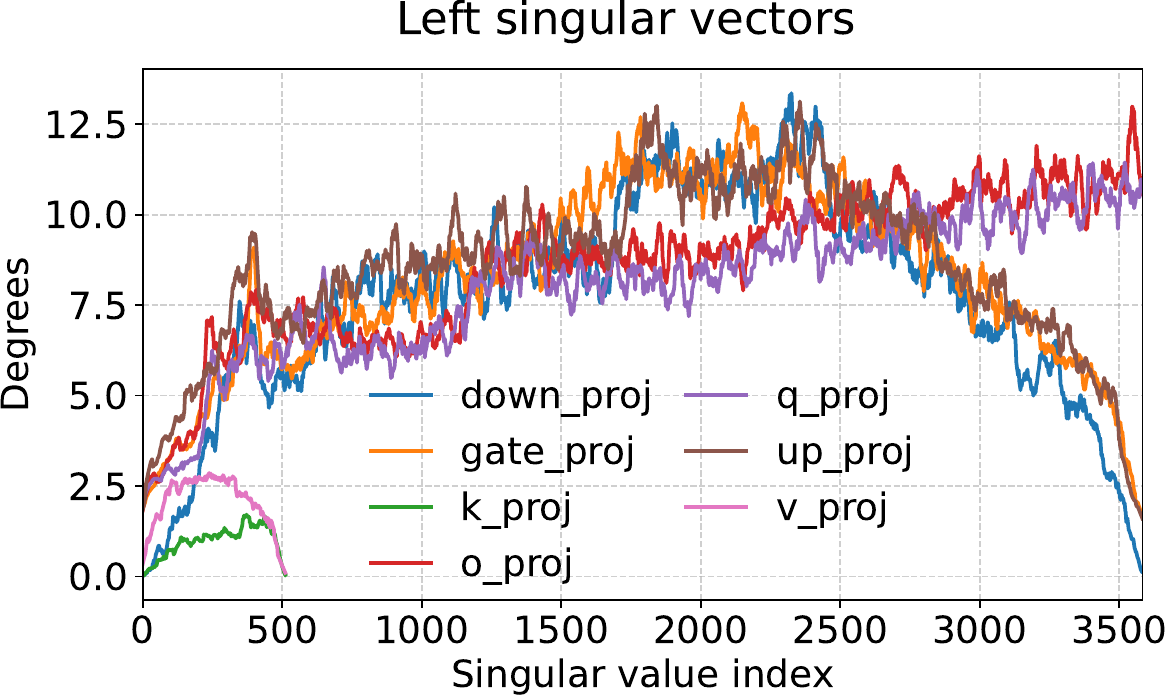}
\end{minipage}\hfill
\begin{minipage}[b]{0.49\linewidth}
    \centering
    \includegraphics[width=\linewidth]{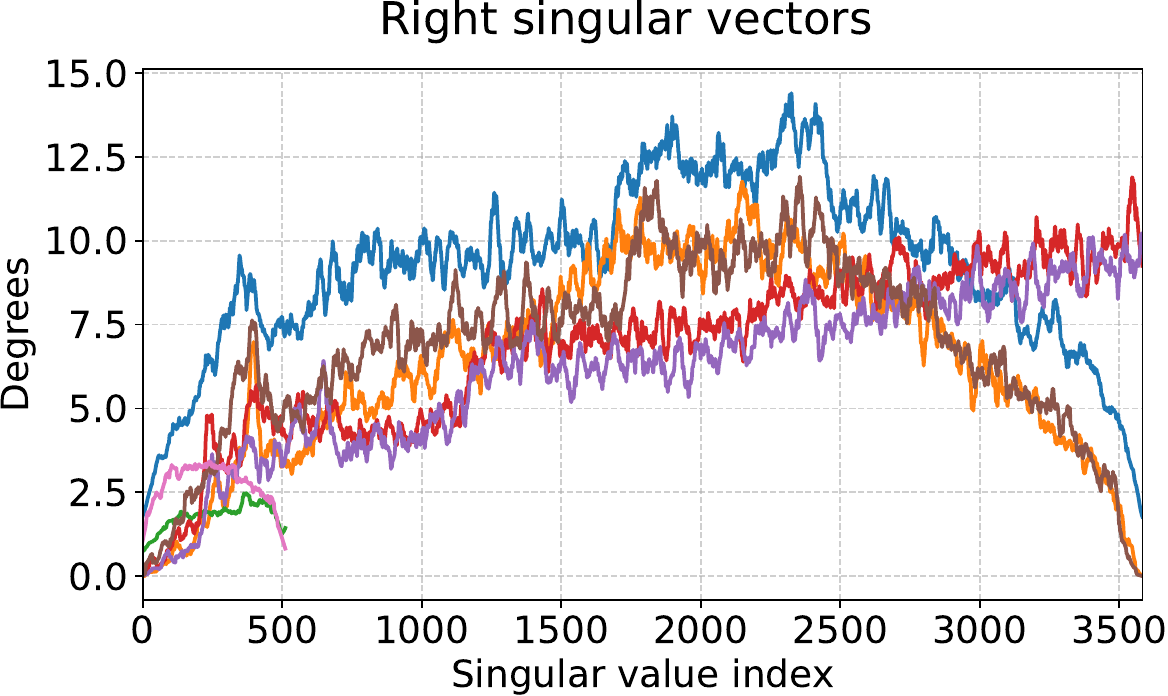}
\end{minipage}
\caption{
\textbf{Qualitative singular-direction rotations under RL post-training.}
Mean principal angles between rank-index paired base and RL-trained singular
directions remain small and smooth across attention and MLP layers. This
suggests that rank-index tracking is qualitatively stable under RLVR in these
runs. These plots are used only as a sanity check. Formal claims use
projector-based diagnostics.
}
\label{fig:rl_rotation}
\vspace{-1em}
\end{figure}

\begin{figure}[!ht]
\centering
\begin{minipage}[b]{0.49\linewidth}
    \centering
    \includegraphics[width=\linewidth]{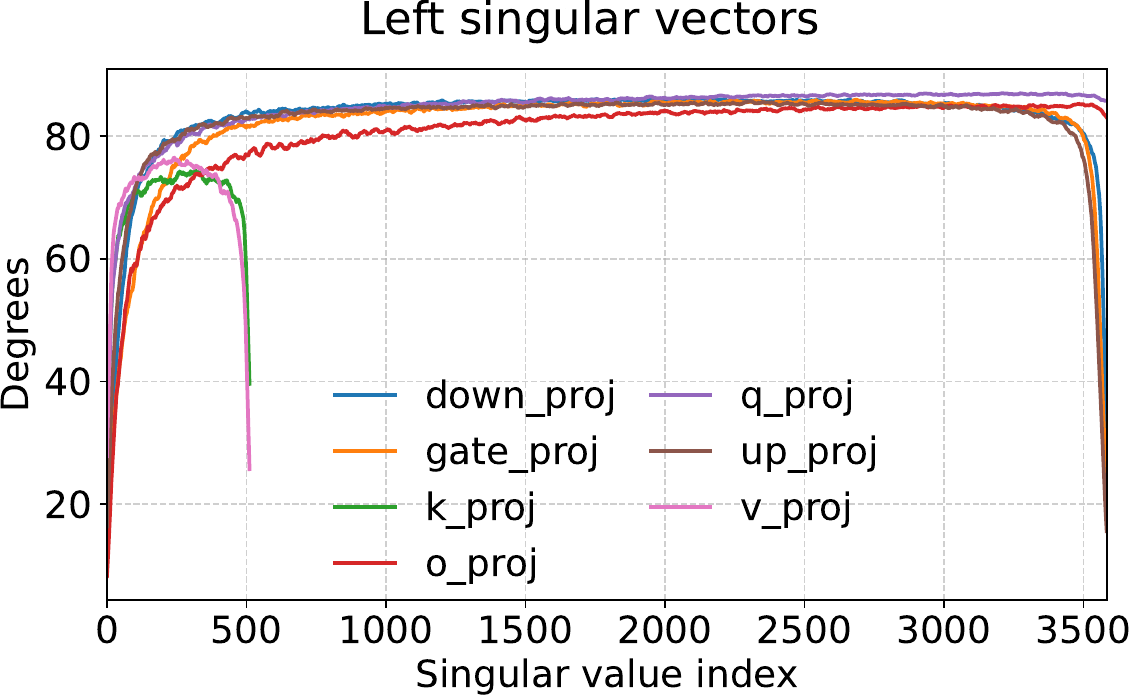}
\end{minipage}\hfill
\begin{minipage}[b]{0.49\linewidth}
    \centering
    \includegraphics[width=\linewidth]{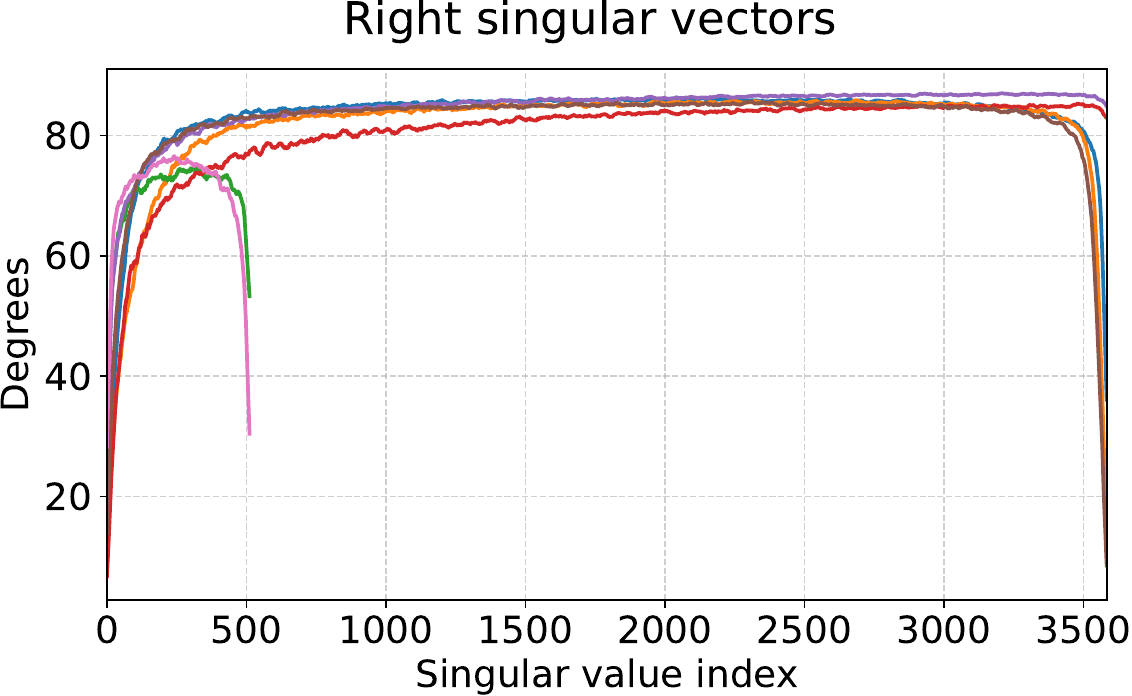}
\end{minipage}
\caption{
\textbf{Qualitative singular-direction rotations under SFT.}
Mean angles between rank-index paired base and SFT singular directions increase
toward near-orthogonality for many ranks. In high-dimensional spaces, unrelated
directions are typically nearly orthogonal, so this indicates that simple
rank-index tracking becomes much less informative under SFT.
}
\label{fig:sft_rotation}
\vspace{-1em}
\end{figure}

\section{Fixed-Spectrum Theory and Frame-Adaptability Diagnostics}
\label{app:fixed_spectrum_theory}

\paragraph{Notation.}
Let
$W\in\mathbb{R}^{d_{\mathrm{out}}\times d_{\mathrm{in}}}$ and
$q=\min\{d_{\mathrm{out}},d_{\mathrm{in}}\}$. We write
\[
    \sigma(W)
    =
    \bigl(\sigma_1(W),\ldots,\sigma_q(W)\bigr)
\]
for all singular values in nonincreasing order, including zeros with
multiplicity. A thin SVD is
\[
    W=U\Sigma V^\top,
    \qquad
    U\in\mathrm{St}(d_{\mathrm{out}},q),
    \quad
    V\in\mathrm{St}(d_{\mathrm{in}},q),
    \quad
    \Sigma=\operatorname{Diag}(\sigma(W)).
\]
Here $\operatorname{Diag}(a)$ constructs a diagonal matrix from a vector,
whereas $\operatorname{diag}(A)$ extracts the diagonal of a matrix.

\subsection{The Fixed-Spectrum Family and Nearest Representatives}
\label{app:fixed_spectrum_distance}

\begin{proposition}[Two-sided characterization of the fixed-spectrum family]
\label{prop:app_orbit_characterization}
For any source matrix $W_0$,
\begin{equation}
\label{eq:app_source_orbit}
\begin{aligned}
    \mathcal F(W_0)
    &:=
    \left\{
        Z\in
        \mathbb{R}^{d_{\mathrm{out}}\times d_{\mathrm{in}}}
        :
        \sigma(Z)=\sigma(W_0)
    \right\}
    \\
    &=
    \left\{
        Q_LW_0Q_R^\top:
        Q_L\in\mathrm O(d_{\mathrm{out}}),
        \;
        Q_R\in\mathrm O(d_{\mathrm{in}})
    \right\}.
\end{aligned}
\end{equation}
\end{proposition}

\begin{proof}
Orthogonal multiplication preserves singular values, proving one inclusion.
For the reverse inclusion, write
\[
    W_0=U_0\Sigma_0V_0^\top,
    \qquad
    Z=U_Z\Sigma_0V_Z^\top.
\]
Extend $U_0,U_Z$ and $V_0,V_Z$ to complete orthogonal bases and choose
$Q_L,Q_R$ mapping the source bases to the target bases. Then
$Q_LW_0Q_R^\top=Z$.
\end{proof}

\begin{proposition}[Exact distance to the fixed-spectrum family]
\label{prop:app_isospectral_projection}
For
$W,W_0\in\mathbb{R}^{d_{\mathrm{out}}\times d_{\mathrm{in}}}$,
\begin{equation}
\label{eq:app_exact_orbit_distance}
    \operatorname{dist}_F
    \bigl(
        W,\mathcal F(W_0)
    \bigr)
    =
    \|\sigma(W)-\sigma(W_0)\|_2.
\end{equation}
If
$W=U\Sigma V^\top$ and
$\Sigma_0=\operatorname{Diag}(\sigma(W_0))$, then
\begin{equation}
\label{eq:app_nearest_orbit_point}
    U\Sigma_0V^\top
    \in
    \operatorname*{arg\,min}_{Z\in\mathcal F(W_0)}
    \|W-Z\|_F.
\end{equation}
\end{proposition}

\begin{proof}
Every $Z\in\mathcal F(W_0)$ has Frobenius norm $\|W_0\|_F$, so
\[
    \|W-Z\|_F^2
    =
    \|W\|_F^2+\|W_0\|_F^2-2\langle W,Z\rangle_F.
\]
By von Neumann's trace inequality,
\[
    \langle W,Z\rangle_F
    \leq
    \sum_{k=1}^q
    \sigma_k(W)\sigma_k(W_0).
\]
The bound is attained by $Z=U\Sigma_0V^\top$. Substitution gives
\[
    \min_{Z\in\mathcal F(W_0)}
    \|W-Z\|_F^2
    =
    \sum_{k=1}^q
    \bigl(
        \sigma_k(W)-\sigma_k(W_0)
    \bigr)^2.
\]
\end{proof}

\begin{remark}[Gauge and nonuniqueness of a rebased representative]
\label{rem:rebasing_gauge}
The distance to the fixed-spectrum family and the set of nearest points are
basis-independent, but a particular representative $U\Sigma_0V^\top$ need not
be unique.

Suppose a block $\mathcal C$ of the target spectrum is exactly repeated.
Replacing
\[
    U_{\mathcal C}\leftarrow U_{\mathcal C}R,
    \qquad
    V_{\mathcal C}\leftarrow V_{\mathcal C}R,
    \qquad
    R\in\mathrm O(|\mathcal C|)
\]
leaves $W$ unchanged but produces the rebased block
\[
    U_{\mathcal C}
    R\Sigma_{0,\mathcal C}R^\top
    V_{\mathcal C}^\top.
\]
Every such representative is a nearest point in $\mathcal F(W_0)$. It is
identical for all $R$ when
$\Sigma_{0,\mathcal C}$ is proportional to the identity. More generally, with
$\bar\sigma_{\mathcal C}$ denoting the block mean,
\begin{equation}
\label{eq:gauge_sensitivity_bound}
    \left\|
        R\Sigma_{0,\mathcal C}R^\top-\Sigma_{0,\mathcal C}
    \right\|_F
    \leq
    2
    \left\|
        \Sigma_{0,\mathcal C}
        -
        \bar\sigma_{\mathcal C}I
    \right\|_F.
\end{equation}
Thus the ambiguity is small when the source spectrum is nearly flat inside the
cluster. Near-degenerate but distinct singular values do not create exact gauge
freedom, but can make the numerical singular vectors ill-conditioned. Our
rebasing experiment evaluates the representative selected by the numerical
SVD; Appendix~\ref{app:index_alignment} reports a qualitative stability check.
\end{remark}

\begin{corollary}[Finite-horizon spectral drift]
\label{cor:finite_horizon_spectral_drift}
Let $W_0,\ldots,W_T$ be a matrix trajectory and
$E_t=W_{t+1}-W_t$. Then
\begin{equation}
    \|\sigma(W_t)-\sigma(W_0)\|_2
    =
    \operatorname{dist}_F
    \bigl(
        W_t,\mathcal F(W_0)
    \bigr)
    \leq
    \|W_t-W_0\|_F
    \leq
    \sum_{\tau=0}^{t-1}\|E_\tau\|_F.
\end{equation}
\end{corollary}

\subsection{Spectral-Distance Metrics and Dimension-Aware Calibration}
\label{app:dimension_calibration}

For matrix $\ell$, let
$\Delta W^{(\ell)}=W_1^{(\ell)}-W_0^{(\ell)}$. We use
\begin{equation}
\begin{aligned}
\delta_{\Sigma}^{(\ell)}
&=
\frac{
    \operatorname{dist}_F
    \left(
        W_1^{(\ell)},
        \mathcal F(W_0^{(\ell)})
    \right)
}{
    \|W_0^{(\ell)}\|_F
},
\\
\rho_{\Sigma}^{(\ell)}
&=
\frac{
    \operatorname{dist}_F
    \left(
        W_1^{(\ell)},
        \mathcal F(W_0^{(\ell)})
    \right)
}{
    \|\Delta W^{(\ell)}\|_F
}.
\end{aligned}
\end{equation}
The pooled residual reported in the main text is
\begin{equation}
\label{eq:pooled_rho_sigma}
    \rho_{\Sigma}^{\mathrm{pool}}
    =
    \left[
        \frac{
            \sum_\ell
            \operatorname{dist}_F^2
            \left(
                W_1^{(\ell)},
                \mathcal F(W_0^{(\ell)})
            \right)
        }{
            \sum_\ell
            \|\Delta W^{(\ell)}\|_F^2
        }
    \right]^{1/2}.
\end{equation}

\begin{proposition}[First-order spectrum-changing subspace and isotropic calibration]
\label{prop:spectrum_changing_calibration}
Assume $W_0$ is full rank with simple singular values and write
$W_0=U_0\Sigma_0V_0^\top$. The first-order spectrum-changing subspace at
$W_0$, equivalently the normal space of its fixed-spectrum family, is
\begin{equation}
\label{eq:first_order_spectrum_changing_subspace}
    N_{W_0}\mathcal F(W_0)
    =
    \left\{
        U_0\operatorname{Diag}(a)V_0^\top:
        a\in\mathbb R^q
    \right\}.
\end{equation}
The orthogonal projection of a perturbation $H$ onto this spectrum-changing
subspace is
\begin{equation}
\label{eq:spectrum_changing_projection}
    \Pi_{\mathrm{spec}}(H)
    =
    U_0
    \operatorname{Diag}
    \left(
        \operatorname{diag}(U_0^\top H V_0)
    \right)
    V_0^\top.
\end{equation}
If $\operatorname{vec}(H)/\|H\|_F$ is isotropic, then
\begin{equation}
\label{eq:isotropic_dimension_fraction}
    \mathbb E
    \left[
        \frac{
            \|\Pi_{\mathrm{spec}}(H)\|_F^2
        }{
            \|H\|_F^2
        }
    \right]
    =
    \frac{
        q
    }{
        d_{\mathrm{out}}d_{\mathrm{in}}
    }.
\end{equation}
Consequently,
\begin{equation}
\label{eq:app_kappa_spec}
    \kappa_{\mathrm{spec}}
    =
    \frac{
        d_{\mathrm{out}}d_{\mathrm{in}}
    }{
        q
    }
    \frac{
        \|
            \operatorname{diag}
            (U_0^\top H V_0)
        \|_2^2
    }{
        \|H\|_F^2
    }
\end{equation}
satisfies $\mathbb E[\kappa_{\mathrm{spec}}]=1$. Thus, one is the isotropic
dimensional reference under this normalization.
\end{proposition}

\begin{proof}
For a simple full-rank spectrum, fixing the $q$ singular values imposes $q$
independent first-order constraints. The matrices
$\{u_{0,k}v_{0,k}^\top\}_{k=1}^q$ form an orthonormal basis for this
spectrum-changing subspace, which gives
Equation~\ref{eq:spectrum_changing_projection}. An isotropic direction places
expected squared energy in a fixed subspace in proportion to its dimension.
The spectrum-changing and ambient dimensions are
$q$ and $d_{\mathrm{out}}d_{\mathrm{in}}$, respectively.
\end{proof}

\subsection{Reconstruction Classes and Unexplained-Update Ratios}

\label{app:frame_adaptability}

This subsection gives the full reconstruction test used in
Section~\ref{sec:transport}: its structural restrictions, optimal
reconstructions, normalized residuals, and interpretation.
The sequential-RL experiments and rank-sensitivity results are
reported in Appendix~\ref{app:rank}.

\paragraph{Incoming checkpoint and admissible truncations.}
For a transition $W_j\rightarrow W_i$, let $W_j$ be the
\emph{incoming checkpoint} and $W_i$ the \emph{learned endpoint}.
For each corresponding weight matrix, let
$q=\min\{d_{\mathrm{out}},d_{\mathrm{in}}\}$.
At an integer rank $1\leq r<q$, write
\begin{equation}
\label{eq:app-transport_truncation}
\begin{aligned}
    W_t^{(r)}
    &= U_t^{(r)}\Sigma_t^{(r)}\bigl(V_t^{(r)}\bigr)^\top,\\
    P_t^{(r)}
    &= U_t^{(r)}\bigl(U_t^{(r)}\bigr)^\top,
    \qquad
    Q_t^{(r)}
    = V_t^{(r)}\bigl(V_t^{(r)}\bigr)^\top,
    \qquad t\in\{i,j\}.
\end{aligned}
\end{equation}
Here $U_t^{(r)}$ and $V_t^{(r)}$ contain the leading $r$ left and
right singular vectors, and $P_t^{(r)}$ and $Q_t^{(r)}$ project
onto the corresponding output and input subspaces.
A transition is \emph{admissible} at rank $r$ when
\begin{equation}
\label{eq:app_rank_admissibility}
    \sigma_r(W_t)>\sigma_{r+1}(W_t),
    \qquad t\in\{i,j\},
\end{equation}
so that both retained subspaces are uniquely defined.
The main-text comparison uses $r=\lfloor0.9q\rfloor$;
Appendix~\ref{app:rank} varies the truncation fraction.

\paragraph{Four structural restrictions.}
Let $A=W_i^{(r)}$.
We compare three subspace-retaining classes,
\begin{equation}
\label{eq:app_reconstruction_classes}
\begin{aligned}
    \mathcal C_{\mathrm{mix}}
    &= \left\{
        U_j^{(r)}C\bigl(V_j^{(r)}\bigr)^\top:
        C\in\mathbb R^{r\times r}
    \right\},\\
    \mathcal C_L
    &= \left\{
        U_j^{(r)}B:
        B\in\mathbb R^{r\times d_{\mathrm{in}}}
    \right\},\\
    \mathcal C_R
    &= \left\{
        B\bigl(V_j^{(r)}\bigr)^\top:
        B\in\mathbb R^{d_{\mathrm{out}}\times r}
    \right\},
\end{aligned}
\end{equation}
with the incoming rank-$r$ fixed-spectrum family,
\begin{equation}
\label{eq:app-truncated_fixed_spectrum_family}
    \mathcal F_j^{(r)}
    := \left\{
        U\Sigma_j^{(r)}V^\top:
        U\in\mathrm{St}(d_{\mathrm{out}},r),\
        V\in\mathrm{St}(d_{\mathrm{in}},r)
    \right\}.
\end{equation}
The classes $\mathcal C_{\mathrm{mix}}$, $\mathcal C_L$, and
$\mathcal C_R$ retain both incoming subspaces, only the output
subspace, or only the input subspace, respectively.
Their remaining mapping and spectrum may refit freely.
By contrast, $\mathcal F_j^{(r)}$ retains the incoming singular
values $\Sigma_j^{(r)}$ while allowing both frames to vary.

\begin{proposition}[Optimal checkpoint reconstructions]
\label{prop:app_frame_adaptability}
For an admissible transition at rank $r$, the Frobenius-optimal
reconstructions of $A=W_i^{(r)}$ under the four restrictions are
\begin{equation}
\label{eq:app-transport_reconstructions}
\begin{aligned}
    \widehat W_{\mathrm{mix}}
    &= P_j^{(r)}AQ_j^{(r)},
    &\widehat W_L &= P_j^{(r)}A,\\
    \widehat W_R
    &= AQ_j^{(r)},
    &\widehat W_{\mathrm{iso}}
    &= U_i^{(r)}\Sigma_j^{(r)}\bigl(V_i^{(r)}\bigr)^\top.
\end{aligned}
\end{equation}
In particular,
\begin{equation}
\label{eq:app_iso_reconstruction_error}
\begin{aligned}
    \widehat W_{\mathrm{iso}}
    &\in \operatorname*{arg\,min}_{Z\in\mathcal F_j^{(r)}}
        \|A-Z\|_F,\\
    \min_{Z\in\mathcal F_j^{(r)}}\|A-Z\|_F
    &= \|\Sigma_i^{(r)}-\Sigma_j^{(r)}\|_F.
\end{aligned}
\end{equation}
\end{proposition}

\begin{proof}
The maps $A\mapsto P_j^{(r)}AQ_j^{(r)}$,
$A\mapsto P_j^{(r)}A$, and $A\mapsto AQ_j^{(r)}$
are orthogonal projections, under the Frobenius inner product,
onto $\mathcal C_{\mathrm{mix}}$, $\mathcal C_L$, and
$\mathcal C_R$, respectively.
They therefore minimize the reconstruction error over these classes.
The isospectral result follows from
Proposition~\ref{prop:app_isospectral_projection}, applied to
$W_i^{(r)}$ and $W_j^{(r)}$: their remaining $q-r$ singular
values are zero, so the distance is
$\|\Sigma_i^{(r)}-\Sigma_j^{(r)}\|_F$, attained by
$U_i^{(r)}\Sigma_j^{(r)}(V_i^{(r)})^\top$.
\end{proof}

The remix reconstruction can equivalently be written as
\begin{equation}
\label{eq:app_incoming_subspace_remix}
\begin{aligned}
    \widehat W_{\mathrm{mix}}
    &= U_j^{(r)}B_{ij}^{(r)}\bigl(V_j^{(r)}\bigr)^\top,\\
    B_{ij}^{(r)}
    &:= \bigl(U_j^{(r)}\bigr)^\top W_i^{(r)}V_j^{(r)}.
\end{aligned}
\end{equation}
The core $B_{ij}^{(r)}$ is unconstrained: it may rotate, mix,
rescale, and change the represented spectrum.
Thus, this class fixes the incoming spans, not the individual
singular directions or mode scales.

\paragraph{Normalized residuals and unexplained-update ratios.}
We normalize each reconstruction error by the endpoint scale:
\begin{equation}
\label{eq:app_frame_residuals}
    e_h
    := \frac{\|A-\widehat W_h\|_F}{\|A\|_F},
    \qquad h\in\{\mathrm{mix},L,R,\mathrm{iso}\}.
\end{equation}
The unchanged incoming checkpoint provides the reference
\begin{equation}
\label{eq:app_frame_adaptability_drift}
    e_{\mathrm{drift}}
    := \frac{\|W_i^{(r)}-W_j^{(r)}\|_F}{\|W_i^{(r)}\|_F}.
\end{equation}
This is a scale reference, not an additional reconstruction class.
For $W_i^{(r)}\neq W_j^{(r)}$, the corresponding
\emph{unexplained-update ratio} is
\begin{equation}
\label{eq:app_unexplained_update_ratio}
    u_h
    := \frac{e_h}{e_{\mathrm{drift}}}
    = \frac{\|W_i^{(r)}-\widehat W_h\|_F}
           {\|W_i^{(r)}-W_j^{(r)}\|_F},
    \qquad h\in\{\mathrm{mix},L,R,\mathrm{iso}\}.
\end{equation}
Because $W_j^{(r)}$ belongs to all four reconstruction classes,
optimality gives
\[
    0\leq \|W_i^{(r)}-\widehat W_h\|_F
    \leq \|W_i^{(r)}-W_j^{(r)}\|_F,
    \qquad 0\leq u_h\leq1.
\]
Thus, $u_h=0$ denotes exact reconstruction, whereas $u_h=1$
means that the best reconstruction under restriction $h$ is no
closer to the endpoint than the unchanged incoming checkpoint.
The quantities $u_h$ are undefined for zero truncated displacement.
The fixed-spectrum residuals have the closed forms
\begin{equation}
\label{eq:app_fixed_spectrum_residual}
    e_{\mathrm{iso}}
    = \frac{\|\Sigma_i^{(r)}-\Sigma_j^{(r)}\|_F}
           {\|\Sigma_i^{(r)}\|_F},
\end{equation}
\begin{equation}
\label{eq:app-iso_transport_residual}
    u_{\mathrm{iso}}
    = \frac{\|\Sigma_i^{(r)}-\Sigma_j^{(r)}\|_F}
           {\|W_i^{(r)}-W_j^{(r)}\|_F}.
\end{equation}

\begin{remark}[Interpretation of the reconstruction test]
The classes have different dimensions, so their residual ordering
is not a complexity-normalized model-selection result.
The test concerns parameterization sufficiency: whether retaining
one or both incoming spans still permits a low-residual description
of the observed rank-$r$ endpoint.
Each subspace-retaining class contains its corresponding
fixed-spectrum restriction, making its optimal residual an
optimistic lower bound for that more restrictive description.
Thus, a large residual cannot be attributed to unnecessarily
fixing the remaining core or spectrum.
\end{remark}

\begin{remark}[Gauge invariance and internal rotations]
For fixed checkpoints, the retained projectors and reconstruction
classes do not depend on the choice of bases for the retained
subspaces, and the optimal residual values are basis-independent.
If $\Sigma_i^{(r)}$ contains repeated singular values,
$\widehat W_{\mathrm{iso}}$ may not be the unique nearest
representative, but $e_{\mathrm{iso}}$ and $u_{\mathrm{iso}}$
remain uniquely determined whenever defined.

The remix class contains the internal fixed-spectrum slice
\[
    \left\{
        U_j^{(r)}R_U\Sigma_j^{(r)}R_V^\top
        \bigl(V_j^{(r)}\bigr)^\top:
        R_U,R_V\in\mathrm O(r)
    \right\},
\]
where $\mathrm O(r)$ is the orthogonal group.
Consequently, a large optimal remix residual also excludes a
low-residual explanation based solely on rotations within both
incoming subspaces.
\end{remark}

\begin{remark}[Why $r<q$]
At full thin rank, at least one of $P_j^{(q)}$ and $Q_j^{(q)}$
is the identity.
The corresponding one-sided reconstruction is then exact, and
the remix reconstruction coincides with the other one-sided
reconstruction.
For square matrices, both projectors are the identity and all
three subspace-retaining reconstructions are exact.
Using $r<q$ makes both retained subspaces proper and keeps the
comparison informative.
\end{remark}

\subsection{Sequential RL Evidence and Rank Sensitivity}
\label{app:rank}

We apply the reconstruction test to address two questions:
whether simpler, subspace-retaining transformations explain
learned endpoints, and whether the same structural pattern
recurs after changing the RL objective.

\paragraph{A sequential objective-shift stress test.}
We analyze a two-stage vision-language-action RL pipeline
initialized from
\texttt{Qwen2.5-VL-3B-Instruct}~\cite{yuan2025embodied}.
Stage~I optimizes spatial reasoning and produces $W_1$ from $W_0$.
Stage~II starts from this already RL-trained checkpoint, optimizes
an embodied-manipulation objective, and produces $W_2$.
We examine corresponding weight matrices in both language and
vision modules across three transitions:
\begin{equation}
\label{eq:app-vla_transitions}
    W_0\rightarrow W_1,
    \qquad W_1\rightarrow W_2,
    \qquad W_0\rightarrow W_2.
\end{equation}
The first transition provides the initial-stage reference.
The second tests recurrence after both an objective shift and
re-anchoring at an RL-trained checkpoint: its incoming spectrum
and subspaces are those of $W_1$, not $W_0$.
The cumulative transition measures the complete two-stage change
relative to $W_0$.
Unless otherwise stated, we use $r=\lfloor0.9q\rfloor$ and the
admissibility condition in
Equation~\ref{eq:app_rank_admissibility}.

\paragraph{Endpoint reconstruction: both frames must remain adaptable.}
For the cumulative transition $W_0\rightarrow W_2$,
Fig.~\ref{fig:frame_adaptability} reports the unexplained-update
ratios across layers of representative language and vision modules.
The reported median ratios are $87\%$ for remixing within both
incoming subspaces, $45\%$ for retaining only the output subspace,
and $42\%$ for retaining only the input subspace.
These residuals are already optimistic because the corresponding
classes freely refit the remaining mapping and spectrum.
By contrast, retaining the incoming spectrum while adapting both
frames leaves a median residual of only $1.8\%$.
Thus, among the tested reconstruction classes, both singular
frames must remain adaptable for a low-residual characterization
of the observed checkpoint changes.

\paragraph{Recurrence under a new RL objective.}
The same qualitative separation holds for $W_1\rightarrow W_2$,
where the incoming checkpoint has already undergone RL and the
objective changes from spatial reasoning to embodied manipulation.
Thus, the two-frame reconstruction pattern is not confined to
the first RL stage: it recurs after re-anchoring at $W_1$ and
changing the objective.
Fig.~\ref{fig:rank_sensitivity} includes this transition,
the first-stage transition $W_0\rightarrow W_1$, and the
cumulative transition $W_0\rightarrow W_2$.

\paragraph{Rank-sensitivity protocol.}
For $\alpha\in\{0.2,0.4,0.6,0.8,0.9\}$, we use
\begin{equation}
\label{eq:app_rank_fraction}
    r_\alpha=\max\{1,\lfloor\alpha q\rfloor\}.
\end{equation}
At each $\alpha$, a matrix--transition pair is included only
when $r_\alpha<q$ is admissible at both endpoints;
inadmissible pairs are omitted at that value of $\alpha$.
For each included pair with nonzero truncated displacement,
we compute $u_h=e_h/e_{\mathrm{drift}}$ as in
Equation~\ref{eq:app_unexplained_update_ratio}.

\paragraph{Rank-sensitivity results.}
The separation between inherited-spectrum and subspace-retaining
reconstructions persists across the displayed ranks in
Fig.~\ref{fig:rank_sensitivity}.
For the representative matrix--transition panels shown,
$u_{\mathrm{iso}}$ remains below $2\%$ throughout the sweep,
whereas the subspace-retaining alternatives leave tens of percent
of the truncated update unexplained.
At $\alpha=0.2$, the reported median across the twelve displayed
panels is $u_{\mathrm{mix}}=0.99$.
The gap is therefore already pronounced at small truncation ranks:
remixing within the dominant incoming spans explains little of
the observed rank-$r$ change, whereas adapting both frames under
the inherited spectrum remains low-residual.

\begin{figure}[t]
    \centering
    \begin{tikzpicture}
        \node[anchor=south west,inner sep=0] (rankimage) at (0,0) {%
            \includegraphics[width=\linewidth]
                {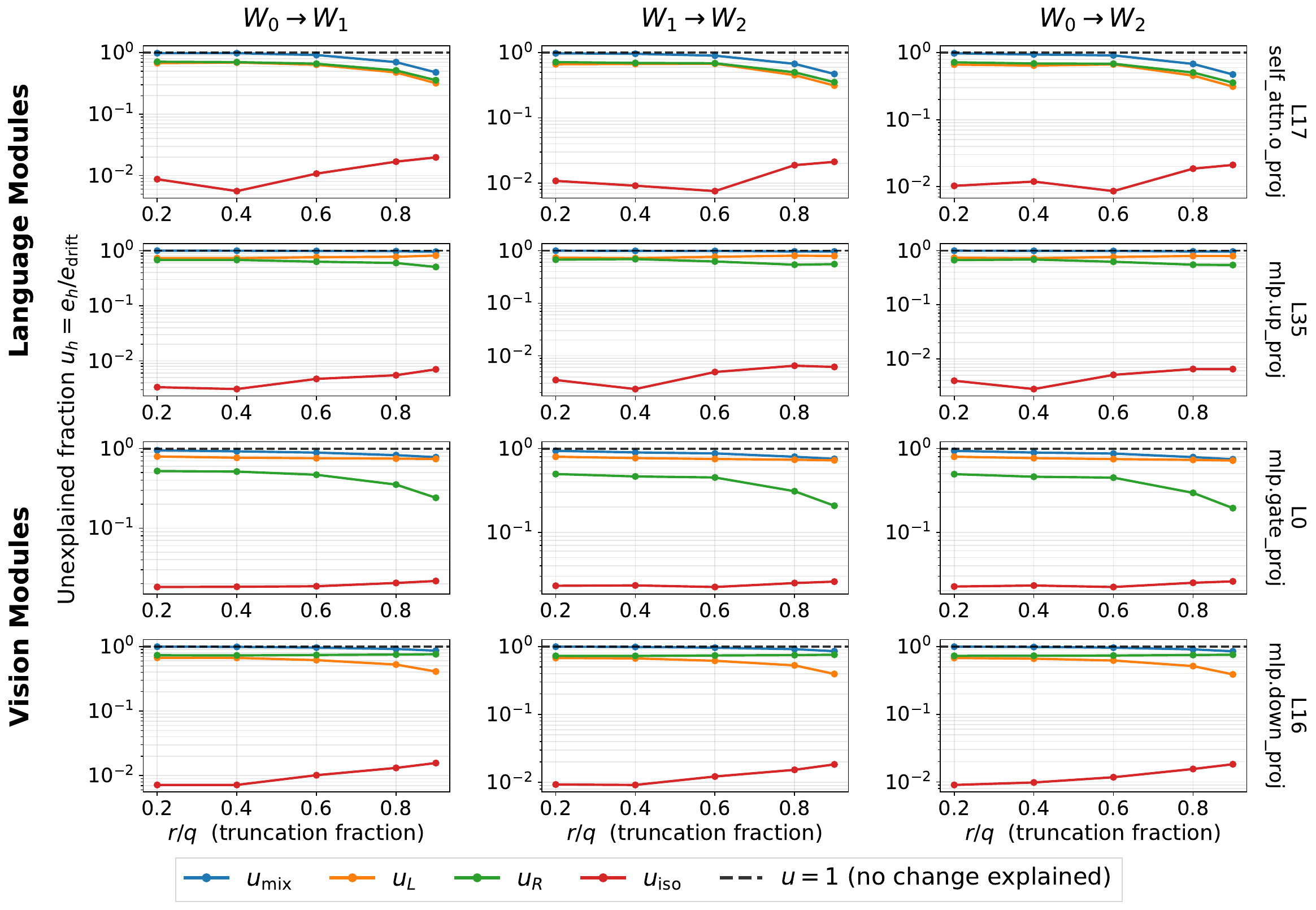}%
        };
        \begin{scope}[
            x={(rankimage.south east)},
            y={(rankimage.north west)}
        ]
            \fill[white] (0.035,0.28) rectangle (0.068,0.72);
            \node[rotate=90,font=\fontsize{8}{9}\selectfont]
                at (0.0515,0.50)
                {Unexplained-update ratio $u_h=e_h/e_{\mathrm{drift}}$};
        \end{scope}
    \end{tikzpicture}
    \caption{\small
    \textbf{Rank sensitivity across sequential RL transitions.}
    Unexplained-update ratios versus $r/q$ for representative
    language and vision matrices and all three transitions.
    The dashed line $u_h=1$ marks the unchanged-checkpoint reference.
    Across the displayed sweep, inherited-spectrum reconstruction
    remains below $2\%$, while the subspace-retaining alternatives
    leave substantially larger residuals.
    The reported median $u_{\mathrm{mix}}$ at $\alpha=0.2$
    is $0.99$ across the twelve displayed matrix--transition panels.
    }
    \label{fig:rank_sensitivity}
\end{figure}

Together, the endpoint, objective-shift, and rank-sensitivity
results support the structural conclusion of
Section~\ref{sec:transport}: the incoming spectrum can be retained
while both frames remain adaptable for a low-residual description
of the observed changes.
Combined with the functional interventions in
Section~\ref{sec:spectral_inheritance_tests}, this motivates
ISO's parameterization $W(U,V)=U\Sigma_0V^\top$ as an
inductive bias for learning and composition.

\section{First-Order Properties of the ISO Parameterization}
\label{app:iso_first_order}

\subsection{Proof of Proposition~\ref{prop:iso_tangent}}
\label{app:proof_iso_tangent}

\begin{proof}
We prove the two claims in Proposition~\ref{prop:iso_tangent}.

First, consider a differentiable curve
$W(\varepsilon)\in\mathcal F(W_0)$ with
$W(0)=W=U\Sigma_0V^\top$. Since all $q$ singular values are positive and
simple, we may locally choose differentiable
Stiefel representatives $U(\varepsilon)$ and $V(\varepsilon)$ such that
\begin{equation}
\small
    W(\varepsilon)
    =
    U(\varepsilon)\Sigma_0 V(\varepsilon)^\top,
    \qquad
    U(0)=U,\quad V(0)=V .
\end{equation}
Let
\begin{equation}
\small
    \dot U
    =
    \left.\frac{d}{d\varepsilon}U(\varepsilon)\right|_{\varepsilon=0},
    \qquad
    \dot V
    =
    \left.\frac{d}{d\varepsilon}V(\varepsilon)\right|_{\varepsilon=0},
    \qquad
    \dot W
    =
    \left.\frac{d}{d\varepsilon}W(\varepsilon)\right|_{\varepsilon=0}.
\end{equation}
Differentiating the fixed-spectrum representation gives
\begin{equation}
\small
    \dot W
    =
    \dot U\Sigma_0 V^\top
    +
    U\Sigma_0 \dot V^\top .
\end{equation}
Multiplying on the left by $U^\top$ and on the right by $V$ yields
\begin{equation}
\small
    U^\top \dot W V
    =
    U^\top\dot U\,\Sigma_0
    +
    \Sigma_0\,\dot V^\top V .
\end{equation}
Because $U(\varepsilon)$ and $V(\varepsilon)$ remain on the Stiefel manifold,
differentiating
$U(\varepsilon)^\top U(\varepsilon)=I$ and
$V(\varepsilon)^\top V(\varepsilon)=I$ at $\varepsilon=0$ gives
\begin{equation}
\small
    U^\top\dot U+\dot U^\top U=0,
    \qquad
    V^\top\dot V+\dot V^\top V=0 .
\end{equation}
Thus $U^\top\dot U$ and $\dot V^\top V$ are skew-symmetric. Since $\Sigma_0$
is diagonal, both $U^\top\dot U\,\Sigma_0$ and
$\Sigma_0\dot V^\top V$ have zero diagonal. Therefore
\begin{equation}
\small
    \operatorname{diag}(U^\top \dot W V)=0 .
\end{equation}

Second, let $H$ be an arbitrary perturbation and define
\begin{equation}
\small
    W(\varepsilon)=W+\varepsilon H .
\end{equation}
For each singular value, choose differentiable singular vectors
$u_k(\varepsilon)$ and $v_k(\varepsilon)$ with singular value
$\sigma_k(\varepsilon)$. Then
\begin{equation}
\small
    \sigma_k(\varepsilon)
    =
    u_k(\varepsilon)^\top
    W(\varepsilon)
    v_k(\varepsilon).
\end{equation}
Differentiating at $\varepsilon=0$ gives
\begin{equation}
\small
\begin{aligned}
    \sigma_k'(0)
    &=
    \dot u_k^\top W v_k
    +
    u_k^\top H v_k
    +
    u_k^\top W \dot v_k .
\end{aligned}
\end{equation}
Using $Wv_k=\sigma_k u_k$ and $u_k^\top W=\sigma_k v_k^\top$, this becomes
\begin{equation}
\small
    \sigma_k'(0)
    =
    \sigma_k \dot u_k^\top u_k
    +
    u_k^\top H v_k
    +
    \sigma_k v_k^\top \dot v_k .
\end{equation}
Since $u_k(\varepsilon)$ and $v_k(\varepsilon)$ remain unit vectors,
differentiating
$u_k(\varepsilon)^\top u_k(\varepsilon)=1$ and
$v_k(\varepsilon)^\top v_k(\varepsilon)=1$ gives
\begin{equation}
\small
    \dot u_k^\top u_k=0,
    \qquad
    v_k^\top \dot v_k=0 .
\end{equation}
Hence
\begin{equation}
\small
    \left.
    \frac{d}{d\varepsilon}
    \sigma_k(W+\varepsilon H)
    \right|_{\varepsilon=0}
    =
    u_k^\top H v_k .
\end{equation}
Collecting this identity for $k=1,\ldots,q$ shows that
$\operatorname{diag}(U^\top H V)$ is exactly the vector of first-order
singular-value changes. This proves the proposition.
\end{proof}

\subsection{Factor-Gradient Derivation}
\label{app:iso_factor_gradients}

Let $\mathcal{L}(W)$ be the RLVR loss, let
$G_W=\nabla_W\mathcal{L}(W)$, and use the fixed-spectrum parameterization
\begin{equation}
\small
    W(U,V)=U\Sigma_0V^\top\in\mathcal F(W_0) .
\end{equation}
For infinitesimal Stiefel-factor changes
$\xi_U\in T_U\mathrm{St}(d_{\mathrm{out}},q)$ and
$\xi_V\in T_V\mathrm{St}(d_{\mathrm{in}},q)$, the induced weight motion is
\begin{equation}
\small
    \dot W
    =
    \xi_U\Sigma_0V^\top
    +
    U\Sigma_0\xi_V^\top .
\end{equation}
The first-order change in loss is
\begin{equation}
\small
\begin{aligned}
    d\mathcal{L}[\dot W]
    &=
    \langle G_W,\dot W\rangle_F \\
    &=
    \langle G_WV\Sigma_0,\xi_U\rangle_F
    +
    \langle G_W^\top U\Sigma_0,\xi_V\rangle_F .
\end{aligned}
\end{equation}
Therefore the raw factor gradients used by ISO are
\begin{equation}
\small
    G_U=G_WV\Sigma_0,
    \qquad
    G_V=G_W^\top U\Sigma_0 .
\end{equation}
These are the Euclidean gradients in the ambient factor coordinates.
ISO does not explicitly construct a weight-space projected gradient.
Instead, the fixed-spectrum parameterization and polar retraction restrict
every represented iterate to $\mathcal F(W_0)$ up to numerical precision.

\section{A Same-Base SFT--RLVR Case Study}
\label{app:14b_sft_rl}

The main-text case study in Section~\ref{sec:spectral_stability} compares an SFT
transition and an RLVR transition initialized from different base checkpoints.
To control for this difference, we repeat the analysis using two public
14B checkpoints derived from the same Qwen2.5-14B base:
\[
    \texttt{Qwen2.5-14B}
    \;\longrightarrow\;
    \begin{cases}
        \texttt{DeepSeek-R1-Distill-Qwen-14B}
        & \text{(distilled SFT)},\\[2pt]
        \texttt{Qwen-2.5-14B-SimpleRL-Zoo}
        & \text{(RLVR)}.
    \end{cases}
\]
Both endpoints are measured relative to the shared base checkpoint. We restrict
the comparison to the seven transformer projection matrices whose parameter
names and shapes match the shared Qwen2.5-14B backbone, avoiding differences
introduced by checkpoint-specific tokenizer or configuration changes. This
comparison therefore controls for the starting backbone weights, while the
training data, objectives, and optimization procedures remain intentionally
different. Fig.~\ref{fig:14b_samebase} summarizes the results.

\paragraph{Raw spectral proximity.}
The two endpoints differ sharply in their spectral displacement from the shared
base. At the scale of Fig.~\ref{fig:14b_samebase} (b), the per-rank changes of
the RLVR endpoint are visually near zero. Its base-normalized spectral distance
is approximately
\[
    \delta_{\Sigma}
    \approx
    6\times10^{-6}
\]
across layers and projection types. By contrast, the distilled-SFT endpoint
substantially changes the spectrum:
$\delta_{\Sigma}$ is approximately $10^{-2}$ for most projection types and
reaches approximately $0.4$ for the attention output projection.

\paragraph{Dimension-aware calibration.}
Raw proximity alone does not establish a preferentially spectrum-preserving
update direction, so we additionally report the dimension-normalized
spectrum-changing energy $\kappa_{\mathrm{spec}}$. Across projection types,
RLVR yields mean values between $1.4$ and $3.2$, with an overall mean of
$1.9$. These values remain order-one relative to the isotropic dimensional
reference, with the largest deviations concentrated in the earliest layers.

The distilled-SFT endpoint yields mean
$\kappa_{\mathrm{spec}}$ values between $117$ and $4810$, with an overall
mean of $873$ and the largest concentration in the attention output
projection. The same-base comparison therefore reproduces the main-text
contrast: SFT concentrates its displacement in spectrum-changing coordinates
by roughly two to three orders of magnitude, whereas RLVR exhibits no
comparable concentration. This difference cannot be explained by the two
transitions starting from different backbone checkpoints.

\begin{figure}[t]
    \centering
    \includegraphics[width=0.9\linewidth]{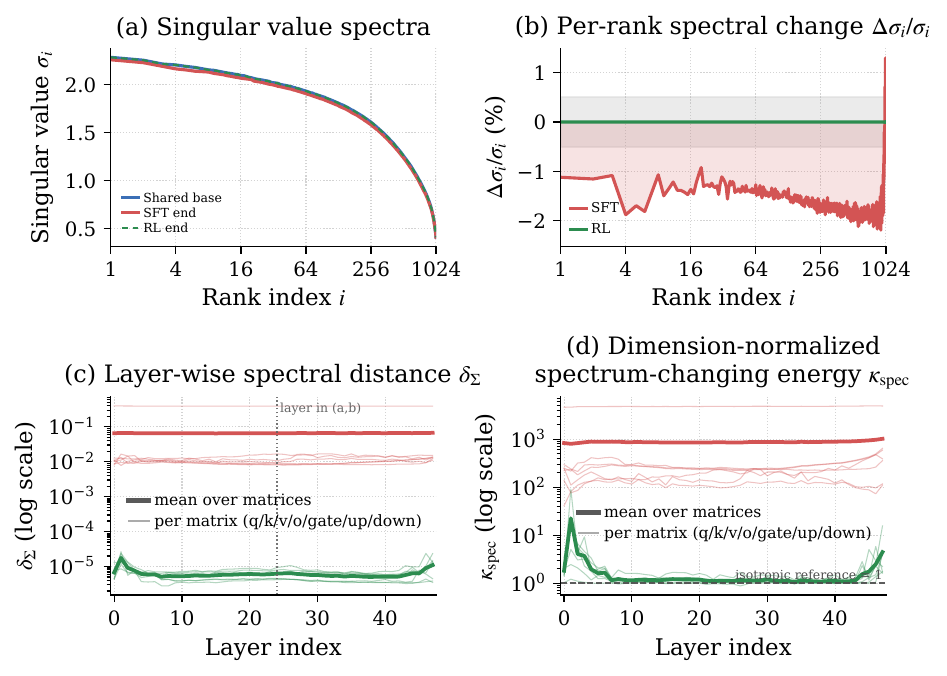}
\caption{
\small
\textbf{Same-base spectral dynamics.}
The distilled-SFT and RLVR endpoints are both compared with their shared
\texttt{Qwen2.5-14B} base.
\textbf{(a)} Singular-value profiles for \texttt{v\_proj} in layer 24.
\textbf{(b)} Rank-wise relative spectral change: the RLVR curve remains near
zero at the plotted scale, whereas distilled SFT systematically reshapes the
spectrum.
\textbf{(c)} Layer-wise base-normalized spectral distance
$\delta_{\Sigma}$, shown per projection matrix (thin) and averaged across
projection types (thick); the RLVR endpoint remains several orders of
magnitude closer to the fixed-spectrum family of the shared base.
\textbf{(d)} Dimension-normalized spectrum-changing energy
$\kappa_{\mathrm{spec}}$: RLVR remains order-one relative to the isotropic
dimensional reference, whereas distilled SFT lies roughly two to three orders
of magnitude above it.
}
 
      \label{fig:14b_samebase}
\end{figure}

\begin{figure}[!ht]
  \centering
  \includegraphics[width=\linewidth]{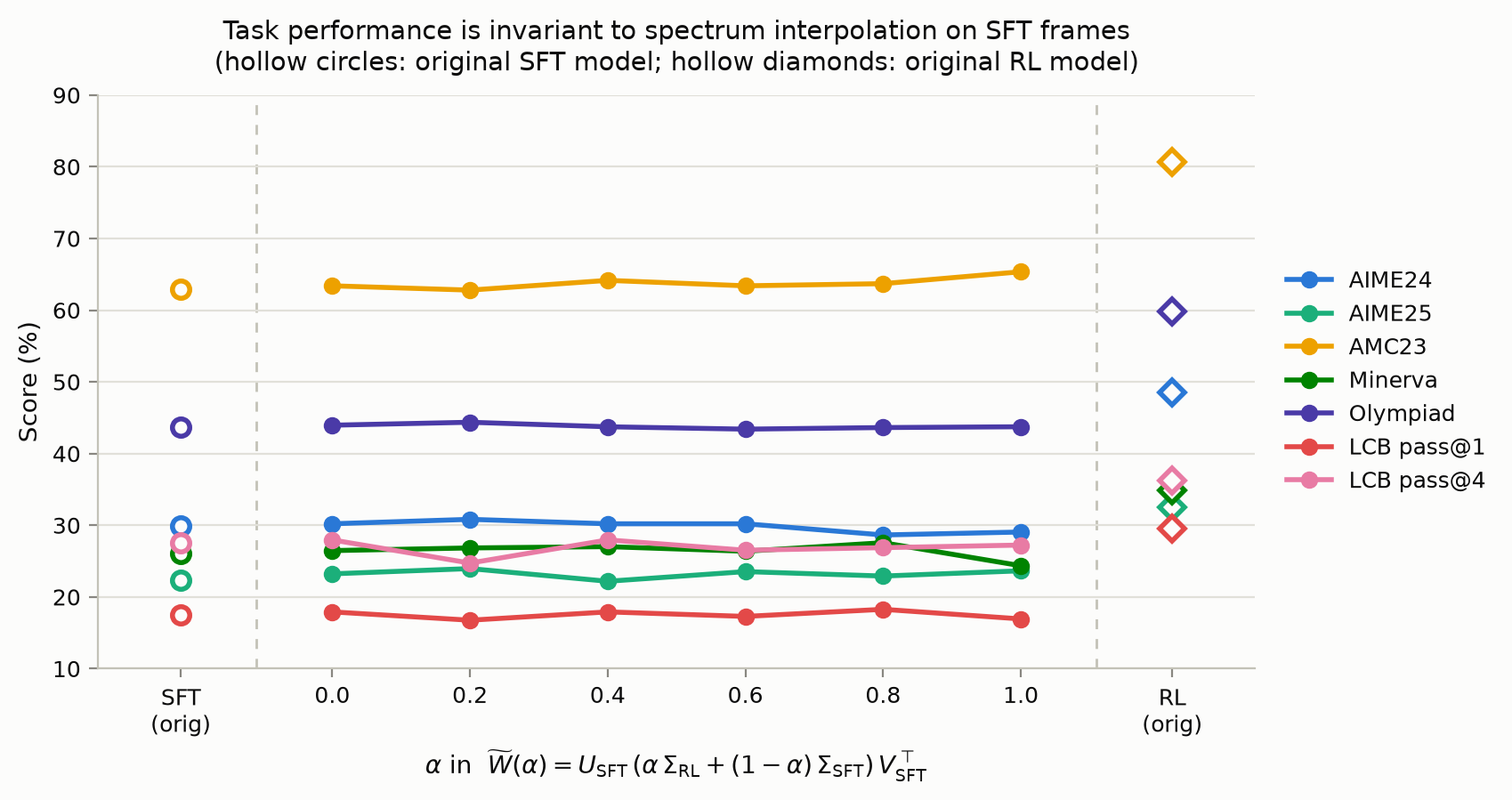}
  \caption{
  \textbf{Task performance is insensitive to spectral interpolation in the
  SFT singular frames.}
  We hold the SFT left and right singular frames fixed and interpolate only
  the spectrum,
  $\widetilde W(\alpha)=U_{\mathrm{SFT}}
  [(1-\alpha)\Sigma_{\mathrm{SFT}}+\alpha\Sigma_{\mathrm{RL}}]
  V_{\mathrm{SFT}}^\top$, for $\alpha\in[0,1]$.
  Performance remains nearly unchanged as an arbitrary proportion of the RL
  spectrum is substituted into the SFT frames. Hollow circles and diamonds
  show the original SFT and RL checkpoints, respectively. This complements
  Fig.~\ref{fig:spectrum_functional_tests}(a), which performs the converse
  intervention by restoring the SFT/base spectrum in the RL-trained frames.
  }
  \label{fig:sft_rebase}
\end{figure}

\section{ISO-Merger Details}
\label{app:iso_merger_details}

This appendix expands the construction in Section~\ref{sec:iso_merger},
including the shared frame coordinates, sign alignment, tangent projection,
mode masking, retention coefficients, and feasibility restoration.
Algorithm~\ref{alg:iso_merger} gives the complete procedure.

\subsection{Setup and Shared Coordinates}
\label{app:iso_merger_setup}

ISO-Merger takes a shared base checkpoint $\theta_0$ and $K\geq1$
RL expert checkpoints $\{\theta_i\}_{i=1}^K$ with corresponding
parameter names and shapes.
We study specialists trained from the same base for distinct domains
or capabilities, rather than already-generalist policies trained on
broad, substantially overlapping mixtures.
Composition uses only these checkpoints, with no post-merge data,
rollouts, gradient updates, or distillation.

For a corresponding 2D weight matrix, let
$q=\min\{d_{\mathrm{out}},d_{\mathrm{in}}\}$ and compute the thin SVDs
\begin{equation}
\label{eq:app_iso_merger_setup}
    W_0=U_0\Sigma_0V_0^\top,
    \qquad
    W_i=U_i\Sigma_iV_i^\top,
    \qquad i=1,\ldots,K.
\end{equation}
All singular values are ordered nonincreasingly, with
$U_i\in\mathrm{St}(d_{\mathrm{out}},q)$ and
$V_i\in\mathrm{St}(d_{\mathrm{in}},q)$, including $i=0$.
ISO-Merger reuses the base spectrum $\Sigma_0$ and represents expert $i$
through its learned frames under that spectrum:
\begin{equation}
\label{eq:app_iso_merger_expert_representative}
    U_i\Sigma_0V_i^\top\in\mathcal F(W_0).
\end{equation}
The expert spectra $\Sigma_i$ need not equal $\Sigma_0$; the input experts
need not have been trained with ISO.
Rather than combining Euclidean task vectors, we compose frame
displacements anchored at the same base factors $(U_0,V_0)$.
All subsequent quantities and coefficients are computed separately
for each 2D weight matrix; matrix indices are omitted for readability.

\subsection{Sign Canonicalization}
\label{app:gauge}

For a positive, simple singular value, the corresponding singular-vector
pair is determined only up to a joint sign:
$(u_k,v_k)$ and $(-u_k,-v_k)$ represent the same matrix.
Consequently, subtracting unaligned frames can produce displacement
columns of norm approximately two even when their represented direction
has barely changed.
We align each expert's rank-indexed singular-vector pairs to the base:
\begin{equation}
\label{eq:app_iso_merger_signs}
    s_{i,k}:=\operatorname{sign}
        \bigl(\langle u_{0,k},u_{i,k}\rangle\bigr),
    \qquad \operatorname{sign}(0):=1,
\end{equation}
\begin{equation}
\label{eq:app_iso_merger_sign_alignment}
    u_{i,k}\leftarrow s_{i,k}u_{i,k},
    \qquad
    v_{i,k}\leftarrow s_{i,k}v_{i,k}.
\end{equation}
Applying the same sign to both frames leaves $W_i$ unchanged.
All later occurrences of $U_i$ and $V_i$ refer to these aligned frames.
Exactly repeated singular values admit additional within-block
rotations, such as $U_{i,C}\mapsto U_{i,C}Q$ and
$V_{i,C}\mapsto V_{i,C}Q$ for an orthogonal $Q$.
The procedure resolves joint signs but does not explicitly canonicalize
this additional freedom.

\subsection{Stiefel Tangent Projection}
\label{app:iso_merger_projection}

For an anchor $U_0\in\mathrm{St}(d_{\mathrm{out}},q)$,
\begin{equation}
\label{eq:app_iso_merger_tangent_space}
    T_{U_0}\mathrm{St}(d_{\mathrm{out}},q)
    =\left\{\xi\in\mathbb R^{d_{\mathrm{out}}\times q}:
        U_0^\top\xi+\xi^\top U_0=0\right\}.
\end{equation}
The orthogonal projection under the Frobenius inner product is
\begin{equation}
\label{eq:app_stiefel_tangent_projection}
    \Pi_{U_0}(Z)
    =Z-U_0\operatorname{sym}(U_0^\top Z),
    \qquad \operatorname{sym}(A)=\tfrac12(A+A^\top).
\end{equation}
After sign alignment, define the finite frame displacements
\begin{equation}
\label{eq:iso_merger_frame_displacements}
    \Delta U_i:=U_i-U_0,
    \qquad \Delta V_i:=V_i-V_0,
\end{equation}
and their tangent projections
\begin{equation}
\label{eq:app_iso_merger_tangent_directions}
    \xi_{U,i}=\Pi_{U_0}(\Delta U_i),
    \qquad \xi_{V,i}=\Pi_{V_0}(\Delta V_i).
\end{equation}
The right-frame projection is defined analogously at $V_0$.
Using a common anchor makes the expert directions directly composable
in shared frame coordinates.

\subsection{Singular-Mode Masking}
\label{app:topk}

The masking step is motivated by the empirical observation that
trailing singular directions are less stable across experts and that
retaining their displacements can degrade merged performance.
For a keep ratio $\rho_{\mathrm{keep}}\in(0,1]$, define
\begin{equation}
\label{eq:iso_merger_mask}
\begin{aligned}
    k_{\mathrm{keep}}
    &:=\operatorname{round}(\rho_{\mathrm{keep}}q),\\
    D_{\mathrm{keep}}
    &:=\operatorname{Diag}
       \left((\mathbf 1\{k\leq k_{\mathrm{keep}}\})_{k=1}^{q}\right).
\end{aligned}
\end{equation}
We retain the leading columns of the projected frame displacements:
\begin{equation}
\label{eq:iso_merger_masked_directions}
    \widetilde\xi_{U,i}=\xi_{U,i}D_{\mathrm{keep}},
    \qquad
    \widetilde\xi_{V,i}=\xi_{V,i}D_{\mathrm{keep}}.
\end{equation}
Unless otherwise stated, $\rho_{\mathrm{keep}}=0.9$.
This operation selects displacement coordinates, not singular values:
the full $q$-mode spectrum $\Sigma_0$ is retained in the final
reconstruction, so masking is not a low-rank truncation of the weight.

Column masking need not preserve the Stiefel tangent constraints.
We therefore use the masked displacements to construct local effect
proxies and form their weighted aggregate, then restore tangent
feasibility before retraction as described in Appendix~\ref{app:retract}.

\subsection{Retention Coefficients with a Unit-Retention Target}
\label{app:iso_merger_retention}

\paragraph{Local effect proxies.}
Linearizing the reconstruction map $W(U,V)=U\Sigma_0V^\top$ at
$(U_0,V_0)$ gives the local first-order effect proxy
\begin{equation}
\label{eq:app_iso_merger_effect_proxy}
    g_i
    :=\widetilde\xi_{U,i}\Sigma_0V_0^\top
      +U_0\Sigma_0\widetilde\xi_{V,i}^\top.
\end{equation}
The proxies are computed from the masked displacements before
aggregate tangent projection.
We measure their interactions through the Gram matrix
\begin{equation}
\label{eq:iso_merger_gram}
    \Gamma_{ij}:=\langle g_i,g_j\rangle_F,
    \qquad \Gamma\in\mathbb R^{K\times K},
\end{equation}
where $\langle A,B\rangle_F=\operatorname{tr}(A^\top B)$.
For coefficients $c\in\mathbb R^K$, let
$g(c)=\sum_{i=1}^K c_i g_i$.
For every nonzero proxy $g_i$, define its self-retention by
\begin{equation}
\label{eq:iso_merger_retention}
    \operatorname{ret}_i(c)
    :=\frac{\langle g(c),g_i\rangle_F}{\|g_i\|_F^2}
    =\frac{(\Gamma c)_i}{\Gamma_{ii}}.
\end{equation}
The unit-retention target $\operatorname{ret}_i(c)=1$ asks the
merged proxy to retain a unit projection coefficient along each
expert proxy. The corresponding ideal equations are
\begin{equation}
\label{eq:app_iso_merger_ideal_system}
    \Gamma c=b,
    \qquad b:=\operatorname{diag}(\Gamma).
\end{equation}

\paragraph{Zero proxies, ridge stabilization, and clipping.}
Let $\mathcal I=\{i:\|g_i\|_F>0\}$ denote the active experts for
this matrix. Zero-norm proxies impose no retention constraint;
we omit their rows and columns and set their coefficients to zero.
When $\mathcal I\neq\varnothing$, we solve
\begin{equation}
\label{eq:app_iso_merger_retention_system}
    \bigl(\Gamma_{\mathcal I\mathcal I}
          +\lambda_{\mathrm{ridge}}I_{|\mathcal I|}\bigr)
       \bar c_{\mathcal I}
    =b_{\mathcal I},
\end{equation}
and apply entrywise clipping:
\begin{equation}
\label{eq:app_iso_merger_clipping}
    c_i^\star=
    \begin{cases}
        \min\{c_{\max},\max\{c_{\min},\bar c_i\}\},
            & i\in\mathcal I,\\
        0, & i\notin\mathcal I.
    \end{cases}
\end{equation}
We use $\lambda_{\mathrm{ridge}}=10^{-12}$ as a small numerical
stabilizer and $[c_{\min},c_{\max}]=[0,1.5]$ to keep the
coefficients nonnegative and limit amplification.
If $\mathcal I=\varnothing$, the Gram solve is skipped and all
coefficients remain zero.

Unit self-retention is a target in the local proxy space.
Ridge stabilization and clipping can change the attained proxy
retention, and subsequent tangent projection and nonlinear
retraction further change the resulting weight displacement.
The procedure therefore does not guarantee exact retention
in the reconstructed model or its task performance.

\subsection{Retraction and Reconstruction}
\label{app:retract}

We combine the masked expert displacements using $c^\star$ and
project the aggregates back onto the tangent spaces:
\begin{equation}
\label{eq:iso_merger_aggregate_projection}
\begin{aligned}
    \xi_{U,\star}
    &=\Pi_{U_0}\!\left(
        \sum_{i=1}^K c_i^\star\widetilde\xi_{U,i}\right),\\
    \xi_{V,\star}
    &=\Pi_{V_0}\!\left(
        \sum_{i=1}^K c_i^\star\widetilde\xi_{V,i}\right).
\end{aligned}
\end{equation}

\paragraph{Polar retraction.}
For $X\in\mathbb R^{d\times q}$ with $d\geq q$, write a thin
SVD $X=P_XS_XQ_X^\top$ and define
\begin{equation}
\label{eq:polar_retraction}
    \operatorname{polar}(X):=P_XQ_X^\top.
\end{equation}
This has orthonormal columns and is a nearest Stiefel factor to $X$
in Frobenius norm.
When $X$ has full column rank,
\begin{equation}
\label{eq:app_polar_inverse_sqrt}
    \operatorname{polar}(X)=X(X^\top X)^{-1/2}.
\end{equation}
For rank-deficient $X$, the SVD expression selects one nearest
Stiefel factor, which need not be unique.

The ISO-Merger retraction arguments have full column rank in
exact arithmetic. In particular, tangent feasibility gives
\begin{equation}
\label{eq:app_iso_merger_full_rank}
\begin{aligned}
    (U_0+\xi_{U,\star})^\top(U_0+\xi_{U,\star})
    &=I_q+\xi_{U,\star}^\top\xi_{U,\star}\succ0,\\
    (V_0+\xi_{V,\star})^\top(V_0+\xi_{V,\star})
    &=I_q+\xi_{V,\star}^\top\xi_{V,\star}\succ0.
\end{aligned}
\end{equation}
The final factors and merged weight are
\begin{equation}
\label{eq:iso_merger_output}
\begin{aligned}
    U_\star&=\operatorname{polar}(U_0+\xi_{U,\star}),\\
    V_\star&=\operatorname{polar}(V_0+\xi_{V,\star}),\\
    W_\star&=U_\star\Sigma_0V_\star^\top.
\end{aligned}
\end{equation}
Because both frames have orthonormal columns,
$W_\star\in\mathcal F(W_0)$ in exact arithmetic; the
implementation preserves this property up to floating-point error.
If all coefficients are zero, this construction returns the base
matrix in exact arithmetic.

\subsection{Merged Parameter Scope}
\label{app:iso_merger_scope}

All per-layer 2D projection matrices, together with embedding and
unembedding matrices, are merged using the construction above.
One-dimensional parameters, including normalization scales and
attention biases, use a task-vector average:
\begin{equation}
\label{eq:app_iso_merger_1d_average}
    w_\star
    =w_0+\frac1K\sum_{i=1}^K(w_i-w_0)
    =\frac1K\sum_{i=1}^K w_i.
\end{equation}
These parameters do not use the matrix-specific retention coefficients.

\subsection{Full ISO-Merger Algorithm}
\label{app:iso_merger_algorithm}

Algorithm~\ref{alg:iso_merger} follows the same order as the preceding
construction: sign alignment, tangent projection, mode masking,
proxy-based coefficient computation, aggregate projection, and
polar reconstruction.

\begin{algorithm}[H]
\caption{ISO-Merger}
\label{alg:iso_merger}
\begin{algorithmic}[1]
\Require Base model $\theta_0$ and $K$ shared-base experts
         $\{\theta_i\}_{i=1}^K$
\Require Keep ratio $\rho_{\mathrm{keep}}=0.9$,
         ridge $\lambda_{\mathrm{ridge}}=10^{-12}$,
         clip bounds $[c_{\min},c_{\max}]=[0,1.5]$
\Ensure Merged model $\theta_\star$
\For{each 2D projection, embedding, or unembedding parameter}
    \State Read corresponding matrices $W_0$ and $\{W_i\}_{i=1}^K$
    \State Compute the thin SVD $W_0=U_0\Sigma_0V_0^\top$
    \State $q\gets\min\{d_{\mathrm{out}},d_{\mathrm{in}}\}$;
           $k_{\mathrm{keep}}\gets
             \operatorname{round}(\rho_{\mathrm{keep}}q)$
    \State $D_{\mathrm{keep}}\gets
           \operatorname{Diag}
             ((\mathbf1\{k\leq k_{\mathrm{keep}}\})_{k=1}^q)$
    \For{$i=1,\ldots,K$}
        \State Compute the thin SVD $W_i=U_i\Sigma_iV_i^\top$
        \For{$k=1,\ldots,q$}
            \State $s_{i,k}\gets\operatorname{sign}
                   (\langle u_{0,k},u_{i,k}\rangle)$
                   \Comment{$\operatorname{sign}(0)=1$}
            \State $u_{i,k}\gets s_{i,k}u_{i,k}$;
                   $v_{i,k}\gets s_{i,k}v_{i,k}$
        \EndFor
        \State $\xi_{U,i}\gets\Pi_{U_0}(U_i-U_0)$;
               $\xi_{V,i}\gets\Pi_{V_0}(V_i-V_0)$
        \State $\widetilde\xi_{U,i}\gets\xi_{U,i}D_{\mathrm{keep}}$;
               $\widetilde\xi_{V,i}\gets\xi_{V,i}D_{\mathrm{keep}}$
        \State $g_i\gets
               \widetilde\xi_{U,i}\Sigma_0V_0^\top
               +U_0\Sigma_0\widetilde\xi_{V,i}^\top$
    \EndFor
    \State $c^\star\gets0\in\mathbb R^K$;
           $\mathcal I\gets\{i:\|g_i\|_F>0\}$
    \If{$\mathcal I\neq\varnothing$}
        \State Form $\Gamma_{\mathcal I\mathcal I}$ with entries
               $\Gamma_{ij}=\langle g_i,g_j\rangle_F$
        \State $b_{\mathcal I}\gets
               \operatorname{diag}(\Gamma_{\mathcal I\mathcal I})$
        \State Solve $(\Gamma_{\mathcal I\mathcal I}
               +\lambda_{\mathrm{ridge}}I_{|\mathcal I|})
               \bar c_{\mathcal I}=b_{\mathcal I}$
        \State $c_{\mathcal I}^\star\gets
               \operatorname{clip}
                 (\bar c_{\mathcal I};c_{\min},c_{\max})$
    \EndIf
    \State $\xi_{U,\star}\gets
           \Pi_{U_0}(\sum_{i=1}^K c_i^\star\widetilde\xi_{U,i})$
    \State $\xi_{V,\star}\gets
           \Pi_{V_0}(\sum_{i=1}^K c_i^\star\widetilde\xi_{V,i})$
    \State $U_\star\gets\operatorname{polar}(U_0+\xi_{U,\star})$;
           $V_\star\gets\operatorname{polar}(V_0+\xi_{V,\star})$
    \State Store $W_\star=U_\star\Sigma_0V_\star^\top$
           in the corresponding parameter of $\theta_\star$
\EndFor
\For{each 1D parameter}
    \State Store $w_\star=w_0+\tfrac1K\sum_{i=1}^K(w_i-w_0)$
           in $\theta_\star$
\EndFor
\State \Return $\theta_\star$
\end{algorithmic}
\end{algorithm}

\section{Data-Free Merging Experimental details and Additional Results}
\label{app:merging_details}

\subsection{Baselines}

We compare against five data-free baselines: Task Arithmetic~\cite{ta}, TIES~\cite{ties}, TSV-Merge~\cite{tsv}, RAM~\cite{ram}, and OrthoMerge-G-TIES~\cite{yang2026orthogonal}.

The first four methods operate on Euclidean task-vector representations, while OrthoMerge-G-TIES additionally introduces a geometry-aware orthogonal component and is described separately below. Let $\theta_0$ denote the shared base model and
let $\{\theta_i\}_{i=1}^{K}$ be $K$ expert models fine-tuned from this same
initialization. The task vector of expert $i$ is
\begin{equation}
    \tau_i := \theta_i-\theta_0.
\end{equation}
A data-free merging method constructs a merged update $\tau_{\merge}$ without
using training data and outputs
\begin{equation}
    \theta_{\merge} = \theta_0+\tau_{\merge}.
\end{equation}
For a matrix-shaped parameter in layer $\ell$, we use the layer-wise notation
\begin{equation}
    \Delta W_i^{(\ell)}
    :=
    W_i^{(\ell)}-W_0^{(\ell)},
    \qquad W^{(\ell)} \in \R^{m_\ell \times n_\ell}.
\end{equation}

\textbf{Task Arithmetic} merges experts by linearly combining their task vectors and
adding the resulting update back to the base model:
\begin{equation}
    \tau_{\TA} = \lambda \sum_{i=1}^{K} \tau_i,
    \qquad
    \theta_{\TA} = \theta_0+\tau_{\TA}.
\end{equation}

Task Arithmetic is computationally cheap, requiring only vector additions and
a global scaling coefficient. However, it does not explicitly address sign
conflicts, redundant updates, layer-wise geometry, or sparse task-specific
signals. As a result, interference and signal dilution can occur when expert
updates overlap or when important updates are distributed over disjoint parameter
regions.

\textbf{TIES-Merging} follows three steps: \emph{Trim}, \emph{Elect Sign}, and
\emph{Merge}. It first removes small-magnitude entries from
each task vector, then elects a consensus sign for every coordinate, and finally
averages only the updates whose signs agree with the elected sign.

Let $\rho_{\mathrm{TIES}}\in(0,1]$ be the retained density. The trimmed update
for expert $i$ is
\begin{equation}
    m_i = \TopKMask(|\tau_i|,\rho_{\mathrm{TIES}}),
    \qquad
    \widetilde{\tau}_i = m_i \odot \tau_i,
\end{equation}
where $m_i^p=1$ indicates that coordinate $p$ is among the largest
$\rho_{\mathrm{TIES}}$ fraction of entries of $|\tau_i|$. The elected sign at
coordinate $p$ is
\begin{equation}
    \gamma^p = \sign\!\left(\sum_{i=1}^{K} \widetilde{\tau}_i^p\right).
\end{equation}
TIES then keeps only the experts aligned with this sign:
\begin{equation}
    A_p = \left\{ i : \widetilde{\tau}_i^p \neq 0,
    \; \sign(\widetilde{\tau}_i^p)=\gamma^p \right\}.
\end{equation}
The merged coordinate is
\begin{equation}
    \tau_{\TIES}^p =
    \begin{cases}
        \dfrac{1}{|A_p|}\sum_{i\in A_p}\widetilde{\tau}_i^p, & |A_p|>0,\\[6pt]
        0, & |A_p|=0,
    \end{cases}
\end{equation}
and the final model is
\begin{equation}
    \theta_{\TIES}=\theta_0+\lambda\tau_{\TIES}.
\end{equation}

TIES remains an element-wise method, but it is more robust than naive averaging
when many small updates are irrelevant or when experts update the same coordinate
in opposite directions. Its main limitation is that it does not explicitly model
matrix-level subspace geometry.

\textbf{TSV-Merge}, short for Task Singular Vectors Merge, operates on layer-wise task matrices rather than fully flattened vectors. For each 2D
weight matrix, TSV decomposes each expert update with SVD, keeps the leading
singular components, orthogonalizes the concatenated singular-vector bases, and
reconstructs the merged matrix.

For layer $\ell$, compute
\begin{equation}
    \Delta W_i^{(\ell)}=U_i\Sigma_iV_i^\top.
\end{equation}
After truncating to rank $k_\ell$,
\begin{equation}
    \Delta W_i^{(\ell)} \approx
    U_i^{(k)}\Sigma_i^{(k)}{V_i^{(k)}}^\top.
\end{equation}
The truncated bases and spectra are concatenated as
\begin{equation}
    U=[U_1^{(k)}|\cdots|U_K^{(k)}],\qquad
    V=[V_1^{(k)}|\cdots|V_K^{(k)}],\qquad
    \Sigma=\blockdiag(\Sigma_1^{(k)},\ldots,\Sigma_K^{(k)}).
\end{equation}
TSV reduces singular-vector interference by orthogonalizing $U$ and $V$. In the
whitening form,
\begin{equation}
    U_\perp = U(U^\top U)^{-1/2},
    \qquad
    V_\perp = V(V^\top V)^{-1/2}.
\end{equation}
A numerically stable implementation can use the orthogonal Procrustes form. If
$U=P_U D_U Q_U^\top$ and $V=P_V D_V Q_V^\top$, then
\begin{equation}
    U_\perp=P_UQ_U^\top,
    \qquad
    V_\perp=P_VQ_V^\top.
\end{equation}
The merged task matrix is reconstructed as
\begin{equation}
    \widehat{\Delta}^{(\ell)}=U_\perp\Sigma V_\perp^\top,
    \qquad
    W_{\TSV}^{(\ell)}=W_0^{(\ell)}+\alpha\widehat{\Delta}^{(\ell)}.
\end{equation}
For non-matrix parameters, such as biases or normalization vectors, TSV typically
falls back to Task Arithmetic.

TSV is a spectral, subspace-level method. By decorrelating singular-vector bases,
it can reduce interference that is invisible to element-wise rules. This comes at
a higher computational cost because SVDs are required for matrix-shaped layers.
The rank $k_\ell$ is either treated as a hyperparameter or chosen through an
automatic rank-reduction rule.

\textbf{RAM} is designed for sparse RL updates, where naive averaging can dilute
coordinates that are important for only one expert. For each coordinate $p$, RAM
identifies the experts with non-negligible updates:
\begin{equation}
    A_p = \left\{ i \in \{1,\ldots,K\}: |\tau_i^p| > \epsilon_{\mathrm{act}} \right\},
    \qquad
    c^p = |A_p|,
\end{equation}
where $\epsilon_{\mathrm{act}}$ is an active-update threshold. The merged
coordinate is
\begin{equation}
\tau_{\RAM}^p =
\begin{cases}
0, & c^p = 0, \\
\tau_j^p, & c^p = 1 \text{ and } A_p = \{j\}, \\
\dfrac{1}{c^p}\displaystyle\sum_{i\in A_p}\tau_i^p, & c^p \ge 2.
\end{cases}
\label{eq:ram-coordinate}
\end{equation}
Thus, inactive coordinates are discarded, task-unique coordinates are preserved
without being divided by $K$, and shared coordinates are averaged only over the
experts that actively update them. The final model is
\begin{equation}
    \theta_{\RAM} = \theta_0+\tau_{\RAM}.
\end{equation}

RAM is element-wise like TIES, but its selection criterion is different. Instead
of trimming by global magnitude density and resolving signs, RAM separates
inactive, unique, and shared update regions. This makes it suitable when
RL-trained experts store task-specific behavior in sparse and partially
non-overlapping parameter subsets.

\textbf{OrthoMerge}~\cite{yang2026orthogonal} is a geometry-preserving merging
framework that factors each expert update into an \emph{implicit orthogonal
transformation} of the base weights and a Euclidean residual, merges the
orthogonal components on the rotation manifold, and merges the residuals with a
standard element-wise rule. The released implementation provides six variants:
two Procrustes-target constructions (\textsc{C}, which fits the rotation only to
rows whose update conflicts with the mean task vector, and \textsc{G}, which
fits it to the full expert weights) combined with three residual backends
(Task Arithmetic, TIES, TSV). We report the \textsc{G}-TIES variant for two reasons. First, TIES is a strong
residual backend in both of our merging settings: among the element-wise
baselines it attains the best or near-best overall average on both backbones
(Tab.~\ref{tab:great_merge1} and~\ref{tab:great_merge2}), so we pair OrthoMerge with
the classical method that is already competitive in our comparison. Second,
the \textsc{G} construction fits the Procrustes rotation to the \emph{full}
expert weights, so the extracted transformation represents the entire expert
update. The \textsc{C} construction instead fits it only to the rows that
conflict with the mean task vector. Since ISO-Merger likewise composes each
expert's full update---represented as motion of its singular frames under the
fixed base spectrum---\textsc{G}-TIES is the directly comparable variant.

For each 2D weight matrix (excluding layer norms), OrthoMerge first solves a
one-sided orthogonal Procrustes problem per expert,
\begin{equation}
    R_i^{(\ell)}
    = \argmin_{R \in \mathrm{O}(n_\ell)}
      \bigl\lVert W_0^{(\ell)} R - W_i^{(\ell)} \bigr\rVert_F
    = \bar{U}\bar{V}^\top,
    \qquad
    {W_0^{(\ell)}}^\top W_i^{(\ell)} = \bar{U}\bar{S}\bar{V}^\top,
\end{equation}
so that the rotation acts only on the input (column) space of the layer. Each
rotation is mapped to a skew-symmetric matrix by the inverse Cayley transform,
\begin{equation}
    A_i^{(\ell)} = \bigl(R_i^{(\ell)} + I\bigr)^{-1}\bigl(R_i^{(\ell)} - I\bigr),
    \qquad
    A_i^{(\ell)} = -{A_i^{(\ell)}}^\top,
\end{equation}
and the $K$ skew matrices are aggregated with a magnitude-corrected mean that
averages direction and intensity separately,
\begin{equation}
    A_{\mathrm{mrg}}^{(\ell)}
    = \Bigl(\tfrac{1}{K}\textstyle\sum_{i=1}^{K}
      \bigl\lVert A_i^{(\ell)} \bigr\rVert_F\Bigr)
      \cdot
      \frac{\sum_{i=1}^{K} A_i^{(\ell)}}
           {\bigl\lVert \sum_{i=1}^{K} A_i^{(\ell)} \bigr\rVert_F},
\end{equation}
which prevents the norm shrinkage that a plain average of nearly orthogonal
directions would induce. The merged rotation is recovered with the forward
Cayley map and applied to the base weights,
\begin{equation}
    R_{\mathrm{mrg}}^{(\ell)}
    = \bigl(I - A_{\mathrm{mrg}}^{(\ell)}\bigr)^{-1}
      \bigl(I + A_{\mathrm{mrg}}^{(\ell)}\bigr),
    \qquad
    W_{\mathrm{rot}}^{(\ell)} = W_0^{(\ell)}R_{\mathrm{mrg}}^{(\ell)}.
\end{equation}
The part of each expert update not captured by its own rotation is treated as a
Euclidean residual,
\begin{equation}
    \delta_i^{(\ell)} = W_i^{(\ell)}-W_0^{(\ell)}R_i^{(\ell)},
\end{equation}
(for non-matrix parameters $\delta_i = \tau_i$), and the residuals are merged
with the TIES rule of trimming, sign election, and sign-consistent averaging
described above. The final model is
\begin{equation}
    W_{\mathrm{OM}}^{(\ell)}
    = W_{\mathrm{rot}}^{(\ell)} + \lambda\, \tau_{\TIES}\bigl(\{\delta_i\}\bigr).
\end{equation}

OrthoMerge is the geometry-aware baseline closest to ISO-Merger, but the two
methods act in different coordinates. OrthoMerge applies a one-sided
orthogonal transformation on the input side of each matrix:
$W_0^{(\ell)}R$ preserves the left singular directions and singular values of
$W_0^{(\ell)}$, so any output-side frame change must be represented through
its Euclidean residual path.

ISO-Merger instead composes two-sided frame changes under the shared base
spectrum, adapting both $U$ and $V$ in fixed-spectrum Stiefel coordinates.
OrthoMerge's rotation class is contained in the more permissive
output-subspace-retaining class $\mathcal C_L$ analyzed in
Section~\ref{sec:transport}. Even this larger class leaves a substantial
unexplained-update ratio, whereas retaining the incoming spectrum while
adapting both frames gives a low-residual description. This motivates
ISO-Merger's two-sided frame parameterization.

\subsection{Experimental details}
\paragraph{Backbones and experts.}
For the Qwen2.5 setting, we use \texttt{Qwen2.5-7B-Instruct} as the shared base and merge three RLVR experts: CURE for coding~\cite{wang2025code}, ToolRL for tool use~\cite{qian2025toolrl}, and
MemAgent for long-context memory~\cite{yu2025memagent}. For the DeepSeek-R1-Distill setting, we use
\texttt{DeepSeek-R1-Distill-Qwen\allowbreak-1.5B} as the shared base and merge Archer2.0 for coding~\cite{archer} and JustRL for math~\cite{he2025justrl}.

\paragraph{Baseline hyperparameters.}
We follow the official implementations for all baselines. For TIES, we use
$\lambda=1.0$ and $\rho_{\mathrm{TIES}}=0.2$. For TSV-Merge, we use
$\alpha=1.0$ with automatic rank reduction. For RAM, we use
$\epsilon=10^{-5}$. For ISO-Merger, all SVD operations are performed in FP64. For OrthoMerge-G-TIES, we use the official implementation with its default
hyperparameters: FP32 SVD-based Procrustes, $20\%$ retained density and
$\lambda=1.0$ for the TIES residual stage. An additional FP64-SVD comparison for OrthoMerge-G-TIES is reported in
Appendix~\ref{app:orthomerge_precision}.

\paragraph{Evaluation protocol.}
For the Qwen2.5 setting, we evaluate coding on LiveBench and LiveCodeBench,
tool use on BFCL, and long-context memory on RULER HotpotQA and SQuAD at
32K--64K context lengths. For the DeepSeek-R1-Distill setting, we evaluate
coding on LiveBench and LiveCodeBench, and math on AIME24, AIME25, AMC23,
Minerva, and OlympiadBench. Unless otherwise specified, we generate 16 rollouts
per prompt and report means and standard deviations across three evaluation seeds (42, 123, and 1234).
BFCL scores are obtained from a single evaluation per merged model.

\paragraph{Decoding settings.}
For DeepSeek-R1-Distill math and coding evaluations, we use temperature $0.6$
and top-$p=0.95$. For Qwen2.5 LiveBench and LiveCodeBench, we use temperature
$1.0$. 
For BFCL, we use single-sample near-greedy decoding with temperature
$0.001$, following the official evaluation protocol.
For RULER
long-context evaluation, we use temperature $0.7$. Inference is conducted with
vLLM on NVIDIA H100 GPUs.

\subsection{Additional results}

\subsubsection{Effect of SVD Precision in OrthoMerge}
\label{app:orthomerge_precision}
To test whether higher SVD precision closes the gap of OrthoMerge~\cite{yang2026orthogonal} to ISO-Merger,
we additionally evaluate OrthoMerge-G-TIES with FP64 SVDs,
keeping all other merging and evaluation settings unchanged.
On \texttt{Qwen2.5-7B-Instruct} with three RL experts
(Tab.~\ref{tab:orthomerge_precision}), OrthoMerge achieves
aggregate scores of $62.16$ with its official FP32 configuration
and $61.36$ with FP64, compared with $63.80$ for ISO-Merger.
Thus, ISO-Merger retains its aggregate advantage under both
configurations; matching SVD precision does not close the
observed performance gap in this setting.

\begin{table*}[t]
\centering
\small
\renewcommand{\arraystretch}{1.15}
\setlength{\tabcolsep}{3pt}
\caption{
\textbf{Effect of SVD precision on OrthoMerge-G-TIES.}
We compare its official FP32 configuration with an FP64-SVD
variant on \texttt{Qwen2.5-7B-Instruct} with three RL experts,
keeping all other merging settings fixed.
ISO-Merger uses FP64 SVDs and is included for reference.
Metrics and evaluation protocols follow
Table~\ref{tab:great_merge1}.
The highest mean score in each column is shown in \textbf{bold}.
}
\label{tab:orthomerge_precision}
\resizebox{\textwidth}{!}{
\begin{tabular}{l c cccc cc cccc c}
\toprule
& &
\multicolumn{4}{c}{Coding} &
\multicolumn{2}{c}{Tool Use} &
\multicolumn{4}{c}{Memory} & \\
\cmidrule(lr){3-6}
\cmidrule(lr){7-8}
\cmidrule(lr){9-12}
& &
\multicolumn{2}{c}{LB} &
\multicolumn{2}{c}{LCB v2} &
Live & Non-Live &
\multicolumn{2}{c}{HotpotQA} &
\multicolumn{2}{c}{SQuAD} & \\
\cmidrule(lr){3-4}
\cmidrule(lr){5-6}
\cmidrule(lr){7-7}
\cmidrule(lr){8-8}
\cmidrule(lr){9-10}
\cmidrule(lr){11-12}
Method & SVD &
ACC & UT & ACC & UT &
ACC & ACC &
32K & 64K & 32K & 64K & Avg \\
\midrule

OrthoMerge-G-TIES
& FP32
& 40.74 $\pm$ 2.00
& \textbf{53.22 $\pm$ 1.90}
& \textbf{31.98 $\pm$ 0.84}
& \textbf{47.48 $\pm$ 0.77}
& 66.98
& 78.17
& 74.79 $\pm$ 0.49
& 74.20 $\pm$ 0.49
& 77.75 $\pm$ 0.12
& 76.32 $\pm$ 0.21
& 62.16 \\

OrthoMerge-G-TIES
& FP64
& 37.91 $\pm$ 1.70
& 49.37 $\pm$ 1.50
& 31.24 $\pm$ 1.03
& 46.42 $\pm$ 0.97
& 66.59
& 78.00
& 75.33 $\pm$ 0.56
& 73.92 $\pm$ 0.35
& 77.99 $\pm$ 0.47
& 76.84 $\pm$ 0.17
& 61.36 \\

\midrule
\rowcolor{ourcolor}
ISO-Merger
& FP64
& \textbf{41.81 $\pm$ 2.00}
& 52.73 $\pm$ 1.57
& 31.06 $\pm$ 1.88
& 45.69 $\pm$ 3.06
& \textbf{72.32}
& \textbf{81.07}
& \textbf{79.46 $\pm$ 0.78}
& \textbf{76.77 $\pm$ 0.27}
& \textbf{79.10 $\pm$ 0.14}
& \textbf{78.01 $\pm$ 0.32}
& \textbf{63.80} \\

\bottomrule
\end{tabular}
}
\end{table*}

\subsubsection{Metrics: best@4 and worst@4}\label{app:iso_merger_best_worst}
Among $n$ sampled responses, let $c$ denote the number with binary correctness
equal to one. For $n\geq k$, we report
\[
    \widehat{\mathrm{best@}k}
    =
    1-
    \frac{\binom{n-c}{k}}{\binom nk},
    \qquad
    \widehat{\mathrm{worst@}k}
    =
    \frac{\binom ck}{\binom nk}.
\]

In addition to average@16, we report best@4 and worst@4 to characterize the
upper and lower tails of stochastic decoding performance. 

\begin{itemize}
    \item \textbf{best@k}$=$pass@k, the probability that at least one of $k$ randomly drawn samples solves the problem (best-case behaviour over $k$ draws).
    \item \textbf{worst@k}$=\binom{c}{k}\big/\binom{n}{k}$, the probability that all $k$ randomly drawn samples solve the problem (consistency / worst-case reading).
\end{itemize}

For the Qwen2.5 setting, unit-test accuracy of LB/LCB coding benchmark is the fraction of unit tests passed, a real value in $[0,1]$ rather than a binary outcome. BFCL is evaluated with single-sample, near-greedy decoding with temperature 0.001, following the official leaderboard protocol. So we only report best@4 and worst@4 for rest benchmark and metrics.

\begin{table}[h]
\centering
\small
\renewcommand{\arraystretch}{1.2}
\setlength{\tabcolsep}{4pt}
\caption{\textbf{Main results on Qwen2.5-7B-Instruct with 3 RL experts under Best@4.} Coding: pass accuracy (ACC) on LiveBench (LB) and LiveCodeBench (LCB). Memory: RULER HotpotQA and SQuAD at 32K–64K context lengths under the Recurrent setting. Best expert/merged result per column in  \textbf{bold}. Shading indicates \colorbox{expertcolor}{Experts} and \colorbox{ourcolor}{Ours}.}\label{tab:merging_qwen_best}
\resizebox{\textwidth}{!}{
\begin{tabular}{l cc cccc c}
\toprule
& \multicolumn{2}{c}{Coding} & \multicolumn{4}{c}{Memory} & \\
\cmidrule(lr){2-3} \cmidrule(lr){4-7}
& LB & LCB v2 & \multicolumn{2}{c}{HotpotQA} & \multicolumn{2}{c}{SQuAD} & \\
\cmidrule(lr){2-2} \cmidrule(lr){3-3} \cmidrule(lr){4-5} \cmidrule(lr){6-7}
Method & ACC & ACC & 32K & 64K & 32K & 64K & Avg \\
\midrule
Base~\cite{qwen2.5} & 48.35 $\pm$ 1.46 & 39.93 $\pm$ 0.64 & 68.98 $\pm$ 1.44 & 66.54 $\pm$ 1.51 & 79.62 $\pm$ 1.23 & 77.20 $\pm$ 0.30 & 63.44 \\
\rowcolor{expertcolor}RLVR-Coder~\cite{wang2025code} & 48.97 $\pm$ 1.20 & \textbf{41.13 $\pm$ 1.50} & 70.03 $\pm$ 0.33 & 66.33 $\pm$ 0.79 & 81.90 $\pm$ 0.59 & 77.60 $\pm$ 0.42 & 64.33 \\
\rowcolor{expertcolor}RLVR-Tool~\cite{qian2025toolrl} & \textbf{51.54 $\pm$ 2.17} & 39.91 $\pm$ 1.66 & 69.31 $\pm$ 0.61 & 67.71 $\pm$ 1.53 & 81.12 $\pm$ 0.80 & 78.74 $\pm$ 0.65 & 64.72 \\
\rowcolor{expertcolor}RLVR-Memory~\cite{yu2025memagent} & 49.74 $\pm$ 2.61 & 37.49 $\pm$ 3.49 & \textbf{86.21 $\pm$ 0.26} & \textbf{86.38 $\pm$ 0.90} & \textbf{87.26 $\pm$ 0.32} & \textbf{88.54 $\pm$ 0.39} & 72.60 \\
\midrule
Task Arithmetic~\cite{ta} & 51.99 $\pm$ 1.78 & 42.46 $\pm$ 0.28 & 81.65 $\pm$ 0.60 & 81.37 $\pm$ 0.37 & 84.13 $\pm$ 0.44 & 87.58 $\pm$ 0.31 & 71.53 \\
TIES~\cite{ties} & 50.04 $\pm$ 2.18 & 42.34 $\pm$ 1.42 & 85.80 $\pm$ 0.18 & 84.32 $\pm$ 0.04 & 85.84 $\pm$ 0.36 & 88.05 $\pm$ 0.37 & 72.73 \\
TSV~\cite{tsv} & 48.82 $\pm$ 0.29 & 39.04 $\pm$ 2.44 & 84.75 $\pm$ 0.37 & 83.54 $\pm$ 0.37 & 86.39 $\pm$ 0.35 & 87.72 $\pm$ 0.44 & 71.71 \\
RAM~\cite{ram} & 48.73 $\pm$ 0.54 & 41.06 $\pm$ 0.25 & 85.91 $\pm$ 0.52 & 84.17 $\pm$ 0.61 & 86.62 $\pm$ 0.43 & \textbf{88.39 $\pm$ 0.57} & 72.48 \\
OrthoMerge-G-TIES~\cite{yang2026orthogonal}  & \textbf{52.20 $\pm$ 2.00} & \textbf{42.52$ \pm$ 1.77} & 84.38 $\pm$ 0.37 & 83.47 $\pm$ 0.40 & 86.06 $\pm$ 0.26 & 87.66 $\pm$ 0.41 & 72.72\\
\midrule
\rowcolor{ourcolor}Ours & 50.39 $\pm$ 1.98 & 40.68 $\pm$ 1.99 & \textbf{86.05 $\pm$ 0.78} & \textbf{84.92 $\pm$ 0.37} & \textbf{86.98 $\pm$ 0.21} & 88.23 $\pm$ 0.38 & \textbf{72.88} \\
\bottomrule
\end{tabular}
}
\end{table}
\begin{table}[h]
\centering
\small
\renewcommand{\arraystretch}{1.2}
\setlength{\tabcolsep}{4pt}
\caption{\textbf{Main results on Qwen2.5-7B-Instruct with 3 RL experts under Worst@4.} Coding: pass accuracy (ACC) on LiveBench (LB) and LiveCodeBench (LCB). Memory: RULER HotpotQA and SQuAD at 32K–64K context lengths under the Recurrent setting. Best expert/merged result per column in  \textbf{bold}. Shading indicates \colorbox{expertcolor}{Experts} and \colorbox{ourcolor}{Ours}.}\label{tab:merging_qwen_worst}
\resizebox{\textwidth}{!}{
\begin{tabular}{l cc cccc c}
\toprule
& \multicolumn{2}{c}{Coding} & \multicolumn{4}{c}{Memory} & \\
\cmidrule(lr){2-3} \cmidrule(lr){4-7}
& LB & LCB v2 & \multicolumn{2}{c}{HotpotQA} & \multicolumn{2}{c}{SQuAD} & \\
\cmidrule(lr){2-2} \cmidrule(lr){3-3} \cmidrule(lr){4-5} \cmidrule(lr){6-7}
Method & ACC & ACC & 32K & 64K & 32K & 64K & Avg \\
\midrule
Base~\cite{qwen2.5} & 24.44 $\pm$ 0.38 & 16.99 $\pm$ 0.11 & 30.76 $\pm$ 0.95 & 24.34 $\pm$ 0.24 & 37.71 $\pm$ 1.31 & 35.55 $\pm$ 0.92 & 28.30 \\
\rowcolor{expertcolor}RLVR-Coder~\cite{wang2025code} & \textbf{25.84 $\pm$ 1.57} & \textbf{19.96 $\pm$ 0.41} & 30.37 $\pm$ 0.84 & 24.72 $\pm$ 0.80 & 38.49 $\pm$ 1.38 & 35.89 $\pm$ 0.65 & 29.21 \\
\rowcolor{expertcolor}RLVR-Tool~\cite{qian2025toolrl} & 21.87 $\pm$ 0.84 & 16.78 $\pm$ 1.66 & 30.61 $\pm$ 1.23 & 23.48 $\pm$ 1.03 & 37.83 $\pm$ 1.69 & 35.89 $\pm$ 1.27 & 27.74 \\
\rowcolor{expertcolor}RLVR-Memory~\cite{yu2025memagent} & 24.26 $\pm$ 0.76 & 17.28 $\pm$ 2.64 & \textbf{70.22 $\pm$ 0.73} & \textbf{68.40 $\pm$ 0.14} & \textbf{69.66 $\pm$ 0.80} & \textbf{67.09 $\pm$ 0.51} & \textbf{52.82} \\
\midrule
Task Arithmetic~\cite{ta} & 26.24 $\pm$ 2.29 & 21.46 $\pm$ 0.52 & 61.86 $\pm$ 0.16 & 58.54 $\pm$ 0.90 & 64.89 $\pm$ 0.89 & 61.44 $\pm$ 0.56 & 49.07 \\
TIES~\cite{ties} & 30.48 $\pm$ 2.58 & 21.41 $\pm$ 0.66 & 65.91 $\pm$ 0.74 & 65.12 $\pm$ 0.20 & \textbf{69.79 $\pm$ 0.59} & \textbf{66.66 $\pm$ 0.46} & 53.23 \\
TSV~\cite{tsv} & 26.12 $\pm$ 1.78 & 19.39 $\pm$ 1.36 & 65.33 $\pm$ 0.99 & 63.43 $\pm$ 0.55 & 69.36 $\pm$ 0.90 & 65.32 $\pm$ 0.98 & 51.49 \\
RAM~\cite{ram} & 27.41 $\pm$ 1.87 & 21.72 $\pm$ 0.57 & 65.30 $\pm$ 0.34 & 65.82 $\pm$ 0.90 & 69.57 $\pm$ 0.43 & 65.75 $\pm$ 0.61 & 52.59 \\
OrthoMerge-G-TIES~\cite{yang2026orthogonal}  & 28.71 $\pm$ 2.11 & 21.65 $\pm$ 0.66 & 64.88 $\pm$ 1.15 & 64.80 $\pm$ 0.59 & 68.31 $\pm$ 0.38 & 63.62 $\pm$ 0.54 & 52.00\\
\midrule
\rowcolor{ourcolor}Ours & \textbf{33.50 $\pm$ 2.28} & \textbf{21.80 $\pm$ 1.61} & \textbf{71.78 $\pm$ 0.96} & \textbf{66.12 $\pm$ 0.31} & 69.75 $\pm$ 0.29 & 66.13 $\pm$ 0.52 & \textbf{54.85} \\
\bottomrule
\end{tabular}
}
\end{table}
\begin{table}[h]
\centering
\small
\renewcommand{\arraystretch}{1.2}
\setlength{\tabcolsep}{3pt}
\caption{\textbf{Main results on DeepSeek-R1-Distill-Qwen-1.5B with 2 RL experts under Best@4}. Coding ACC on LB and LCB v5. Math: ACC on AIME24, AIME25, AMC23, Minerva, and OlympiadBench. Best expert/merged result per column in \textbf{bold}. Shading indicates \colorbox{expertcolor}{Experts} and \colorbox{ourcolor}{Ours}.}\label{tab:merging_ds_best}
\resizebox{\textwidth}{!}{
\begin{tabular}{l cc ccccc c}
\toprule
& \multicolumn{2}{c}{Coding} & \multicolumn{5}{c}{Math} & \\
\cmidrule(lr){2-3} \cmidrule(lr){4-8}
& LB & LCB v5 & AIME24 & AIME25 & AMC23 & Minerva & Olympiad & Total \\
\cmidrule(lr){2-2} \cmidrule(lr){3-3} \cmidrule(lr){4-4} \cmidrule(lr){5-5} \cmidrule(lr){6-6} \cmidrule(lr){7-7} \cmidrule(lr){8-8}
Method & ACC & ACC & ACC & ACC & ACC & ACC & ACC & Avg \\
\midrule
Base~\cite{deepseek_r1} & 29.68 $\pm$ 0.17 & 25.96 $\pm$ 0.42 & 54.49 $\pm$ 1.34 & 33.34 $\pm$ 0.99 & 80.49 $\pm$ 0.80 & 41.67 $\pm$ 2.00 & 56.05 $\pm$ 0.56 & 45.95 \\
\rowcolor{expertcolor}Archer2.0~\cite{archer} & \textbf{36.11 $\pm$ 0.45} & \textbf{36.75 $\pm$ 0.27} & 64.74 $\pm$ 0.72 & 39.88 $\pm$ 1.12 & 86.49 $\pm$ 0.12 & 43.38 $\pm$ 0.52 & 61.19 $\pm$ 0.32 & 52.65 \\
\rowcolor{expertcolor}JustRL~\cite{he2025justrl} & 31.67 $\pm$ 0.42 & 32.63 $\pm$ 0.62 & \textbf{71.62 $\pm$ 0.16} & \textbf{49.25 $\pm$ 1.32} & \textbf{90.83 $\pm$ 0.65} & \textbf{45.22 $\pm$ 1.08} & \textbf{64.49 $\pm$ 0.91} & 55.10 \\
\midrule
Task Arithmetic~\cite{ta} & \textbf{36.10 $\pm$ 0.45} & 37.18 $\pm$ 0.35 & 72.34 $\pm$ 0.54 & 47.12 $\pm$ 2.32 & 90.51 $\pm$ 0.17 & 44.61 $\pm$ 0.46 & 64.20 $\pm$ 0.56 & 56.01 \\
TIES~\cite{ties} & 34.31 $\pm$ 0.69 & 37.66 $\pm$ 0.31 & \textbf{73.40 $\pm$ 0.42} & 47.37 $\pm$ 0.77 & 89.62 $\pm$ 0.38 & 45.10 $\pm$ 0.87 & 64.35 $\pm$ 0.35 & 55.97 \\
TSV~\cite{tsv} & 35.84 $\pm$ 0.34 & \textbf{37.67 $\pm$ 0.39} & 72.41 $\pm$ 1.07 & 46.94 $\pm$ 1.18 & 89.93 $\pm$ 1.10 & \textbf{45.10 $\pm$ 0.92} & 64.94 $\pm$ 0.42 & 56.12 \\
RAM~\cite{ram} & 35.34 $\pm$ 0.87 & 36.96 $\pm$ 1.04 & 73.06 $\pm$ 0.62 & 49.50 $\pm$ 1.81 & 90.36 $\pm$ 0.63 & 44.98 $\pm$ 0.76 & 64.59 $\pm$ 0.55 & 56.40 \\
OrthoMerge-G-TIES~\cite{yang2026orthogonal}  & 35.18 $\pm$ 0.72 & 37.47 $\pm$ 0.50 & 73.95 $\pm$ 1.16 & 47.13 $\pm$ 0.96 & 89.82 $\pm$ 0.51 & 44.36 $\pm$ 0.56 & 64.74 $\pm$ 0.68 & 56.09\\
\midrule
\rowcolor{ourcolor}Ours & 34.04 $\pm$ 0.56 & 36.87 $\pm$ 0.09 & 72.69 $\pm$ 0.34 & \textbf{50.43 $\pm$ 0.71} & \textbf{91.31 $\pm$ 0.61} & 44.73 $\pm$ 0.87 & \textbf{64.99 $\pm$ 0.56} & \textbf{56.44} \\
\bottomrule
\end{tabular}
}
\end{table}

\begin{table}[h]
\centering
\small
\renewcommand{\arraystretch}{1.2}
\setlength{\tabcolsep}{3pt}
\caption{\textbf{Main results on DeepSeek-R1-Distill-Qwen-1.5B with 2 RL experts under Worst@4}. Coding ACC on LB and LCB v5. Math: ACC on AIME24, AIME25, AMC23, Minerva, and OlympiadBench. Best expert/merged result per column in \textbf{bold}. Shading indicates \colorbox{expertcolor}{Experts} and \colorbox{ourcolor}{Ours}.}\label{tab:merging_ds_worst}
\resizebox{\textwidth}{!}{
\begin{tabular}{l cc ccccc c}
\toprule
& \multicolumn{2}{c}{Coding} & \multicolumn{5}{c}{Math} & \\
\cmidrule(lr){2-3} \cmidrule(lr){4-8}
& LB & LCB v5 & AIME24 & AIME25 & AMC23 & Minerva & Olympiad & Total \\
\cmidrule(lr){2-2} \cmidrule(lr){3-3} \cmidrule(lr){4-4} \cmidrule(lr){5-5} \cmidrule(lr){6-6} \cmidrule(lr){7-7} \cmidrule(lr){8-8}
Method & ACC & ACC & ACC & ACC & ACC & ACC & ACC & Avg \\
\midrule
Base~\cite{deepseek_r1} & 8.38 $\pm$ 0.38 & 8.95 $\pm$ 0.47 & 12.42 $\pm$ 0.92 & 12.96 $\pm$ 0.56 & 46.02 $\pm$ 0.64 & 14.95 $\pm$ 0.76 & 30.17 $\pm$ 0.73 & 19.12 \\
\rowcolor{expertcolor}Archer2.0~\cite{archer} & \textbf{15.95 $\pm$ 0.80} & \textbf{17.39 $\pm$ 0.24} & 19.80 $\pm$ 0.63 & 17.36 $\pm$ 0.36 & 57.83 $\pm$ 1.03 & 18.75 $\pm$ 1.08 & 38.22 $\pm$ 0.60 & 26.47 \\
\rowcolor{expertcolor}JustRL~\cite{he2025justrl} & 14.02 $\pm$ 0.22 & 14.42 $\pm$ 0.14 & \textbf{33.61 $\pm$ 0.93} & \textbf{27.25 $\pm$ 1.15} & \textbf{72.89 $\pm$ 0.65} & \textbf{23.90 $\pm$ 0.60} & \textbf{46.32 $\pm$ 0.14} & \textbf{33.20} \\
\midrule
Task Arithmetic~\cite{ta} & 15.82 $\pm$ 0.26 & 16.67 $\pm$ 0.50 & 28.49 $\pm$ 0.56 & 20.99 $\pm$ 0.47 & 68.35 $\pm$ 0.85 & 21.94 $\pm$ 0.92 & 43.41 $\pm$ 0.64 & 30.81 \\
TIES~\cite{ties} & \textbf{17.76 $\pm$ 0.87} & 18.22 $\pm$ 0.27 & 33.40 $\pm$ 1.45 & 25.21 $\pm$ 0.22 & 71.85 $\pm$ 0.97 & 23.28 $\pm$ 1.25 & 45.23 $\pm$ 0.07 & 33.57 \\
TSV~\cite{tsv} & 15.91 $\pm$ 0.52 & 17.46 $\pm$ 0.16 & 32.33 $\pm$ 0.38 & 24.69 $\pm$ 1.48 & 68.36 $\pm$ 1.44 & 24.02 $\pm$ 0.87 & 42.86 $\pm$ 0.81 & 32.23 \\
RAM~\cite{ram} & 17.41 $\pm$ 1.11 & 17.84 $\pm$ 0.25 & 33.90 $\pm$ 0.86 & 24.10 $\pm$ 0.14 & 71.46 $\pm$ 0.29 & 25.12 $\pm$ 0.35 & 45.48 $\pm$ 0.12 & 33.62 \\
OrthoMerge-G-TIES~\cite{yang2026orthogonal}  & 16.45 $\pm$ 0.67 & 18.27 $\pm$ 0.19 & 33.21 $\pm$ 1.99 & 24.86 $\pm$ 0.60 & 70.85 $\pm$ 0.80 & 23.77 $\pm$ 0.21 & 45.28 $\pm$ 0.60 & 33.24\\
\midrule
\rowcolor{ourcolor}Ours & 16.18 $\pm$ 0.18 & \textbf{18.95 $\pm$ 0.45} & \textbf{36.10 $\pm$ 1.01} & \textbf{27.41 $\pm$ 0.83} & \textbf{73.23 $\pm$ 1.81} & \textbf{25.12 $\pm$ 1.21} & \textbf{47.85 $\pm$ 0.12} & \textbf{34.98} \\
\bottomrule
\end{tabular}
}
\end{table}

Tab.~\ref{tab:merging_qwen_best},~\ref{tab:merging_qwen_worst},
~\ref{tab:merging_ds_best}, and~\ref{tab:merging_ds_worst} show that
ISO-Merger achieves the highest upper and lower tails under both backbones.
In particular, for worst@4, which is substantially more challenging,
ISO-Merger reaches an overall average of $54.85$, compared with $53.23$ for
the strongest training-free baseline. On
\texttt{DeepSeek-R1-Distill-Qwen-1.5B}, ISO-Merger reaches a total average of
$34.98$, improving over the best baseline at $33.62$. For best@4, ISO-Merger
remains competitive with the original specialists while composing multiple
skills into a single model. These results further demonstrate that different
RL experts initialized from the same base model can be combined in
fixed-spectrum Stiefel coordinates.

\subsection{Ablations of ISO-Merger Design Choices}
\label{app:iso_merger_ablate}

\paragraph{Setup.}
We isolate three design choices in ISO-Merger: spectrum re-basing,
retention-aware composition, and mode masking.
First, \textbf{TIES + spec-rebase} applies the strongest Euclidean
merging baseline, TIES~\cite{ties}, and subsequently re-bases the
merged checkpoint to the base spectrum, testing whether spectrum
preservation alone accounts for ISO-Merger's gains.
Second, \textbf{w/o retention-aware composition} replaces the learned
retention coefficients with uniform coefficients, $c_1=c_2=1$.
Third, \textbf{w/o mode masking} retains all frame modes by setting
$\rho_{\mathrm{keep}}=1$, instead of $\rho_{\mathrm{keep}}=0.9$
used by the full method.
All other ISO-Merger components and evaluation settings are held fixed.
Results on \texttt{DeepSeek-R1-Distill-Qwen-1.5B} are reported in
Tab.~\ref{tab:iso_merger_ablation}.

\paragraph{Results.}
The ablations separate the contributions of the three design choices.
Spectrum re-basing alone is insufficient: TIES + spec-rebase obtains
an aggregate score of $43.09$, below full ISO-Merger at $44.38$
and also below vanilla TIES at $43.52$
(Tab.~\ref{tab:great_merge2}).
Mode masking has the largest effect among the tested components:
removing it reduces the aggregate score from $44.38$ to $42.02$,
with degradation in both coding and math.
Retention-aware composition provides a smaller but complementary
improvement: uniform coefficients reach $44.14$.
Interestingly, removing retention-aware composition slightly improves
several math metrics, but reduces the coding-domain average from
$32.91$ to $32.06$; the full method therefore achieves a better
cross-domain balance and the highest overall score.
Together, these results show that ISO-Merger's gain is not explained
by spectrum re-basing alone: mode selection provides the largest
benefit among the tested choices, while retention-aware weighting
further improves balanced specialist composition.

\begin{table}[ht]
\centering
\caption{
\textbf{ISO-Merger ablations on DeepSeek-R1-Distill-Qwen-1.5B
with two RL experts.}
All variants merge the same Archer2.0 and JustRL checkpoints and
follow the evaluation protocol of Tab.~\ref{tab:great_merge2}. Values are mean $\pm$ standard deviation; Avg is the average over all columns. Best expert/merged result per column in \textbf{bold}. Shading indicates \colorbox{expertcolor}{Experts} and \colorbox{ourcolor}{Ours}.
}
\label{tab:iso_merger_ablation}
\begingroup
\setlength{\tabcolsep}{4pt}
\renewcommand{\arraystretch}{1.1}
\resizebox{\textwidth}{!}{%
\begin{tabular}{l*{10}{c}}
\toprule
\multirow{3}{*}{Method}
& \multicolumn{4}{c}{Coding}
& \multicolumn{5}{c}{Math}
& \multirow{3}{*}{Avg} \\
\cmidrule(lr){2-5} \cmidrule(lr){6-10}
& \multicolumn{2}{c}{LB}
& \multicolumn{2}{c}{LCB v5}
& AIME 2024 & AIME 2025 & AMC 2023
& Minerva & OlympiadBench & \\
\cmidrule(lr){2-3} \cmidrule(lr){4-5}
& ACC & UT & ACC & UT
& Avg@32 & Avg@32 & Avg@8 & Avg@4 & Avg@4 & \\
\midrule
Base~\cite{deepseek_r1} & 17.99 $\pm$ 0.29 & 26.24 $\pm$ 0.53 & 16.82 $\pm$ 0.28 & 20.85 $\pm$ 0.18 & 30.90 $\pm$ 1.15 & 23.40 $\pm$ 0.86 & 63.15 $\pm$ 0.38 & 27.33 $\pm$ 0.43 & 43.11 $\pm$ 0.26 & 29.98 \\
\rowcolor{expertcolor}
Archer2.0~\cite{archer} & \textbf{26.66 $\pm$ 0.45} & \textbf{37.26 $\pm$ 0.55} & \textbf{26.90 $\pm$ 0.21} & \textbf{38.99 $\pm$ 0.37} & 42.05 $\pm$ 0.43 & 28.02 $\pm$ 0.44 & 73.44 $\pm$ 0.50 & 30.09 $\pm$ 0.23 & 49.99 $\pm$ 0.65 & 39.27 \\
\rowcolor{expertcolor}
JustRL~\cite{he2025justrl} & 22.23 $\pm$ 0.23 & 31.85 $\pm$ 0.74 & 23.05 $\pm$ 0.21 & 30.18 $\pm$ 0.31 & \textbf{53.16 $\pm$ 0.44} & \textbf{36.84 $\pm$ 0.73} & \textbf{82.78 $\pm$ 0.51} & \textbf{34.71 $\pm$ 0.50} & \textbf{55.67 $\pm$ 0.42} & \textbf{41.16} \\
\midrule
\rowcolor{ourcolor}
ISO-Merger (full)
& 25.52 $\pm$ 0.08 & 36.42 $\pm$ 0.55
& \textbf{27.84 $\pm$ 0.13} & \textbf{41.84 $\pm$ 0.31}
& 55.00 $\pm$ 0.52 & \textbf{37.81 $\pm$ 0.31}
& 83.53 $\pm$ 1.02 & 34.77 $\pm$ 0.23
& 56.64 $\pm$ 0.05 & \textbf{44.38} \\
\midrule
w/o retention-aware composition
& 24.38 $\pm$ 0.12 & 35.62 $\pm$ 0.50
& 27.02 $\pm$ 0.33 & 41.22 $\pm$ 0.27
& \textbf{55.62 $\pm$ 1.05}
& 37.29 $\pm$ 0.23
& \textbf{84.29 $\pm$ 1.31}
& \textbf{34.90 $\pm$ 0.23}
& \textbf{56.91 $\pm$ 0.12}
& 44.14 \\

w/o mode masking
& 21.48 $\pm$ 0.08 & 33.76 $\pm$ 0.04
& 24.34 $\pm$ 0.46 & 39.36 $\pm$ 0.30
& 52.78 $\pm$ 0.51 & 35.42 $\pm$ 0.68
& 82.18 $\pm$ 0.70 & 33.33 $\pm$ 0.09
& 55.54 $\pm$ 0.54 & 42.02 \\

TIES + spec-rebase
& \textbf{26.43 $\pm$ 0.45} & \textbf{36.50 $\pm$ 0.56}
& 27.19 $\pm$ 0.06 & 38.09 $\pm$ 0.45
& 52.88 $\pm$ 0.93 & 36.01 $\pm$ 0.32
& 81.88 $\pm$ 0.94 & 33.70 $\pm$ 0.50
& 55.09 $\pm$ 0.17 & 43.09 \\\bottomrule
\end{tabular}%
}
\endgroup
\end{table}

\section{Numerical Precision of the SVD-Based Retraction}
\label{sec:svd_precision_issue}

In exact arithmetic, the ISO reconstruction
\[
    W^+
    =
    U^+\Sigma_0(V^+)^\top
\]
has singular values $\operatorname{diag}(\Sigma_0)$ whenever
$U^+$ and $V^+$ have orthonormal columns. In practice, however, the polar
retraction is computed using a finite-precision SVD, so numerical errors in
the decomposition can weaken this invariant.

We therefore compare four PyTorch SVD configurations:
FP32 and FP64 on CPU and GPU. As a representative sanity check, we extract the
\texttt{q\_proj} matrix from layer 10 of
\texttt{Qwen3-1.7B-Base},
\[
    W_0\in\mathbb R^{2048\times2048},
\]
compute
\[
    W_0=U\Sigma V^\top,
\]
and reconstruct
\[
    W_{\mathrm{rec}}=U\Sigma V^\top.
\]
We report the mean squared spectral reconstruction error
\begin{equation}
\label{eq:svd_spectral_error}
\small
    \epsilon_{\sigma}
    :=
    \frac{1}{q}
    \left\|
        \sigma(W_{\mathrm{rec}})
        -
        \sigma(W_0)
    \right\|_2^2,
\end{equation}
together with the wall-clock decomposition time.

Tab.~\ref{tab:svd_precision} shows that FP32 GPU SVD introduces substantially
larger numerical error than FP64 GPU SVD:
$7.7899\times10^{-4}$ versus
$2.8924\times10^{-8}$ in this representative matrix.
FP64 reduces the error by approximately
$2.7\times10^4$ while increasing the measured runtime only from
$0.894$ to $0.923$ seconds.
We therefore use FP64 GPU SVD for the polar retraction throughout ISO.

\begin{table}[ht!]
    \centering
    \caption{
        Runtime and mean squared spectral reconstruction error under different
        PyTorch SVD precision and device configurations.
        }
    \label{tab:torchsvd_speed_error}
    \begin{tabular}{cc|cc}
        \toprule
        Precision & Device & Mean squared spectral error & Time (s) \\
        \midrule
        FP32  & CPU & 3.6633e-7 & 2.115  \\
        FP64  & CPU & 1.1126e-8 & 3.699 \\
        FP32  & GPU & \textbf{7.7899e-4} & 0.894 \\
        FP64  & GPU & 2.8924e-8 & 0.923 \\
        \bottomrule
    \end{tabular}\label{tab:svd_precision}
\end{table}

\section{Online RLVR Training Details}
\label{app:training_details}

The ISO parameterization, factor gradients, and retraction are described in
Section~\ref{sec:iso_optimizer}. 

\subsection{Training and Evaluation Configurations}

\paragraph{Common setup.}
All experiments are implemented using
\textsc{Verl}~\cite{sheng2025hybridflow} and run on NVIDIA H100 80\,GB GPUs.
We use the DAPO algorithm~\cite{yu2025dapo} and hold all non-optimizer
settings fixed between the weight-space and ISO variants.

We use asymmetric policy-ratio clipping with
\begin{equation}
\small
    (\epsilon_{\mathrm{low}},\epsilon_{\mathrm{high}})
    =
    (0.2,0.28).
\end{equation}
The KL coefficient is set to
\[
    \beta_{\mathrm{KL}}
    =
    10^{-3},
\]
with the KL term applied as a loss-shaping term rather than as part of the
rollout reward. We enable online dynamic filtering by removing prompt groups
whose sampled rollouts are either all correct or all incorrect. Unless stated
otherwise, training rollouts use temperature $1.0$ and top-$p=1.0$.

We set weight decay to zero for both the weight-space and ISO variants. At the
learning rates used here, the standard coefficient $\lambda=10^{-2}$ falls
below the BF16-visible update scale and provides no measurable benefit in our
RLVR runs~\cite{zhu2025path}.

\paragraph{Mathematical reasoning.}
We train
\texttt{Qwen3-1.7B-Base} and \texttt{Qwen3-4B-Base}~\cite{qwen3technicalreport}
on \textsc{DeepMath}-103K~\cite{deepmath}, following prior RLVR
post-training recipes~\cite{schulman2025lora}.

For \texttt{Qwen3-1.7B-Base} and \texttt{Qwen3-4B-Base}, we use a global
prompt batch size of $256$ and train for $400$ actor updates, corresponding to
approximately one nominal pass over the prompt set before online filtering.
The 1.7B runs sample $16$ rollouts per prompt, whereas the 4B runs sample $12$.
Both use a maximum prompt length of $1{,}024$ tokens and a maximum response
length of $16{,}384$ tokens.

ISO-AdamW uses a learning rate of
$1\times10^{-6}$ for the frame variables $(U,V)$, followed by polar
retraction after every tentative factor update.\footnote{On
\texttt{Qwen3-1.7B-Base}, we compare learning rates
$7.5\times10^{-7}$ and $1\times10^{-6}$. The latter performs better, so we
use it throughout the math experiments.}
For weight-space AdamW, we sweep
\[
    \left\{
        \;7.5\times10^{-7},
        \;1\times10^{-6},
        \;2\times10^{-6},
        \;3\times10^{-6}
    \right\}.
\]
For the weight-space Muon baselines, we sweep
\[
    \left\{
        2.5\times10^{-5},
        \;5\times10^{-5},
        \;7.5\times10^{-5}
    \right\}.
\]
The ISO-Muon experiment is conducted on
\texttt{Qwen3-4B-Base} using a learning rate of
$2.5\times10^{-5}$.

We evaluate mathematical reasoning on AIME 2024, AIME 2025, AMC 2023,
Minerva, and OlympiadBench. For each problem, we sample $16$ responses using
temperature $1.0$ and top-$p=0.8$. The maximum evaluation length is
$16{,}384$ tokens for the 1.7B and 4B models.
For Tab.~\ref{tab:math_main_results}, we
select the final checkpoint at step 400 for evaluation and report all benchmark scores from that checkpoint.

\paragraph{Competitive coding.}\label{app:coding_details}
For competitive coding, we train
\texttt{DeepSeek-R1-Distill-Qwen-1.5B}
~\cite{deepseek_r1}
on the \textsc{ArcherCodeR} training split~\cite{wang2025stabilizing}, which
contains $6{,}753$ problems. We use a global prompt batch size of $128$ and
sample $16$ rollouts per problem. The maximum prompt length is $2{,}048$
tokens, and the maximum response length is $32{,}768$ tokens.

ISO-AdamW uses a learning rate of
$1.5\times10^{-6}$ for $(U,V)$, followed by polar retraction after each factor
update. For the weight-space AdamW baseline, we sweep
\[
    \left\{
        1\times10^{-6},
        \;2\times10^{-6},
        \;3\times10^{-6},
        \;5\times10^{-6}
    \right\},
\]
where $1\times10^{-6}$ is the learning rate used by the original
Archer-style recipe. AdamW with a learning rate of $1\times10^{-6}$ performs strictly worse and is omitted from the figure for clarity.

We train the primary runs for $200$ actor updates. At a prompt batch size of
$128$, this budget corresponds to approximately
\[
    \frac{200\times128}{6753}
    \approx
    3.79
\]
nominal passes over the prompt set before accounting for online filtering.
Because the coding dataset is relatively small, longer training can enter a
multi-pass overfitting regime. To test whether the primary budget truncates the
weight-space baseline prematurely, we continue the two strongest AdamW runs,
with learning rates $3\times10^{-6}$ and $5\times10^{-6}$, to $260$ actor
updates.

For evaluation, we report results on
\textsc{LiveCodeBench}~\cite{livecodebench} v5 and v6. We sample $8$
responses per problem using temperature $0.8$ and top-$p=1.0$, and report
average accuracy under the same evaluation protocol for all methods.

\subsection{Memory Overhead of ISO-AdamW}
\label{app:memory_overhead}

For a weight matrix $W\in\mathbb{R}^{d_{\mathrm{out}}\times d_{\mathrm{in}}}$,
ISO-AdamW uses the fixed-spectrum parameterization $W=U\Sigma_0V^\top$, where $U\in\mathbb{R}^{d_{\mathrm{out}}\times q}$, $V\in\mathbb{R}^{d_{\mathrm{in}}\times q}$,
$\Sigma_0\in\mathbb{R}^{q\times q}$, and $q=\min(d_{\mathrm{out}},d_{\mathrm{in}})$. The factors satisfy $U^\top U=V^\top V=I_q$. ISO-AdamW maintains AdamW's first- and second-moment buffers for $U$ and $V$ rather than for $W$; the fixed spectrum $\Sigma_0$ has no associated optimizer state. We first quantify the intrinsic overhead of these moment buffers, independently of how they are distributed across devices.

\paragraph{Intrinsic overhead.}
For each transformed matrix, AdamW stores
$2d_{\mathrm{out}}d_{\mathrm{in}}$ moment elements, whereas
ISO-AdamW stores $2q(d_{\mathrm{out}}+d_{\mathrm{in}})$.
When both optimizers use the same precision for their moment
buffers, the moment-memory ratio is
\begin{equation}
R_{\mathrm{moment}}
=
\frac{2q(d_{\mathrm{out}}+d_{\mathrm{in}})}
     {2d_{\mathrm{out}}d_{\mathrm{in}}}
=
1+
\frac{\min(d_{\mathrm{out}},d_{\mathrm{in}})}
     {\max(d_{\mathrm{out}},d_{\mathrm{in}})}
\in(1,2].
\end{equation}
The worst-case ratio of $2\times$ is attained only for square
matrices. For \texttt{Qwen3-4B-Base}, the per-matrix ratios are
approximately $1.63$ for q/o projections, $1.40$ for k/v
projections, and $1.26$ for gate/up/down projections.
Aggregating over the 252 ISO-transformed matrices and the
remaining parameters---normalization parameters and the tied
embedding, which use plain AdamW in both configurations---the
model-wide FP32 moment memory increases from $32.2$\,GB to
$42.2$\,GB, an increase of $10.0$\,GB or approximately $31\%$.
For this parameter scope and moment precision, the model-wide
ratio is independent of sharding and offloading, and we report
it as our primary moment-memory metric.

\paragraph{Per-device cost under the training configuration.}
All optimizer states are sharded by fully sharded data parallelism, so the intrinsic overhead maps to a much
smaller per-accelerator increment. Under the identical 16-way FSDP configuration (2 nodes $\times$ 8 H100
80\,GB) used for both AdamW and ISO-AdamW in Fig.~\ref{fig:iso_overhead}, the $+10.0$\,GB model-wide
overhead corresponds to $+0.63$\,GB per device (Tab.~\ref{tab:memory_overhead}). The two quantities answer
complementary questions: the model-wide ratio measures the algorithmic storage cost, the per-device increment
determines practical deployability; we report both.

\begin{table}[h]
\centering
\small
\caption{Optimizer moment memory for Qwen3-4B-Base under AdamW and ISO-AdamW (FP32). The 16-way row is the
configuration of the runs in Fig.~\ref{fig:iso_overhead}.}
\label{tab:memory_overhead}
\begin{tabular}{lccc}
\toprule
Metric & AdamW & ISO-AdamW & ISO overhead \\
\midrule
Model-wide FP32 moment buffers            & 32.2 GB & 42.2 GB & $+10.0$ GB ($+31\%$) \\
Per device, 16-way FSDP (our runs)        & 2.01 GB & 2.64 GB & $+0.63$ GB ($0.8\%$ of 80 GB) \\
\bottomrule
\end{tabular}
\end{table}

Our implementation additionally holds $U$ and $V$ in FP32 as optimizer state ($q(d_{\mathrm{out}}+d_{\mathrm{in}})$
elements per matrix; $19.6$\,GB model-wide for Qwen3-4B, i.e.\ $1.2$\,GB per device at 16-way sharding), and
writes the reconstructed $W=U\Sigma_0 V^\top$ back into the model parameters after every update. This storage is
an implementation convenience rather than an intrinsic requirement, for two reasons.
(1)~\emph{The factors are a deterministic function of the weights.} Since $\Sigma_0$ is fixed, $U$ and $V$ can be
recovered from $W$ by an SVD at the start of each update, exactly as they are obtained once at initialization.
The cost
of this SVD is small relative to an RL step: computing it for all 252 matrices of Qwen3-4B takes
$26$\,s on the 16-GPU layout of our runs, compared to a mean RL step of roughly one hour; the paired sign
ambiguity of singular vectors is resolved by a fixed sign convention.
(2)~\emph{They are ordinary optimizer state and can be offloaded.} Like the moment buffers, $U$ and $V$ live in
the optimizer's state dictionary and are handled by the standard optimizer-state offloading of the training
framework, which moves all optimizer tensors to host memory between updates. With offloading enabled, no ISO-specific memory is resident on the accelerator outside the optimizer update itself.

\paragraph{Measured peak memory.}
In the two runs of Fig.~\ref{fig:iso_overhead} (no offloading), the per-GPU peak
 allocated memory reported by
the actor workers is $80.8$\,GB for AdamW and $82.6$\,GB for ISO-AdamW ($+1.8$\,GB), consistent with the
$+0.63$\,GB of moment buffers plus the $1.2$\,GB of factors above.

\paragraph{Trade-off.}
The practical memory increment should be read alongside time-to-target. Under the same distributed
configuration, ISO-AdamW reaches the final accuracy of the AdamW run (0.448 at 400 updates, 364\,h) at
update 260 after 240\,h, a reduction of 124 wall-clock hours (34\%). The realized trade-off is therefore
$+0.63$\,GB of moment memory per device (plus $1.2$\,GB of factor storage that can be recomputed or
offloaded) for 124 fewer hours to matched accuracy. This efficiency gain does not erase the intrinsic $31\%$ overhead of the moment buffers; it provides the context needed to evaluate the full resource trade-off.

\subsection{More RLVR Results and Ablations}\label{app:more_rl_res}

\begin{wrapfigure}{r}{0.46\textwidth}
  \vspace{-1.5em}
  \centering
  \includegraphics[width=\linewidth]{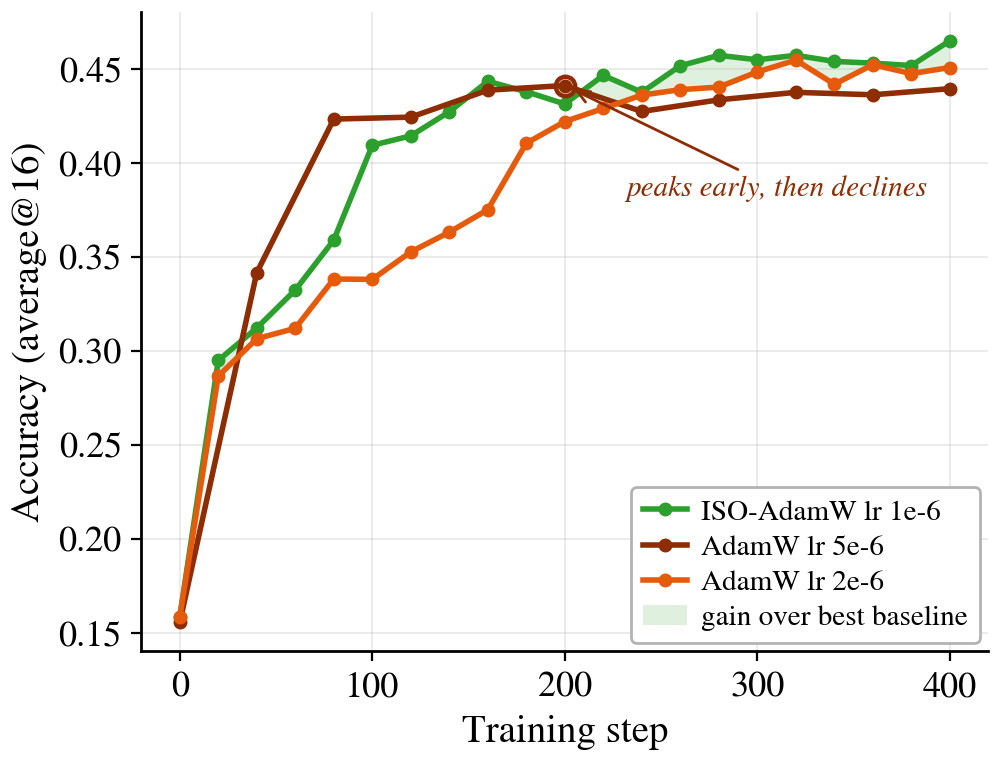}
  \vspace{-2em}
\caption{
\small
\textbf{AdamW with learning rate of $5\times10^{-6}$  on \texttt{Qwen3-4B-Base}.}\vspace{-1em}
}
  \label{fig:math_4b_5e6}
\end{wrapfigure}
\paragraph{Math Results.}
As in Fig.~\ref{fig:math_4b_5e6}, ISO-AdamW at lr $1\times10^{-6}$ is compared with AdamW at its best-performing lr ($2\times10^{-6}$) and at a larger lr ($5\times10^{-6}$). The larger lr accelerates AdamW early but peaks around update 200 and then declines, so a higher learning rate does not close the gap to ISO-AdamW. Therefore, we omit the curve of AdamW $5\times10^{-6}$ in Fig.~\ref{fig:math_adam_curves} for clarity.

\paragraph{Wall-clock cost and optimizer overhead details.}
In Fig.~\ref{fig:iso_overhead} (left), we measure the complete \texttt{optimizer.step()} call for \texttt{Qwen3-4B-Base}. For ISO-AdamW, the measured call includes all-to-all communication
to assemble full matrices, factor-gradient construction, AdamW
updates to both frames, polar retraction, weight reconstruction,
all-to-all redistribution of the updated shards, and plain AdamW
updates to normalization and embedding parameters. It excludes forward/backward computation, gradient clipping, and gradient clearing, which occur outside \texttt{optimizer.step()}
within the actor-update stage. The one-time initialization SVD takes approximately $26$\,s and
is reported separately. 
We average ten measured optimizer steps, obtaining $49.2$\,s for ISO-AdamW and $15$\,ms for AdamW. For the training-step comparison in Fig.~\ref{fig:iso_overhead} (left), ISO-AdamW's mean logged training-step time over 400 indices is $3{,}958$\,s; normalizing the optimizer-step latency by this
mean gives $49.2/3{,}958 \approx 1.2\%$.

In Fig.~\ref{fig:iso_overhead} (right), the cumulative-time curves sum the actual runs'
\texttt{timing\_s/step} measurements.
The trainer's step timer covers rollout generation
(including repeated DAPO sampling rounds), reward computation,
old-policy and reference-policy log-probability computation,
advantage computation, and the actor update.






\end{document}